\documentclass{article}

\usepackage[final]{neurips_2022}

\usepackage[utf8]{inputenc} 
\usepackage[T1]{fontenc}    
\usepackage{hyperref}       
\usepackage{url}            
\usepackage{booktabs}       
\usepackage{amsfonts}       
\usepackage{nicefrac}       
\usepackage{microtype}      
\usepackage{xcolor}         
\usepackage{multirow}
\usepackage{subcaption}
\usepackage{graphicx} 
\usepackage{amsthm}
\usepackage[normalem]{ulem} 
\theoremstyle{plain}
\ifx\theorem\undefined
\newtheorem{theorem}{Theorem}

\newtheorem{lemma}{Lemma}

\theoremstyle{definition}
\newtheorem{definition}{Definition}

\theoremstyle{remark}

\definecolor{mygreen}{RGB}{30, 180, 50}

\definecolor{fhcolor}{rgb}{0.523, 0.235, 0.625}

\providecommand{\realnum}{\mathbb{R}}
\providecommand{\naturalnum}{\mathbb{N}}
 
\providecommand{\pinet}{$\Pi$-Nets}

\providecommand{\pnn}{NN-Hp}
\providecommand{\pnns}{NNs-Hp}
\providecommand{\mfn}{MFN}

\providecommand{\nn}{NN}
\providecommand{\nns}{NNs}
\providecommand{\fcnn}{fully-connected NN}
\providecommand{\fcnns}{fully-connected NNs}

\providecommand{\degorlay}{degree}

\newcommand{\minus}{\scalebox{0.75}[1.0]{$-$}}

\usepackage{amsmath,amsfonts,bm}

\def\eqref#1{Eqn.~\ref{#1}}

\def\1{\bm{1}}

\def\fm{{y}}
\def\xi{{x}}
\def\chi{{y}}

\def\nmone{{n\text{--}1}}

\def\rvtheta{{\mathbf{\theta}}}

\def\va{{\bm{a}}}
\def\vb{{\bm{b}}}

\def\ve{{\bm{e}}}

\def\vv{{\bm{v}}}
\def\vw{{\bm{w}}}
\def\vx{{\bm{x}}}
\def\vy{{\bm{y}}}

\providecommand{\prodbias}{\vb}

\def\mB{{\bm{B}}}

\def\mD{{\bm{D}}}

\def\mF{{\bm{F}}}

\def\mK{{\bm{K}}}

\def\mW{{\bm{W}}}

\def\mBeta{{\bm{\beta}}}

\DeclareMathAlphabet{\mathsfit}{\encodingdefault}{\sfdefault}{m}{sl}
\SetMathAlphabet{\mathsfit}{bold}{\encodingdefault}{\sfdefault}{bx}{n}

\def\gX{{\mathcal{X}}}

\def\emJ{{J}}

\newcommand{\preact}{\tilde{\alpha}}
\newcommand{\act}{\alpha}

\newcommand{\E}{\mathbb{E}}

\newcommand{\R}{\mathbb{R}}

\providecommand{\naturalnum}{\mathbb{N}}
 
\newcommand{\pr}{\mathbb{P}} 

\newcommand{\x}{\bm{x}} 
 
\newcommand{\dd}{\mathrm{d}} 
\newcommand{\xp}{\bm{x}^{\prime}} 
\newcommand{\w}{\bm{w}} 
\newcommand{\W}{\bm{W}}

\newcommand{\bt}{\bm{\beta}} 
\newcommand{\para}{\bm{\theta}} 
\newcommand{\I}{\mathbb{I}} 
 
\newcommand{\N}{\mathcal{N}} 
 
\newcommand{\polyin}{\vx}
\newcommand{\polyout}{f({\vx})}

\newcommand{\polyweight}{A}
\newcommand{\polyinre}{\vx}
\newcommand{\polyoutre}{f({\vx})}

\newcommand{\mlpout}{f({\vx})}

\newcommand{\Gm}{p}
\newcommand{\Gn}{q}

\providecommand{\norm}[1]{\left\lVert#1\right\rVert}

\usepackage{cleveref}
\crefname{ineq}{inequality}{inequalities}
\crefname{equation}{Eq.}{Eq.}
\creflabelformat{ineq}{#2{\upshape(#1)}#3} 
\crefname{theorem}{Theorem}{Theorem}
\crefname{claim}{Claim}{Claim}
\crefname{lemma}{Lemma}{Lemma}
\crefname{appendix}{Appendix}{Appendix}
\crefname{figure}{Figure}{Figure}
\crefname{table}{Table}{Table}
\crefname{section}{Section}{Section}
\title{
Extrapolation and Spectral Bias of Neural Nets with Hadamard Product: a Polynomial Net Study
}
\usepackage{selectp}

\author{
Yongtao Wu, \quad Zhenyu Zhu, \quad Fanghui Liu, \quad Grigorios G Chrysos, \quad Volkan Cevher\vspace{2mm} \\
{\hspace*{\fill}EPFL, Switzerland\hspace*{\fill}}\\
{\hspace*{\fill}\texttt{\{[first name].[surname]\}@epfl.ch}\hspace*{\fill}}
}
\begin{document}

\maketitle

\begin{abstract}
Neural tangent kernel (NTK) is a powerful tool to analyze training dynamics of neural networks and their generalization bounds. The study on NTK has been devoted to typical neural network architectures, but it is incomplete for neural networks with Hadamard products (NNs-Hp), e.g., StyleGAN and polynomial neural networks (PNNs). In this work, we derive the finite-width NTK formulation for a special class of NNs-Hp, i.e., polynomial neural networks. We prove their equivalence to the kernel regression predictor with the associated NTK, which expands the application scope of NTK. 
Based on our results, we elucidate the separation of PNNs over standard neural networks with respect to extrapolation and spectral bias. Our two key insights are that when compared to standard neural networks, PNNs can fit more complicated functions in the extrapolation regime and admit a slower eigenvalue decay of the respective NTK, leading to a faster learning towards high-frequency functions. Besides, our theoretical results can be extended to other types of NNs-Hp, which expand the scope of our work. Our empirical results validate the separations in broader classes of NNs-Hp, which provide a good justification for a deeper understanding of neural architectures.
\end{abstract}
\section{Introduction}
\label{sec:thesis_introduction} 
In deep learning theory, neural tangent kernel (NTK) \citep{NEURIPS2018_5a4be1fa} is a powerful analysis tool that links the training dynamics of neural networks (NNs) trained by gradient descent to kernel regression~\citep{NEURIPS2018_5a4be1fa,arora2019exact}. NTK provides a tractable analysis for several phenomena in deep learning, e.g., the global convergence of gradient descent~\citep{chizat2019lazy,du2019gradient,du2018gradient},
the inductive bias behind NNs~\citep{NEURIPS2019_c4ef9c39}, the spectral bias toward different frequency components~\citep{cao2019towards,choraria2022the}, the extrapolation behavior~\citep{xu2021how}, and the generalization ability~\citep{huang2020deep}. The study on the NTK has been devoted to typical NNs architectures, e.g., \fcnns{}~\citep{NEURIPS2018_5a4be1fa}, residual NNs~\citep{tirer2020kernel,huang2020deep}, convolutional NNs~\citep{arora2019exact}, graph NNs~\citep{du2019graph} and recurrent NNs~\citep{alemohammad2021the}.

Recently, NNs with Hadamard products (\pnns), e.g., StyleGAN~\citep{karras2018style}, polynomial neural networks~\citep{chrysos2020p}, non-local multiplicative networks~\citep{babiloni2021poly}, have received increasing attention due to their expressivity and efficiency over traditional NNs~\citep{chrysos2021conditional,890159,su2020convolutional}.
There have been several works attempting to demystify the success of \pnns{}. 
For instance, 
\citet{fan2021expressivity} prove that second-degree multiplicative interactions allow \pnns{} to enlarge the set of functions that can be represented exactly with zero error. \citet{choraria2022the} reveal that \pnns{} with second-degree multiplicative interactions yield a faster learning of high-frequency function during training in the NTK regime. Yet, the theoretical analysis of \pnns{} with high-degree multiplicative interactions is still unclear.
More importantly, when using NTK for analysis, only deriving the NTK matrix is not enough. The complete and rigorous proof is achieved by including the stability of empirical NTK during training and the equivalence to kernel regression. This is crucial to allow for NTK-based analysis of typical NNs~\citep{arora2019exact,tirer2020kernel} but is still missing for \pnns.

Polynomial neural networks (PNNs) \citep{chrysos2021conditional}, a special class of NNs-Hp~\citep{Jayakumar2020Multiplicative}, have showcased remarkable performance on a broad range of applications. 
As a step for analyzing NNs-Hp, in this work, we take PNNs as an example, derive the NTK for PNNs
with high-degree multiplicative interactions and present a rigorous proof for the equivalence to the kernel regression predictor.
This analysis enables us to further examine properties of PNNs in a theoretical perspective, e.g., the extrapolation~\citep{haley1992extrapolation,barnard1992extrapolation,xu2021how}.
Neural networks have demonstrated a stellar in-distribution performance but admit some weaknesses in extrapolating simple arithmetic problems~\citep{saxton2019analysing} or learning simple functions~\citep{haley1992extrapolation,sahoo2018learning}. 
Recently, \citet{xu2021how} theoretically and empirically point out that two-layer \fcnns{} with ReLU can only extrapolate to linear functions. The contrast on the in-/out of-distribution performance of standard NNs motivates us to scrutinize the extrapolation performance of PNNs. Additionally, studying the NTK of PNNs also allows us to investigate its spectral bias.

Overall, our main contributions and findings can be summarized as follows:
\begin{itemize}
    \vspace{-0.5em}
    \item We derive the NTK formulation for
    PNNs with high-degree multiplicative interactions, 
    and give a concrete bound of the widths requirement for convergence to the NTK at initialization, and stability during training, which allows us to bridge the gap among 
    PNNs trained via gradient descent and kernel regression predictor.
    \item
    We provably demonstrate the extrapolation behavior of PNNs as well as other NNs-Hp, including multiplicative filter networks and non-local multiplicative networks.
    Our findings highlight that PNNs can extrapolate to unseen data in a non-linear way.
    Besides, the spectral analysis of NTK of PNNs is also given for better understanding.
    PNNs admit a slower eigenvalue decay when compared to standard NNs, which leads to a faster learning towards high-frequency functions.
    \item
    We empirically show the advantage of NNs-Hp over standard NNs in learning commonly used functions, performing arithmetic extrapolation in real-world dataset, and conducting visual analogy extrapolation task.
    We scrutinize the role of multiplicative interactions in the task of learning spherical harmonics.
\end{itemize}
\fi
\section{Background}
\label{sec:thesis_related}
In this section, we establish the notation, provide an overview of the NTK, and summarize the most closely related work in \pnns{} as well as extrapolation. 
\subsection{Notation}
\label{sec:relatenotation} 
The core operators and symbols are summarized in Table~\ref{tbl:prodpoly_primary_symbols} at \cref{sec:appendixbackground}.
Vectors (matrices) are symbolized by lowercase (uppercase) boldface letters, e.g., $\boldsymbol{a}$, $\boldsymbol{A}$. 
We use the shorthand $[n] := \{1,2,\dots, n\} $ for a positive integer $n$.
We use $\lbrace \x_i\rbrace_{i=1}^{|\gX|} $and $\lbrace y_i\rbrace_{i=1}^{|\gX|}$  to present the input features and their labels of the training set $(\mathcal{X} ,\mathcal{Y})$ in a compact space, where ${|\gX|}$ denotes the cardinality. We symbolize by $K(\x, \x')$ the neural tangent kernel with respect to input $\x$ and $\x'$, the kernel matrix $\bm{K}\in\mathbb{R}^{|\mathcal{X}| \times |\mathcal{X}|}$ with ${K}^{(ij)}=K(\x_i, \x_j)$. Next, we denote by  $\bm{\rvtheta}_t$ the parameter vector, 
$\ell_{\text{2}}(\bm{\rvtheta}_t)$ the empirical training loss, and $\hat{\bm{K}}_t$ the empirical NTK Gram matrix at time step $t$. The following notation is used:
\begin{align}
\nonumber
&\ell_{\text{2}}(\bm{\rvtheta_t})
= \frac{1}{2} \sum_{(\x_i,y_i)\in
(\mathcal{X},\mathcal{Y})
} (f(\x_i;\bm{\rvtheta}_t) - y_i )^2,
\quad 
f(\bm{\rvtheta}_t) 
= \mathrm{vec}(\{ f(\x_i;\bm{\rvtheta}_t) \}_{\x_i\in\mathcal{X}}) \in \mathbb{R}^{|\mathcal{X}|}, 
\\ \nonumber
&\emJ(\bm{\rvtheta}_t) = \frac{\partial f(\bm{\rvtheta}_t)}{\partial \bm{\rvtheta}} \in \mathbb{R}^{|\mathcal{X}| \times |\bm{\rvtheta}|}
\,, \quad  \quad \quad \quad
\quad
\quad
\hat{\bm{K}}_t 
= \emJ(\bm{\rvtheta}_t) \emJ(\bm{\rvtheta}_t)^\top \in \mathbb{R}^{|\mathcal{X}| \times |\mathcal{X}|} .
\end{align}

\subsection{Neural tangent kernel}
\label{sec:relatedntk}
Neural networks (NNs) are relevant to the kernel method, under proper initialization~\citep{NIPS2016_abea47ba,2018gaussian}. \citet{NEURIPS2018_5a4be1fa} provably demonstrate the equivalence between the training dynamics by gradient descent and kernel regression induced by NTK when employing the $\ell_{\text{2}}$ loss. Below, we recall the exact formula regarding the NTK of $N$-layer ($N>2$) \fcnns{} with ReLU activation functions $\sigma$. The corresponding NTK $K(\x, \x') = K_N(\x, \x')$ could be computed recursively by:
\begin{equation*}
\begin{split}
& K_0(\x, \x') = \Sigma_0(\x, \x') = \x^\top \x'\,, \quad K_n(\x, \x') = \Sigma_n(\x, \x') + 2K_{\nmone}(\x, \x') \cdot \dot\Sigma_n(\x, \x') \,,
\label{eq:ntk_mlps}
\end{split}
\end{equation*}
$\forall n \in [N]$, where the covariance $\Sigma_n$ and its derivative $\dot\Sigma_n$ are defined as:
\begin{align*}
& \Sigma_n(\x, \x') = 2\E_{(u, v) \sim \mathcal N(0, \bm{\Lambda}_{i} )}[\sigma(u) \sigma(v)] \,, \quad \dot\Sigma_n(\x, \x')=\E_{(u, v) \sim \mathcal N(0, \bm{\Lambda}_{n} )}[\sigma'(u) \sigma'(v)]  \\
& 
	\quad \quad
	\quad \quad
	\quad \quad
	\quad \quad
	\bm{\Lambda}_{n}  = \left(
\begin{array}{cc}
	\Sigma_{\nmone}(\x, \x) & \Sigma_{\nmone}(\x, \x') \\ 
	\Sigma_{\nmone}(\x, \x') & \Sigma_{\nmone}(\x',\x')
\end{array}{}
\right), \forall n \in [N]\,.
\end{align*}
Furthermore, the aforementioned NTK is extended to residual NNs~\citep{tirer2020kernel,huang2020deep}, convolutional NNs~\citep{arora2019exact}, graph NNs~\citep{du2019graph}, and recurrent NNs~\citep{alemohammad2021the}. One of the roles of such kernel is to analyze the training behavior of the neural network in the over-parameterization regime~\citep{allen2019convergence,chizat2019lazy,du2019gradient,du2018gradient,zou2020gradient}. For instance, \citet{lee2019wide} showcase that NNs under the NTK parameterization trained via gradient descent of any depth evolve to linear models. Meanwhile, the inductive bias of convolutional networks, e.g.,  deformation stability of the images, has been studied in the NTK regime~\citep{NEURIPS2019_c4ef9c39}.
\looseness-1
\subsection{Neural networks with Hadamard product}
\label{sec:relatepnns}
The ideas of augmenting NNs with Hadamard products to allow multiplicative interactions can be traced back to at least~\citep{ivakhnenko1971polynomial} that investigate the learnable polynomial relationships. Most of the early work e.g., Group Method of Data Handling~\citep{ivakhnenko1971polynomial}, pi-sigma network~\citep{shin1991pi} do not scale well for high-dimensional signals.~\citet{9353253} factorize the weight of \pnns{} based on tensor decompositions to reduce the number of parameters. 
They exhibit how to convert popular networks, such as residual networks, and convolutional NNs to the form of \pnns. StyleGAN can be also considered as a special type of \pnns~\citep{chrysos2019polygan}. 
New efforts have recently emerged to improve the architecture of the network with Hadamard products~\citep{chrysos2021polynomial,babiloni2021poly,chrysos2021conditional}. In this work, we adopt the complementary approach and focus on the extrapolation as well as the spectral bias from a theoretical perspective. 

\subsection{Extrapolation}
\label{sec:relatedextrapolation}
The study of extrapolation properties of NNs dates at least back to the 90's~\citep{barnard1992extrapolation,Kramer1990DiagnosisUB}. Experimental results show poor performance of NNs in case of learning simple functions~\citep{barnard1992extrapolation}.
~\citet{browne2002representation} also suggest that \fcnns{} cannot extrapolate well and then illustrate how the representation of the input impact the extrapolation. \citet{xu2021how} provably present the extrapolation behavior of \fcnns{} and Graph neural networks. Specifically, they show that two-layer \fcnns{} with ReLU activation function extrapolate to linear function in extrapolation region. Our work exhibits that \pnns{} can learn high degree nonlinear function.
Apart from \fcnns,~\citet{martius2016extrapolation,sahoo2018learning} showcase a novel family of functions with linear mapping and a non-linear transformation, which allows to use sine and cosine as nonlinearities, enabling such networks to learn well in analytical expressions. Note that there exist multiplication units in EQL, which is similar to the multiplicative interactions in \pnns{}. Lastly, extrapolation is often considered in the context of out-of-distribution (OOD).
There are other types of OOD problems with specific setting among machine learning community~\citep{shen2021towards}.
Domain adaption assumes the source and the target domains lie in the same feature space but with different distributions~\citep{kouw2019review}, which differs from extrapolation. Another category of methodologies to solve the OOD generalization problem, called invariant learning, aims to discover high-level invariance feature from low-level observations through latent causal mechanisms~\citep{arjovsky2019invariant,rosenfeld2021the}. We believe our analysis can also encourage the usage of \pnns{} in these OOD problems.
\section{Analysis of polynomial neural networks}
 Our analysis admits the following structure: we firstly study the NTK of PNNs
 in \cref{sec:methodntk}, which allows us to conduct analysis towards extrapolation in \cref{sec:methodextra}, and spectral bias in  \cref{sec:methodspec}.
 In \cref{sec:mfn,sec:nonlocal},
 of the supplementary, we consider extensions beyond PNNs to  other families of NNs-Hp, e.g. multiplicative filter networks and non-local networks with Hadamard product.
\subsection{Neural tangent kernel}
\label{sec:methodntk}
 We now derive the NTK for PNNs, then we bridge the gap between the PNNs trained by gradient descent with respect to squared loss and the kernel regression predictor involving the NTK. The goal of such networks is to 
learn an $N$-\degorlay{} ($N\ge2$) polynomial expansion that outputs $\polyoutre \in \mathbb{R}$ with respect to the input $\polyinre\in \mathbb{R}^{d}$. 
For simplifying the proof, we consider the following formulation, which is a reparameterization version of PNNs~\citep{zhu2022controlling}. The output is given by:
\begin{align}
& \bm{\fm}_1 = \sqrt{\frac{2}{m}}\sigma(\W_1 \x),\;\;
\mlpout =  \sqrt{\frac{2}{m}}(\W_{N+1} \bm{\fm}_{N})
,\;\;
\bm{\fm}_n =\sqrt{\frac{2}{m}}\sigma\left(\W_n \x \right)*\bm{\fm}_{\nmone},\; n = 2, \ldots, N  \,,
\label{equ:ccp_repara_infinite_relu}
\end{align}
where $\sigma$ is the ReLU activation function, each element in $\mW_{N+1} \in \R^{ 1\times m }$ and $\mW_n \in \R^{m\times d}$, $ \forall n \in [N]$ is independently sampled from $\mathcal{N}( 0, 1)$.
Three remarks are in place: a) We multiply by the scaling factor $\sqrt{\frac{2}{m}}$ after each \degorlay{}
to ensure that the norm of the network output is preserved at initialization with infinite-width setting. b) ReLU is usually required to increase the performance of \pnns{} in experiments~\citep{9353253}. c) The original formulation before reparameterization that is used in practice can be founded in \cref{sec:primeronpinet}.
\begin{theorem} \label{thm:ndegree_infiniteNTK_formula}
The NTK of $N$-\degorlay{} PNNs, denoted by $K(\x, \x^{\prime})$, can be derived as:
\begin{equation}
\label{equ:ntkpinet_ndegree}
\begin{split}
K(\x, \x^{\prime})   =  & \; 
2N \cdot \langle \x,\xp  \rangle  \kappa_1(\x,\xp) (\kappa_2(\x,\xp))^{N-1} + 2 (\kappa_2(\x,\xp))^{N}\,,
\end{split}
\end{equation}
where $\kappa_1$ and $\kappa_2$ are defined by taking the random Gaussian vector $ \w \in \realnum^{d}$
\begin{equation}
\begin{split}
 & \kappa_1= \mathbb{E}_{\w \sim \mathcal{N}(\bm{0}, \sqrt{\frac{2}{m}} \cdot
   \bm{I} )}
 \left( 
 \dot{\sigma}(\w^{\top} \x )
 \cdot  
 \dot{\sigma}(\w^{\top} \xp)
 \right), 
  \kappa_2=
  \mathbb{E}_{ \w \sim \mathcal{N}(\bm{0},
  \sqrt{\frac{2}{m}} \cdot
 \bm{I} )} 
 \left( 
 \sigma(\w^{\top} \x )
 \cdot  
 \sigma(\w^{\top} \xp)
 \right)\,.
\end{split}
\label{equ:kappadefinition}
\end{equation}
\end{theorem}
The proof, which is provided in \cref{proof:ndegree_infiniteNTK_formula}, is based on the standard NTK calculations. Differently from the NTK of \fcnns{}, the existence of multiplicative interaction in PNNs induces the product form of multiple kernels.

Next, we provide the following theorem that gives a concrete requirement for the width of the networks that is sufficient for nonasymptotic convergence to the NTK at initialization,
\begin{theorem}
\label{thm:convergen_init}
(Convergence to the NTK).
Consider $N$-degree PNNs, and assume that the width $m\geq 2^{4N-2}\log^{2N-1} (2N/\delta)$ for any $\delta \in (0,1)$, then given two inputs $\x, \x^{\prime}$ on the unit sphere, with probability at least $1-\delta$ over the randomness of initialization, we have that
$$\left|\left\langle\nabla 
f
(\x), \nabla f(\x^{\prime})\right\rangle-
K(\x, \x^{\prime})
\right| \leq 
4N \rho e
\sqrt{\frac{\log(2N/\delta)}{m}}
,
$$
where $\rho = \sqrt{2}^{2N-1}\sqrt{8} e^{3}(2 \pi)^{1 / 4} e^{1 / 24}\left(e^{2/ e} (2N-1)/ 2\right)^{(2N-1) / 2}$.
\label{thm:subweibull}
\end{theorem}
{\bf Remark:} 
This result exhibits that the inner product of the Jacobian converges to the NTK at initialization, which has not been studied before for PNNs. This theorem allows us to further analyze the extrapolation of networks from the perspective of the NTK.
It should be noted that the term $\rho$ and the width are exponential with respect to the degree $N$, but the degree $N$ is not large in practice, e.g., at most 15 in \citet{chrysos2021conditional}.
Hence the bound is fair and reasonable.

The technical key issue of the proof is to provide probability estimates for the multiplication of several sub-exponential random variables. To this end, we rely on the concentration of 
sub-Weibull random variables~\citep{zhang2020concentration} to complete the proof, which is deferred to \cref{proof:subweibull}. 

Below, we show that under certain
conditions, the limiting NTK of PNNs stays constant when training with gradient descent using the squared loss.
\begin{theorem}[Stability of the NTK during training]
\label{thm:stability_of_ntk}
Given PNNs in \cref{equ:ccp_repara_infinite_relu}, assume $\lambda_{\min}(\bm{K})>0$ and the training data $(\mathcal{X},\mathcal{Y})$ in a compact space admitting $\x \neq \tilde{\x}$ for all $\x,\tilde{\x} \in \mathcal{X}$, then there exist some constants $R_0>0$, $M>1$, and $Q>1$ such that for every $m>M$, when minimizing the squared loss with gradient descent and sufficient small learning rate $\eta_0 < 2(\lambda_{\min}(\bm{K})+\lambda_{\max}(\bm{K}))^{-1}$, the following inequality holds with high probability over the random initialization of model parameters:
\begin{align}
&\sup_{t} \|
\hat{\bm{K}}_t
- \hat{\bm{K}}_0 \|_F \leq \frac{6Q^3 R_0}{\lambda_{\min}(\bm{K})\sqrt{m}}\,.
\label{Eq_thm_ntk_resnet_train_3}
\end{align}
\end{theorem}

{\bf Remark:}
\cref{Eq_thm_ntk_resnet_train_3} shows that $\hat{\bm{K}}_t \xrightarrow{
m
\xrightarrow{} \infty} \hat{\bm{K}}_0$.
Combining this with \cref{thm:convergen_init} that states $\hat{\bm{K}}_0 \xrightarrow{
m
\xrightarrow{} \infty} \bm{K}$, 
we have $\hat{\bm{K}}_t \xrightarrow{
m
\xrightarrow{} \infty} \bm{K}$. 
Thus, the equivalence to the kernel regression is established.
Note that \cref{thm:stability_of_ntk} is an extension from the corresponding theorem of NNs with residual connection \citep{tirer2020kernel} to 
PNNs. This property allows us to characterize the training process as kernel regression.

Regarding the proof of \cref{thm:stability_of_ntk}, we firstly introduce the norm control of the Gaussian weight matrices and then derive the local boundness and local Lipschitzness. The last step is to apply the induction rules over different time steps.
Details are presented in the \cref{proof:stability_of_ntk}.

\subsection{Extrapolation behavior}
\label{sec:methodextra}
Firstly, we provide the definition of extrapolation from~\citet{xu2021how} as follows.
\begin{definition}
    Extrapolation occurs when the domain of test samples is larger than the support of the training distribution.
\vspace{-2mm}
\end{definition}
{\bf Remark:}
The definition presented above is different from the one in \citet{balestriero2021learning} that claims extrapolation occurs when the test samples fall outside of the convex hull of the training set. Even though these definitions are not completely compatible, both definitions are suitable for our subsequent analysis.

The derived kernel in the previous section enables us to study how PNNs with ReLU activation trained by gradient descent extrapolates. Note that our theorem can also be extended to the raw PNNs without activation function.
\begin{theorem}[$\gamma$-degree extrapolation of $N$-degree PNNs]
\label{thm:ndegree_pinets_infiniteNTK_extra}
Suppose we train $N$-degree ($N\ge2$) 
PNNs
$f: \R^d \rightarrow \R$ with infinite-width on $\lbrace (\x_i, y_i)\rbrace_{i=1}^{|\mathcal{X}|}$, and the network is optimized with the squared loss in the NTK regime.
For any direction $\vv \in \R^d$ that satisfies $\| \bm v \|_2 =\max\lbrace\| \bm x_{i} \|^2\rbrace$,
let $\x_0 = t \vv$ and $\x = \x_0 + h \vv$ with $t>1$ and $h>0$ be the extrapolation
data points, the output $f(\x_0 + h \vv)$  follows a $\gamma$-degree ($\gamma \leq N$) function with respect to $h$.
\end{theorem}
Apart from PNNs, we also consummate~\citet{xu2021how} that consider the extrapolation of \fcnns{} with only two-layer. We provide the following generalized theorem for $N$-layer ($N>2$) \fcnns{}.
\begin{theorem}[Linear extrapolation of $N$-layer \fcnns{}]
\label{thm:ndegree_mlps_infiniteNTK}
 Suppose we train $N$-layer ($N\ge2$) \fcnns{} $f: \R^d \rightarrow \R$
 on $\lbrace (\x_i, y_i)\rbrace_{i=1}^{|\mathcal{X}|} $.
For any direction $\vv \in \R^d$ that satisfies $\| \bm v \|_2 =\max\lbrace\| \bm x _{i}\|^2\rbrace$,
$\x_0 = t \vv$ and $\x = \x_0 + h \vv$ with $t>1$ and $h>0$ are extrapolation data points, the output $f(\x_0 + h \vv)$  follows a linear function with respect to $h$.
\end{theorem}
We have already shown that PNNs extrapolate to a function with specific degree and are more flexible than \fcnns{}. However, only knowing the information of the degree of the extrapolation function is not enough. Naturally, we might ask under which condition PNNs can achieve successful extrapolation. Below, we build our analysis in the NTK regime and show how the geometry of the training set affects the behavior of PNNs.
\begin{theorem}[Condition for exact extrapolation of PNNs]
\label{thm:exact}
 Let $f_{\rho}(\bm x) = \x^{\top} \bm{\mBeta} \x$  be the target function with $\x \in \R^d$ and $\bm{\mBeta} \in \R^{d \times d}$.
 Suppose that $\lbrace \x_i \rbrace_{i=1}^{|\mathcal{X}|} $
 contains the orthogonal basis $\lbrace \bm{e}_i \rbrace_{i=1}^d$ and $\lbrace - \bm{e}_i \rbrace_{i=1}^d$. Then if we train two-degree 
 PNNs
 $f$ on $\lbrace (\x_i, f_{\rho}(\x_i))\rbrace_{i=1}^{|\mathcal{X}|} $ with the squared loss in the NTK regime, we have $f(\x) = \x^{\top} \bm{\mBeta} \x$ for all $\x \in \R^d$.
\end{theorem}
{\bf Remark:} This result only considers quadratic functions as our proof heavily relies on the construction of the feature map
of the NTK, which is harder for the high-degree case.

Due to constrained space, the proof of aforementioned theorems can be found in \cref{proof:ndegree_pinets_infiniteNTK_extra,proof:ndegree_mlps_infiniteNTK,proof:exact}.
\subsection{Spectral analysis}
\label{sec:methodspec}
In this section, we characterize the approximation properties of  $N$-degree 
PNNs in the in-distribution regime. By studying the spectral analysis in the form of a Mercer decomposition, we explicitly show the eigenvalues and eigenfunctions of
NTK.
We firstly introduce some notation. Denote by $\{Y_{k, j}\}_{j=1}^{N(d, k)}$ the spherical harmonics of degree $k$ in $d+1$ variables. $G_{k}^{(\gamma)}$ represents the Gegenbauer polynomials with respect to the weight function $x \mapsto (1-x^2)^{\gamma - \frac{1}{2}}$ and degree $k$. Finally, denote by $F(d, k) := \frac{2k + d - 1}{k} {k+d-2 \choose d-1}$. 

The following lemma enables us to connect spherical harmonics to Gegenbauer polynomials.
\begin{lemma}\cite[Theorem 4.11]{frye2012spherical}
For any $\bm{x}, \bm{x}' \in \mathbb{S}^d$, the $k$-degree spherical harmonics in $d+1$ variables satisfies:
$$
\sum_{j = 1}^{F(d, k)} Y_{k, j}(\bm{x}) Y_{k, j}(\bm{x}') = F(d, k) G_{k}^{(\frac{d-1}{2})}(\langle \bm{x}, \bm{x}' \rangle).
$$
\label{lemma_frye2012spherical}
\vspace{-4mm}
\end{lemma}
For any dot product Mercer kernel $K'$, 
denote by $(\mu_k)_{k=0}^\infty$the eigenvalues associated to the kernel, we can apply the following Mercer’s decomposition in the form of spherical harmonics, and using \cref{lemma_frye2012spherical} we obtain:
\begin{equation}
\begin{split}
    K'(\bm{x}, \bm{x'}) & = \sum_{k = 0}^{\infty} \mu_{k} \sum_{j = 1}^{F(d, k)} Y_{k, j}(\bm{x}) Y_{k, j}(\bm{x}')  = \sum_{k = 0}^{\infty} \mu_{k} F(d, k) G_{k}^{(\frac{d-1}{2})}(\langle \bm{x}, \bm{x}' \rangle),
\end{split}
\label{eqa:sph_eigen_poly}
\end{equation}
In order to study the decay rate of the eigenvalues, we can express the NTK as the product of multiple kernels and present the decay rate of the eigenvalues of PNNs.

\begin{theorem}
\label{theorem:pi_decay}
Consider PNNs with $N$-degree ($N\ge2$) multiplicative interactions and denote by $(\mu_k)_{k=0}^\infty$ the eigenvalues associated to the NTK. Then for $k \gg d\ 
$, we have $\mu_{k} = \Omega((N^2k)^{-d/2})$.
\end{theorem}
The proof can be found in~\cref{sec:proofofeigenvalue}. As a comparison, the decay rate for both deep fully-connected NNs and residual NNs is $\Omega((k)^{-d})$~\citep{belfer2021spectral}. Thus, we can see a slower decay rate when inserting Hadamard product into standard NNs, which leads to a faster learning towards high-frequency functions.
\section{Experiments}
Our experiments are organized as follows: We firstly showcase the extrapolation of \pnns{} in learning some common functions in \cref{sec:learningsimplefunction}. Next, we assess the extrapolation performance on non-synthetic dataset in \cref{sec:extrapolationrealdataset} and conduct the experiment in learning spherical harmonics in \cref{sec:learningspherical}.
Due to the constrained space, the extrapolation in a visual analogy task and the spectral bias in image classification task are deferred to \cref{sec:appendix_vaec} and \cref{sec:append_labelnoise}, respectively.
\begin{figure}[t]
    \centering
     \centering \includegraphics[width=0.8\textwidth]{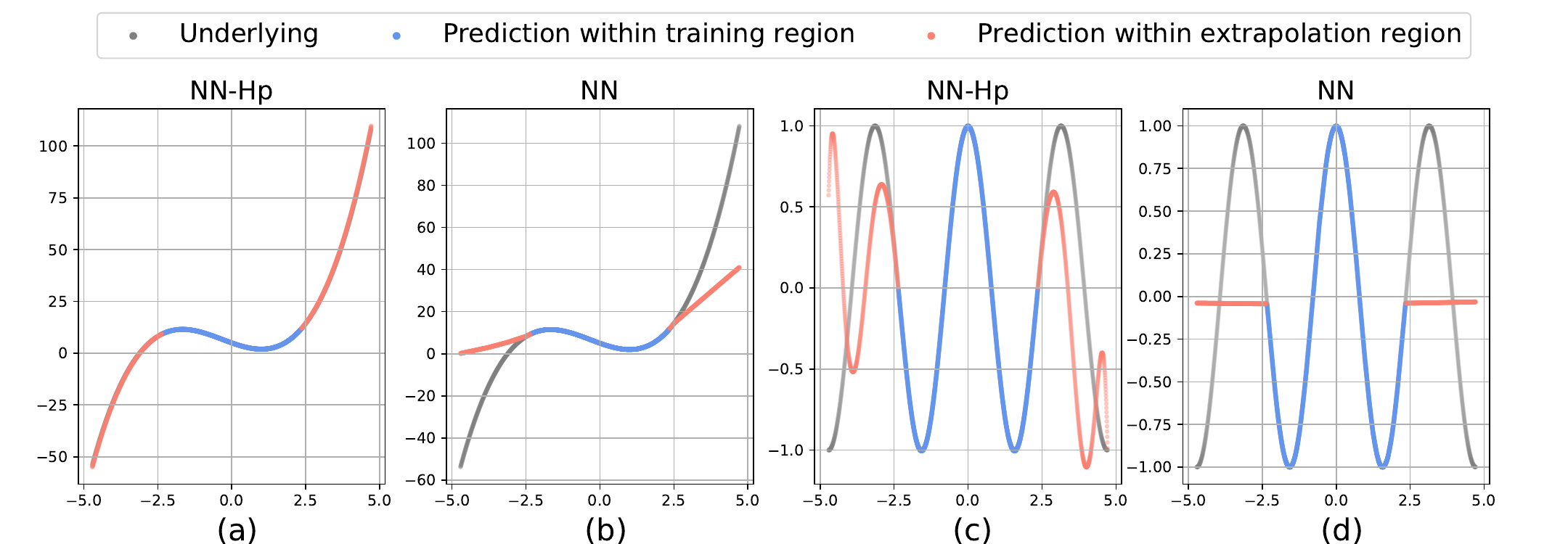}
\caption{\textbf{Extrapolation function.}
The blue curve indicates the training regime while the pink color symbolizes the extrapolation regime.
(a) and (b) show the fitting results towards $f_{\rho}(x)=x^3+x^2-10x+5$.
We can see that \nn{} extrapolates linearly without the Hadamard product (Hp) while \pnn{} is able to extrapolate to the underlying non-linear function nearly. (c) and (d) present the fitting results towards $f_{\rho}(x)=\text{cos}(2x)$. Notably, \pnn{} is more flexible to learn the non-linear function outsides the training region.}
\label{fig:exp_Extrapolationfunction}
\end{figure}
\label{sec:thesis_experiments}
\begin{figure}[htb]
    \centering
     \centering \includegraphics[width=0.8\textwidth]{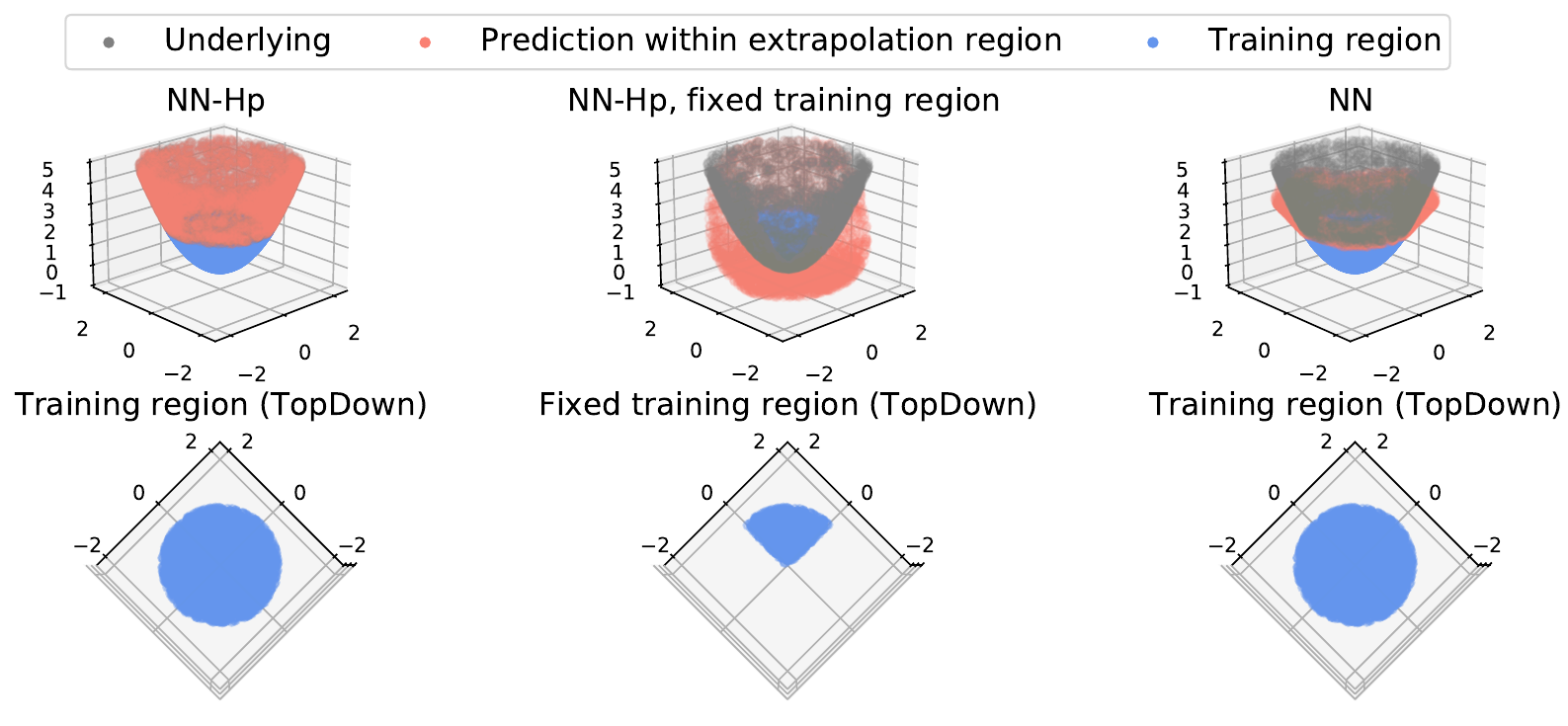}
\caption{Fitting results for the underlying function 
$f_{\rho}(\bm x)={(x^{(1)})}^2+{(x^{(2)})}^2$, where $x^{(1)}$ and $x^{(2)}$ are the first and second dimension of $\bm x \in R^2$. 
Blue points indicate the training regime, pink
points symbolize the extrapolation regime, gray points indicate the underlying function.
\textbf{Left}: We train \pnns{} with the training set containing support in all
directions, the network is able to extrapolate successfully. \textbf{Middle}: We train \pnn{} with the training set wherein two dimensions of the data are fixed to be positive
, \pnn{} fails to extrapolate. \textbf{Right}. We remove the Hadamard product of the network, which leads to linear extrapolation.}
\label{fig:exact}
\end{figure}

\subsection{Extrapolation in learning analytically-known functions}
\label{sec:learningsimplefunction}
These experiments aim to examine the  extrapolation behavior of \pnns{} in regression tasks. Our first experiment includes training the networks via the squared loss to fit several well-known and analytically-known underlying functions. During prediction, we sample data points beyond the training regime and observe the extrapolation performance. More details on implementation can be found in~\cref{sec:appendix_simplefunction}.
We set the target function as $f_{\rho}(x)=x^3+x^2-10x+5$ and use four-layer \fcnn.
As presented in \cref{fig:exp_Extrapolationfunction}(a) and \cref{fig:exp_Extrapolationfunction}(b), \fcnn{} extrapolates linearly while \pnn{} approximates better the extrapolation part of the underlying non-linear function, which are consistent with~\cref{thm:ndegree_pinets_infiniteNTK_extra} and~\cref{thm:ndegree_mlps_infiniteNTK}.

Learning $f_{\rho}(x)=\text{cos}(2x)$. We choose eleven-layer \fcnn{}.
The training set and testing set are the same as in the previous experiment. Observing \cref{fig:exp_Extrapolationfunction}(c) and \cref{fig:exp_Extrapolationfunction}(d), we find that \pnn{} is more flexible to learn the non-linear function outside the training region while \fcnn{} still extrapolates linearly.

Learning $f_{\rho}(\bm x)={(x^{(1)})}^2+{(x^{(2)})}^2$, where ${x^{(1)}}$ and ${x^{(2)}}$ is the first and second dimension of $\bm x \in R^2$.
In this task, we choose three-layer \nns{}. Each model is trained with different data distribution, i.e., the training set contains support in all
directions in the first case while two
dimension
of the training set are fixed to be positive in the second case. The result is visually depicted in \cref{fig:exact}, which shows that \pnns{} can achieve exact extrapolation if the training set contains support in all directions and thereby validates \cref{thm:exact}. On the other hand, \nns{} without the Hadamard product fail to extrapolate to the underlying function due to its linear extrapolation.
\subsection{Extrapolation in real-world dataset}
\label{sec:extrapolationrealdataset}
In this section, we assess the extrapolation performance beyond synthetic datasets.

\textbf{Variation of brightness.}
This experiment is conducted on two well-known grayscale image datasets: MNIST dataset~\citep{lecun1998gradient} and Fashion-MNIST dataset \cite{xiao2017fashion}.
For these two datasets, the original range of the pixel of each image is $[0, 1]$, we divided it by $10$ for the raw training set to construct the new one where the pixels range from $0$ to $0.1$. During extrapolation, we limit the range of the original testing set to $[0, r_{\text{max}}]$ through division, where
$r_{\text{max}} \in  \{ 0.1, 0.2, 0.3, ... , 1.0 \}$, as illustrated in the top two panels in~\cref{fig:darkness}. Then we feed these images into the trained network and evaluate the accuracy.
More details on the implementation can be found in~\cref{sec:appendix_darkness}. 
The accuracy is summarized in the two bottom plots of ~\cref{fig:darkness}.
We find that both networks achieve similar accuracy in the case $r_{\text{max}}=0.1$ while inserting Hadamard product (Hp) into \nn{} improves the performance during extrapolation.

\textbf{Arithmetic extrapolation.}
Now we turn to a more challenging task. As human we can usually extrapolate to arbitrarily large numbers in arithmetic. How do the neural networks perform during extrapolation? Following the setup of~\citet{bloice2020performing}, we use MNIST dataset, where there are 100 different two-image combinations of the digits $0\sim9$. We randomly pick up $90$ combinations as the training set and the remaining $10$ combinations as the extrapolation set.
This problem is treated as regression instead of classification for higher error tolerance following \citet{bloice2020performing}. In addition, if we design the network as a classifier, the number of the class will vary as the change of the splitting for the training set and testing set. The network only outputs one single discrete value. However, we still measure the accuracy by rounding the network output. five-layer \fcnns{} and convolution NNs are chosen as the baselines. 
For comparison, we implement 
\pnn{} with dense layers and 
\pnn{} with convolution layers, respectively. More details on the implementation can be found in the \cref{sec:appendix_addition}.
The results obtained by a three-fold cross validation are summarized in~\cref{tab:ari}, where we can see \pnn{} has a better extrapolation behavior in such more difficult task.
\begin{figure}[!htb]
    \centering
     \begin{subfigure}[b]{0.49\textwidth}
         \centering \includegraphics[width=0.9\textwidth]{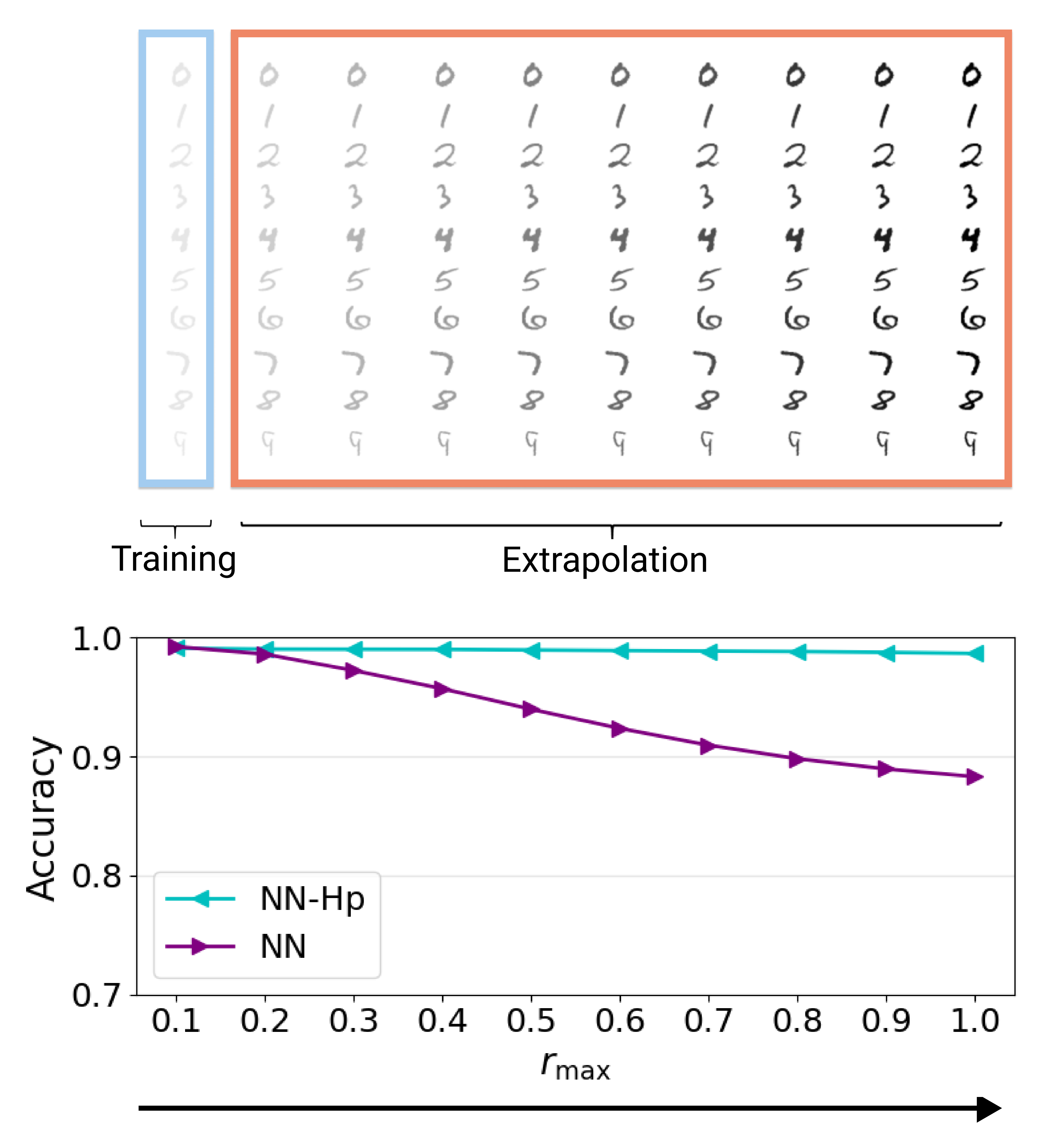}
         \caption{Examples and results on MNIST dataset.}
     \end{subfigure}
     \hfill
     \begin{subfigure}[b]{0.49\textwidth}
         \centering
     \includegraphics[width=0.9\textwidth]{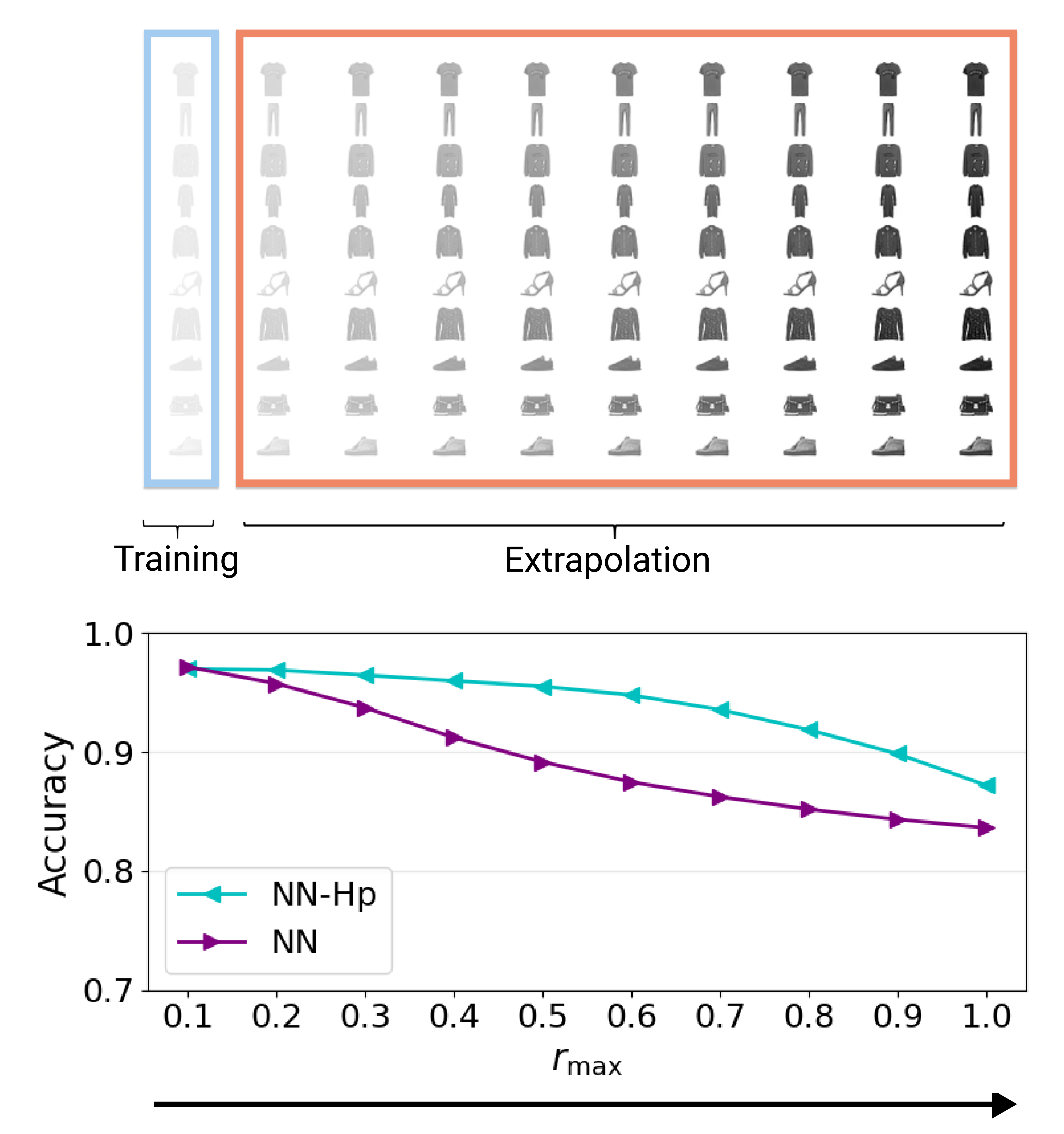}
         \caption{Examples and results on Fashion-MNIST dataset.}
     \end{subfigure}
\caption{The top two panels show the examples of extrapolation in MNIST dataset and Fashion-MNIST dataset. $r_{\text{max}}$ varies from $0.1$ to $1.0$. from left to right, indicating the variation of the darkness of the image. The bottom two panels show the accuracy as $r_{\text{max}}$ increasing. Both networks achieve similar accuracy in the case $r_{\text{max}}=0.1$ while inserting Hadamard product (Hp) into \nn{} improves the performance during extrapolation.} 
\label{fig:darkness}
\vspace{-8mm}
\end{figure}
\begin{figure}[htb]
    \centering \includegraphics[width=0.9\textwidth]{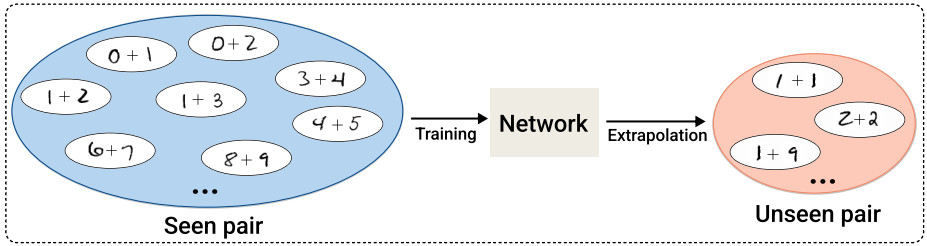}
\caption{A schematic illustration for the task of arithmetic extrapolation. 
}
\label{fig:demo_addition}
\end{figure}
\begin{table}[htb]
\centering
\caption{
Results in the task of arithmetic extrapolation, which aims to predict the target label with regression. 'Interpolation' indicates the accuracy in the seen pairs during training while 'Extrapolation' indicates the accuracy tested in those unseen pairs. 
Three ways are used for the network output: (a) Rounding, the output is rounded to the nearest integer. (b) Floor/ceiling,  A floor and ceiling function is applied for the output and if one of those equals to the ground truth label, we treat it as a correct prediction. (c) $\pm1$.  An error of $\pm1$ is allowed. 
We can observe that \pnn{} has a better extrapolation behavior compared with the baselines.
} 
\scalebox{0.94}{
\begin{tabular}{l@{\hspace{0.15cm}} l@{\hspace{0.15cm}} c@{\hspace{0.25cm}} c@{\hspace{0.25cm}} c@{\hspace{0.25cm}}c}
    \hline
Method&&Rounding&Floor/ceiling&$\pm1$
    \\
    \hline
    \multirow{2}{*}{\nn (Dense)} & Interpolation 
    &
    $0.980\pm0.002$
        &
    $ 0.999\pm 0.000$
            &
    $ 0.999\pm 0.000$
\\&   Extrapolation 
    &
    $ 0.436\pm 0.065$
        &
    $ \bm{0.805\pm 0.042}$
            &
    $ 0.887\pm 0.011$
    \\
    \hline
     \multirow{2}{*}{\pnn{} (Dense) 
     } 
     &   Interpolation 
    &
    $ 0.926 \pm 0.031$
        &
    $ 0.996 \pm 0.001$
            &
    $ 0.999\pm 0.000$
     &
\\&   Extrapolation 
    &
    $ \bm{0.554\pm 0.011}$
        &
    $ 0.802\pm0.010$
            &
    $ \bm{0.889 \pm 0.008}$
    \\
    \hline
    & & \multicolumn{4}{c}{} \\
\hline

     \multirow{2}{*}{\nn (Conv)} 
     &   Interpolation 
    &
    $ 0.945\pm 0.983$
    &
    $ 0.983\pm 0.021$
        &
    $ 0.994\pm 0.007$

     &
\\&   Extrapolation 
    &

    $ 0.617 \pm 0.103$
        &
    $ 0.918\pm 0.016$
            &
    $ 0.953\pm 0.006$
    \\
    \hline
     \multirow{2}{*}{\pnn{} (Conv)}
     &   Interpolation 
    &
    $ 
   0.991 \pm 0.002
    $
        &
    $0.998 \pm 0.000$
            &
    $ 0.999 \pm 0.000$
     &
\\&   Extrapolation 
    &
    $ \bm{0.825 \pm 0.109}$
        &
    $ \bm{0.948\pm 0.006}$
            &
    $ \bm{0.963\pm0.007}$
    \\
    \hline
\end{tabular}
}
\label{tab:ari}
\end{table}
\subsection{Spectral bias in  learning spherical harmonics}
\label{sec:learningspherical}
Here, we aim to learn the linear combinations of spherical
Harmonics where inputs are sampled from the uniform distribution on the unit sphere.
Our experiment follows the setup in \cite{choraria2022the}, which only considers \pnns{} with one Hadamard product.
The target function is defined by:
$
    f^*(\bm{x}) = \frac{1}{N(\mathcal{K})}\sum_{k \in K}A_k P_k(\langle \bm{x}, \zeta_k \rangle), \;\;\; \mathcal{K} = \{1, 3, 4, 5, 8, 12\},
    $
where $P_k(t)$ denotes the $k$-degree Gegenbauer polynomial, $\zeta_k$ are fixed vectors that are independently sampled from uniform distribution on unit sphere, and $N(\mathcal{K})$ is the normalizing constant. The error residuals with different $\mathcal{K}$ are compared during the training process. Implementation details are deferred to \cref{sec:appendix_harmonic}.
In this experiment, we 
show that increasing the number of multiplicative interactions can improve the rate of convergence of error, as presented in \cref{fig:harmonic}.
\begin{figure}[!tb]
    \centering
     \centering \includegraphics[width=\textwidth]{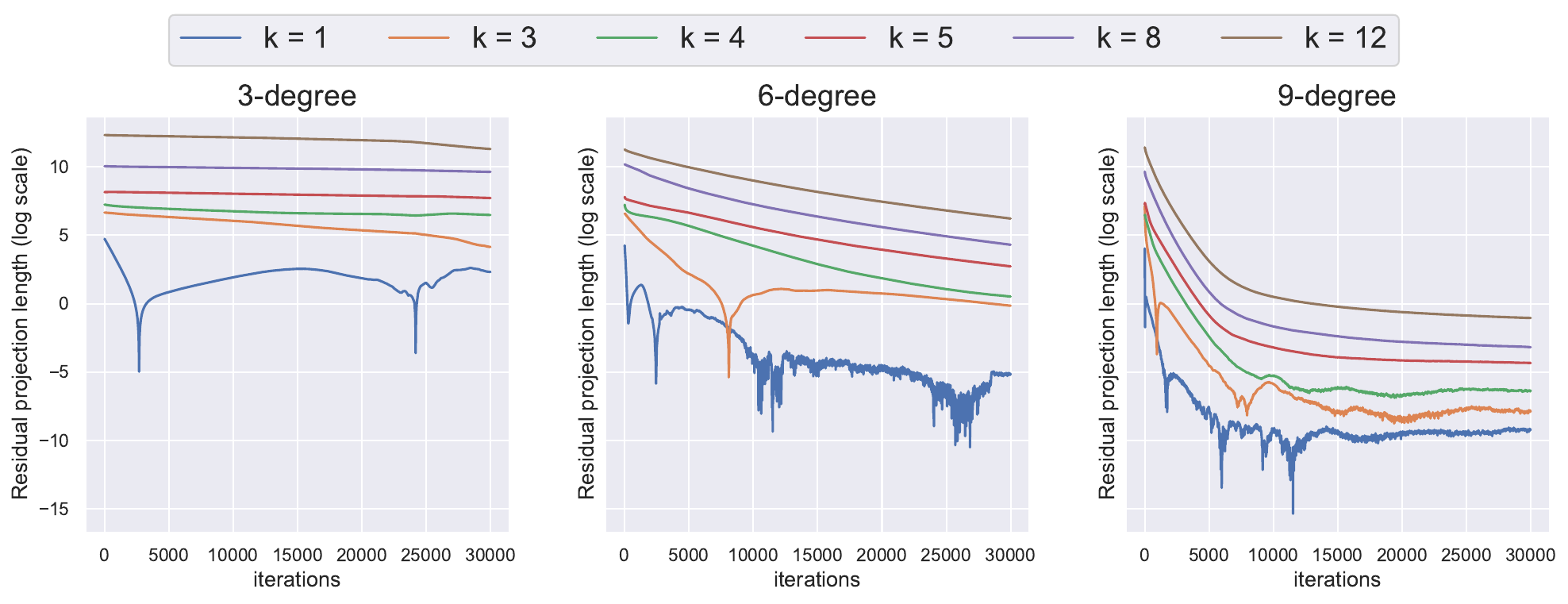}
\caption{
Comparison of convergence curve of error projection lengths for \pnns{} with $N$-degree multiplicative interactions, where $N \in  \{ 3,6,9 \}$  for different order harmonics with $K \in  \{ 1,3,4,5,8,12\}$ We can see the improvement for high-degree interactions in the rate of convergence of error.}
\label{fig:harmonic}
\end{figure}

\section{Conclusion}
\label{sec:conclusion}
This paper examines neural network with Hadamard product with a particular focus on polynomial neural networks from a theoretical perspective. The analysis of the NTK paves the way for knowing interesting properties of the networks, such as the extrapolation behavior and the spectral bias. Experimental results in learning analytically-known functions validate our hypothesis. We further conduct several experiments in real-world datasets and demonstrate the advantage of inserting Hadamard products into standard neural networks. We believe not only our framework provides a good justification for a deeper understanding of neural architecture, but it also lays the foundations to investigate other more complicated OOD problems such as domain adaption and invariant learning in future work.
\section*{Acknowledgements}
We are thankful to the reviewers for providing constructive feedback. This project has received support from the European Research Council (ERC) under the European Union’s Horizon 2020 research and innovation programme (grant agreement number 725594 - timedata). This work was sponsored by the Army Research Office and was accomplished under Grant Number W911NF-19-1-0404. This work was supported by Zeiss.
 This work has received funding from SNF project – Deep Optimisation of the Swiss National Science Foundation (SNSF) under grant number 200021\_205011.
 This work was supported by Hasler Foundation Program: Hasler Responsible AI (project number 21043).
\bibliography{bibliography}
\bibliographystyle{plainnat}

\newpage
\appendix
\newpage
\section*{Contents of the Appendix}
The Appendix is organized as follows:
\begin{itemize}
  \vspace{-0.5em}
    \setlength\itemsep{-0.1em}
    \item In \cref{sec:primeronpinet}, we provide a theoretical overview for a specific family of neural networks with Hadamard product, called~\pinet{}. The background in neural tangent kernel (NTK) is elaborated in \cref{sec:primerNTK}.
    \item The derivations and the proofs for the NTK of PNNs are further elaborated in \cref{sec:appendix_proodNTK}, including the formulation of NTK, the width requirement for the empirical kernel to converge to NTK at initialization, and the equivalent to kernel regression.
    \item  The proof for extrapolation-related theorem is included in \cref{sec:appendix_proofextrapolation}.
    \item  In \cref{sec:proofofeigenvalue}, we present the proof for the decay rate of eigenvalues of the NTK of PNNs.
    \item Details on the experiment are developed in \cref{sec:appendix_datail_experiment}.
    \item In \cref{sec:mfn,sec:nonlocal}, we extend this analysis beyond the class of polynomial neural networks.
    \item Societal impact and limitations of this work are discussed in \cref{sec:societalimpact,sec:pntk_limitations}, respectively.
\end{itemize}
\section{Theoretical background}
\label{sec:appendixbackground}
Our analysis in the main paper is built on a specific family of neural network with Hadamard product, called ~\pinet{}.
In \cref{sec:primeronpinet}, we will overview the theoretical background of \pinet{}. Next, in \cref{sec:primerNTK}, we briefly introduce the background on neural tangent kernel (NTK), which is used for analyzing the networks. We start with introducing some notation. The \textit{mode-$m$} vector product of a $M^{\text {th}}$ order tensor $\boldsymbol{\mathcal{A}}  \in \realnum^{J_{1} \times J_{2} \times \cdots \times J_{m-1} \times J_{m} \times J_{m+1} \times \cdots \times J_{M}}$ and a vector $\polyin \in \realnum^{J_{m}}$ is denoted by $\boldsymbol{\mathcal{A}}  \times_{m} \polyin \in \realnum^{J_{1} \times J_{2} \times \cdots \times J_{m-1} \times J_{m+1} \times \cdots \times J_{M}}$, resulting in $(M\minus1)^{\text {th}}$ order tensor:
$$
\left(\boldsymbol{\mathcal{A}}  \times_{m} \polyin\right)_{j_{1}, \ldots, j_{m-1}, j_{m+1}, \ldots, j_{M}}=\sum_{j_{m}=1}^{J_{m}} a_{j_{1}, j_{2}, \ldots, j_{M}} z_{j_{m}}\,.
$$
The \textit{mode-$m$} vector product of a tensor and multiple vectors is denoted as:
$$\boldsymbol{\mathcal{A}}  \times_{1} \polyin^{(1)} \times_{2} \polyin^{(2)} \times_{3} \cdots \times_{M} \polyin^{(M)} =\boldsymbol{\mathcal{A}}  \prod_{m=1}^{M} \times_{m} \polyin^{(m)}\,.$$
In  \textit{CANDECOMP/PARAFAC (CP) decomposition}~\citep{kolda2009tensor}, the tensor is decomposed into a sum of component rank-one tensors. The rank-$R$ CP decomposition of an $M^{\text{th}}$ order tensor $\boldsymbol{\mathcal{A}}$ is symbolized by
\begin{equation}\label{E:CP}
\boldsymbol{\mathcal{A}}=  \sum_{r=1}^R \polyin_r^{(1)}  \circ \polyin_r^{(2)}  \circ \cdots \circ \polyin_r^{(M)}\,,
\end{equation}
where $\circ$ is the outer product of vectors. 
\label{sec:content_appendix}
 \begin{table}[h]
\caption{Core symbols}
\label{tbl:prodpoly_primary_symbols}
\centering
\begin{tabular}{c | c | c}
\toprule
Symbol 	& Dimension(s) 		&	Definition \\
\midrule
$\sigma(\cdot)$, $ \dot{\sigma}(\cdot)$& -& ReLU function and its derivative
\\ \hline
$\odot$,  $ *$        &   -       & Khatri-Rao product, Hadamard product \\ \hline
$\ve_j$  &  $\R^m$  &  $j^{th}$ canonical basis vector of $\R^m$
\\ \hline
$n$, $N$ 	 & $\naturalnum$   &	Polynomial term degree and total degree  \\
\hline
$\polyinre$   & $\realnum^d$   
& Input to the network
\\
$\polyoutre$   & $\realnum$    
& Output of the network
\\
$\boldsymbol{\mathcal{\polyweight}}^{[n]}$   & $\realnum^{1 \times\prod_{i=1}^{n}{\times}_{i} d}$   
& Parameter tensor of the polynomial\\
$
\prodbias,
\boldsymbol{W}_n,
\boldsymbol{W}_{N+1}$
&
$\realnum$,
$
\realnum^{m\times d},
\realnum^{1\times m}$
& Learnable parameters \\
\hline
$\ell_{\text{2}}(\bm{\rvtheta})
= \frac{1}{2} \sum_{(\x_i,y_i)\in
(\mathcal{X},\mathcal{Y})
} (f(\x_i;\bm{\rvtheta}) - y_i )^2$ 
&
$\R$
&Empirical training loss
\\
 $\lbrace \x_i\rbrace_{i=1}^{|\gX|}$, $\lbrace y_i\rbrace_{i=1}^{|\gX|}$ && Features and labels of training set $(\mathcal{X} ,\mathcal{Y})$
\\ \hline
\end{tabular}
\end{table}
\subsection{A Primer on polynomial nets}
\label{sec:primeronpinet}
The goal of \pinet{} is to learn an $N$-degree polynomial expansion that outputs $\polyout \in \mathbb{R}^{d}$ with respect to the input $\vx \in \mathbb{R}^{d}$:
\begin{equation}
f(\x)
=\sum_{n=1}^{N}\left(\boldsymbol{\mathcal{\polyweight}}^{[n]} \prod_{j=2}^{n+1} \times{}_{j} \polyin\right)+\boldsymbol{\prodbias},
\label{equ:3_3}
\end{equation}
where  $\left\{\boldsymbol{\mathcal{\polyweight}}^{[n]} \in \mathbb{R}^{
1 \times
\prod_{i=1}^{n}{\times}_{i} d} \right\}_{n=1}^{N}$ and $\boldsymbol{\prodbias} \in \mathbb{R}
$ are learnable parameters.
Nevertheless, as the degree of the polynomial increases, the number of parameters in \cref{equ:3_3} grows exponentially. 
In order to improve the scalability, a coupled CP decomposition (CCP) with factor sharing is used to reduce parameters~\citep{kolda2009tensor,9353253}.
With CCP, all the weight tensors $\{\boldsymbol{\mathcal{{\polyweight}}}^{[n]} \}_{n=1}^{N}$  are jointly factorized by a coupled CP decomposition where
the factors between different degrees
are shared. For instance, the parameters of the third degree expansion follows:
\begin{itemize}
    \item First degree parameters: $\boldsymbol{\polyweight}^{[1]} = \mW_{4}\mW_{1}$.
    \item Second degree parameters:    $\boldsymbol{\polyweight}^{[2]}_{(1)} = \mW_{4}(\mW_{3} \odot \mW_{1}) + \mW_{4}(\mW_{2} \odot \mW_{1})$.
    \item Third degree parameters:  $\boldsymbol{\polyweight}^{[3]}_{(1)} = \mW_{4}(\mW_{3} \odot \mW_{2} \odot \mW_{1}) $.
\end{itemize}
Combining the aforementioned factorizations, the third degree expansion of \cref{equ:3_3} can be  expressed as:
\begin{equation}
\begin{split}
     f(\polyinre)= \prodbias + \mW_{4}\mW_{1}\polyinre + 
     \mW_{4}\Big(\mW_{3} \odot \mW_{1}\Big)(\polyinre\odot \polyinre) +
     \mW_{4}\Big(\mW_{2} \odot \mW_{1}\Big)(\polyinre\odot\polyinre)  +
    \\   \mW_{4}\Big(\mW_{3} \odot \mW_{2} \odot \mW_{1}\Big)(\polyinre \odot \polyinre \odot \polyinre).
\label{eq:polygan_recursive_gen_third_order_cccp}
\end{split}
\end{equation}

Next, we introduce the following lemma used to convert the Khatri-Rao products into a Hadamard product. 
\begin{lemma}\cite[]{chrysos2019polygan}
Given two sets of real-valued matrices $\{\boldsymbol{A}_{\nu} \in \mathbb{R}^{I_{\nu} \times K} \}_{\nu=1}^N$  and $\{\mB_{\nu} \in \mathbb{R}^{I_{\nu} \times L} \}_{\nu=1}^N$, the following
equality holds:
\begin{equation}
    (\bigodot_{\nu=1}^N \boldsymbol{A}_{\nu})^\top \cdot (\bigodot_{\nu=1}^N \mB_{\nu}) = (\boldsymbol{A}_1^\top \cdot \mB_1) * \ldots * (\boldsymbol{A}_N^\top \cdot \mB_N).
\end{equation}
\label{lemma:hadamard_kr1}
\end{lemma}

Applying the above lemma 
on \cref{eq:polygan_recursive_gen_third_order_cccp}, we obtain:
    \begin{equation}
    \begin{split}
    f(\polyinre)=
    \prodbias+\mW_{4}\left\{( \mW_{3} \polyinre) * \left[\left(\mW_{2}  \polyinre\right) *\left(\mW_{1} \polyinre\right)+\right.\right. \\
    \left.\left.\mW_{1} \polyinre\right]+\left(\mW_{2} \polyinre\right) *\left(\mW_{1} \polyinre\right)+\mW_{1} \polyinre\right\},
    \end{split}
    \end{equation}
which can be further converted to the following recursive relationship and generalized to arbitrary degree:
\begin{equation}
\begin{split}
&\vy_{n}=\left(\mW_{n} \vx\right) * \vy_{n-1}+\vy_{n-1}, \quad n = 2, \ldots, {N}
\\
&\vy_{1}=\mW_{1} \vx,
\quad
\quad
f(\polyinre)=\mW_{N+1}
\vy_{N}+\prodbias\,.
\label{equ:ccporiginal}
\end{split}
\end{equation}
The parameters $\boldsymbol{\prodbias} \in \mathbb{R}
$, $\boldsymbol{W}_{N+1} \in \mathbb{R}^{
1
\times
m}, \mW_n \in \mathbb{R}^{d \times m}$ for $n=1, \ldots, N$, are learnable. To simplify the proof, we follow ~\citet{zhu2022controlling} to reparameterize \cref{equ:ccporiginal} and obtain \cref{equ:ccp_repara_infinite_relu} (in the main body).
Apart from CCP, we can also factorize the polynomial networks by other decompositions. The recursive formula of NCP is:
\begin{equation}
\begin{split}
&\boldsymbol{y}_{n} = \left( \mW_{n} \vx\right) * \left( {\boldsymbol{F}_{n}^T\boldsymbol{y}_{n-1} + \boldsymbol{B}_{n}^T\boldsymbol{b}_{n}} \right), \quad n = 2, \ldots, {N}
\\&\boldsymbol{y}_{1}=\left(\mW_{1} \vx\right) *\left(\boldsymbol{B}_{1}^{T} \boldsymbol{b}_{1}\right), \quad f(\vx)=\mW_{N+1}\boldsymbol{y}_{N}+\boldsymbol{{\prodbias}}\,,
\end{split}
\end{equation}
where the parameters $\boldsymbol{\prodbias},\boldsymbol{W}_{N+1}, \mW_n, \mB_n, \vb_n, \mF_n$ are learnable.
The recursive formula of NCP-skip is:
\begin{equation}
\begin{split}
&\boldsymbol{y}_{n} = \left( \mW_{n} \vx\right) * \left( {\boldsymbol{F}_n^T\boldsymbol{y}_{n-1} + \boldsymbol{B}_n^T\boldsymbol{b}_n} \right)
+  \boldsymbol{D}_{n}\boldsymbol{{y}}_{n-1}, \quad n = 2, \ldots, {N}
\\&\boldsymbol{y}_{1}=\left(\mW_{1} \vx\right) *\left(\boldsymbol{B}_{1}^{T} \boldsymbol{b}_{1}\right), \quad f(\vx)=\mW_{N+1}\boldsymbol{y}_{N}+\boldsymbol{{\prodbias}}\,,
\end{split}
\end{equation}
where the parameters $\boldsymbol{\prodbias},\boldsymbol{W}_{N+1}, \mW_n, \mB_n, \vb_n, \mF_n, \mD_n$ are learnable.
\subsection{A Primer on NTK}
\label{sec:primerNTK}
In this section, we summarize how training a neural network by minimizing squared loss, i.e., $
\ell_{\text{2}}(\bm{\rvtheta}_t)
= \frac{1}{2} \sum_{(\x_i,y_i)\in
(\mathcal{X},\mathcal{Y})
} (f(\x_i;\bm{\rvtheta}_t) - y_i )^2$, via gradient descent can be characterized by the kernel regression predictor with NTK.

By choosing an infinitesimally small learning rate, we can obtain the following
gradient flow:
$$\frac{d \para_t}{dt} = -\nabla \ell_{\text{2}} (\para_t)\,.$$
By substituting the loss into the above equation and using the chain rule, we can find that the network outputs $f(\bm{\rvtheta}_t) 
= \mathrm{vec}(\{ f(\x_i;\bm{\rvtheta}_t) \}_{\x_i\in\mathcal{X}}) \in \mathbb{R}^{|\mathcal{X}|}$ admit the following dynamics:
\begin{align}
\frac{d f(\bm{\rvtheta}_t)  }{dt} = -\hat{\bm{K}}_t  (f(\bm{\rvtheta}_t)  - \bm{y})\,,
\label{equ:functionoutdynamic}
\end{align}
where
$\bm{y}= \mathrm{vec}
(\{y_i \}_{y_i\in\mathcal{Y}}) \in \mathbb{R}^{|\mathcal{X}|}
$, $\hat{\bm{K}}_t
=
\emJ(\bm{\rvtheta}_t) \emJ(\bm{\rvtheta}_t)^\top = 
\left(\frac{\partial f(\bm{\rvtheta}_t)}{\partial \bm{\rvtheta}}\right)
\left(\frac{\partial
f(\bm{\rvtheta}_t)}{\partial \bm{\rvtheta}}\right)^\top \in \mathbb{R}^{|\mathcal{X}| \times |\mathcal{X}|}$.
\citet{NEURIPS2018_5a4be1fa,arora2019exact}  have shown that for fully-connected neural networks, under the infinite-width setting and proper initialization, $\hat{\bm{K}}_t$ will keep constant during training and $\hat{\bm{K}}_0$ will converge to a fixed matrix $\bm{K}\in\mathbb{R}^{|\mathcal{X}| \times |\mathcal{X}|}$, where ${K}_{ij}=K(\x_i, \x_j)$ is the NTK value for the inputs $\x_i$ and $\x_j$. Then, based on $\hat{\bm{K}}_t=\hat{\bm{K}}_0=\bm{K}$, we can rewrite \cref{equ:functionoutdynamic} as:
\begin{align}
\frac{d f(\bm{\rvtheta}_t)  }{dt} = -\bm{K} (f(\bm{\rvtheta}_t)  - \bm{y})\,.
\end{align}
This implies the network output for any $\x \in \R^d$ can be calculated by the kernel regression predictor with the associated NTK:
\begin{align*}
 f (\x) &=\left(
 K(\x,\x_1), \cdots, 
 K(\x,\x_{|\gX|}) \right)
 \cdot \bm K^{-1} \bm{y}\,,
\end{align*}
where $K(\x,\x_i)$ is the kernel value between test data $\x$ and training data $\x_i$.
\section{Proofs of NTK}
\label{sec:appendix_proodNTK}
We derive the NTK of \pnns{} with multiple multiplicative interactions in \cref{proof:ndegree_infiniteNTK_formula}. Next, we prove the width requirement for the empirical kernel to converge to NTK at initialization in \cref{proof:subweibull}. Lastly, we analyze the training dynamics of \pnns{} under gradient descent in \cref{proof:stability_of_ntk}.

\subsection{Proof of \cref{thm:ndegree_infiniteNTK_formula}}
\label{proof:ndegree_infiniteNTK_formula}
Recall that the NTK is defined as the limit of the following inner product:
\begin{align*}
\bm K(\x, \x^{\prime})=\lim _{m \rightarrow \infty}\left\langle\nabla_{\bm\theta} f_{\bm\theta}(\x), \nabla_{\bm\theta} f_{\bm\theta}\left(\x^{\prime}\right)\right\rangle
\,,
\end{align*}
where $\bm\theta$ represents all the parameters.
Observing \cref{equ:ccp_repara_infinite_relu}, we can compute the gradient with respect to each weight and then sum up the inner products to obtain the NTK.
\\
\\
Below, we denote by $\bm{\preact}_n=\mW_n\x, \quad n \in [N]$ the pre-activation vectors and  $\bm\act_n$ the vectors after applying the element-wise ReLU activation to $\bm\preact_n$.
\\
\\
Firstly, we compute the contribution to the NTK w.r.t $\mW_1$, its corresponding derivative is as follows:
\begin{align}
\partial_{\mW_{1}} f(\x)
& \nonumber
=\sqrt{\frac{2}{m}} \sum_{j=1}^{m} W_{N+1}^{(j)} 
\left(
\prod \limits_{n=2}^{N}\sigma\left(\sqrt{\frac{2}{m}}\widetilde{\alpha}_{n}^{j}(\boldsymbol{x})\right)
\right)
\dot{\sigma}\left(\sqrt{\frac{2}{m}}\widetilde{\alpha}_{1}^{(j)}(\boldsymbol{x})\right) \partial_{\mW_{1}} \widetilde{\alpha}_{1}^{(j)}(\boldsymbol{x})
\\ & \nonumber
=\sqrt{\frac{2}{m}} \sum_{j=1}^{m} W_{N+1}^{(j)} 
\left(
\prod \limits_{n=2}^{N}\sigma\left(\sqrt{\frac{2}{m}}\widetilde{\alpha}_{n}^{(j)}(\boldsymbol{x})\right)
\right)
\dot{\sigma}\left(\sqrt{\frac{2}{m}}\widetilde{\alpha}_{1}^{(j)}(\boldsymbol{x})\right)  \partial_{\mW_{1}} \left( \ve_j^\top  \mW_N \vx\right)
\\ & \nonumber
=\sqrt{\frac{2}{m}} \sum_{j=1}^{m} W_{N+1}^{(j)} 
\left(
\prod \limits_{n=2}^{N}\sigma\left(\sqrt{\frac{2}{m}}\widetilde{\alpha}_{n}^{(j)}(\boldsymbol{x})\right)
\right)
\dot{\sigma}\left(\sqrt{\frac{2}{m}}\widetilde{\alpha}_{1}^{(j)}(\boldsymbol{x})\right) \left( \ve_j \vx^\top \right) \,.
\end{align}
The inner product of the derivative is:
\begin{align}
& \nonumber
\langle
\partial_{\mW_{1}} f(\x),
\partial_{\mW_{1}} f(\x^{\prime})
\rangle
\\
& \nonumber
= 
\frac{2}{m} 
\sum_{j, k=1}^{m}
W_{N+1}^{(j)}
W_{N+1}^{(k)}
\left(
\prod \limits_{n=2}^{N}\left(\frac{2}{m}
\sigma\left(\widetilde{\alpha}_{n}^{(j)}(\x)\right)
\sigma\left(\widetilde{\alpha}_{n}^{(j)}(\x^{\prime})\right)
\right)
\right)
\left(
\frac{2}{m}
\dot{\sigma}\left(\widetilde{\alpha}_{1}^{(j)}(\boldsymbol{x})\right) 
\dot{\sigma}\left(\widetilde{\alpha}_{1}^{(k)}
\left(\boldsymbol{x}^{\prime}\right)\right)\right)
\left
\langle\boldsymbol{e}_{j} \boldsymbol{x}^\top, \boldsymbol{e}_{k} \boldsymbol{x}^{\prime T}\right\rangle
\\
\nonumber
&=
\frac{2}{m} 
\sum_{j, k=1}^{m}
W_{N+1}^{(j)}
W_{N+1}^{(k)}
\left(
\prod \limits_{n=2}^{N}\left(\frac{2}{m}\sigma\left( \widetilde{\alpha}_{n}^{(j)}(\x)\right)
\sigma\left(\widetilde{\alpha}_{n}^{(k)}
(\x^{\prime})\right)\right)
\right)
\left(
\frac{2}{m}
\dot{\sigma}\left(\widetilde{\alpha}_{1}^{(j)}(\boldsymbol{x})\right) 
\dot{\sigma}\left(\widetilde{\alpha}_{1}^{(k)}
\left(\boldsymbol{x}^{\prime}\right)\right)\right)
\boldsymbol{x}^\top \boldsymbol{x}^{\prime} \delta_{j k}
\\
&=
\frac{2}{m} 
\sum_{j=1}^{m}
W_{N+1}^{(j)}
W_{N+1}^{(j)}
\left(
\prod \limits_{n=2}^{N}\left(\frac{2}{m}
\sigma\left({\widetilde{\alpha}}_{n}^{(j)}(\x)\right)
\sigma\left({\widetilde{\alpha}}_{n}^{(j)}(\x^{\prime})\right)
\right)
\right)
\left(
\frac{2}{m}
\dot{\sigma}\left(\widetilde{\alpha}_{1}^{(j)}(\boldsymbol{x})\right) 
\dot{\sigma}\left(\widetilde{\alpha}_{1}^{(j)}
\left(\boldsymbol{x}^{\prime}\right)\right)\right)
\boldsymbol{x}^\top \boldsymbol{x}^{\prime}.
\label{equ:innerproduct}
\end{align}
By the law of large numbers
and the fact that
$
\mathop{\E}_{ w \sim  \mathcal{N}(0,1 )} 
w^2 = 1
$, we obtain:
\begin{align}
\lim _{m \rightarrow \infty}
\langle
\partial_{\mW_{1}} f(\x),
\partial_{\mW_{1}} f(\x^{\prime})
\rangle
=
2 \langle \x,\xp  \rangle  \kappa_1(\x,\xp) 
(\kappa_2(\x,\xp) )^{N-1} \,,
\label{equ:contribution1}
\end{align}
where $\kappa_1$ and $\kappa_2$ are defined in \cref{equ:kappadefinition}.
Given that \cref{equ:ccp_repara_infinite_relu} is symmetric w.r.t $\lbrace \mW_n\rbrace_{n=1}^N $, the contributions of $\lbrace \mW_n \rbrace_{n=1}^N $ to the NTK are the same, we can trivially multiply \cref{equ:contribution1} by $N$.
\\
\\
Next, we compute the contribution to the NTK w.r.t $\mW_{N+1}$, its corresponding derivative is as follows:
\begin{align*}
\partial_{\mW_{N+1}} f(\x)
 = \sqrt{\frac{2}{m}}
\left( 
\sqrt{\frac{2}{m}}
 \sigma(\widetilde{\bm\alpha}_{N}) *
   \ldots
   *
   \sqrt{\frac{2}{m}}
   \sigma(\widetilde{\bm\alpha}_{1})
\right)\,.
\end{align*}
The inner product of the derivative is:
\begin{align*}
& 
\langle
\partial_{\mW_{N+1}} f(\x),
\partial_{\mW_{N+1}} f(\x^{\prime})
\rangle
=
    \frac{2}{m} 
    \sum_{j=1}^{m}
    \left(
    \prod \limits_{n=1}^{N}
    \left(
    \frac{2}{m}
\sigma\left(\widetilde{\alpha}_{n}^{(j)}(\boldsymbol{x})\right)
\sigma\left(\widetilde{\alpha}_{n}^{(j)}
    (\x^{\prime})\right)
    \right)
    \right)\,.
\end{align*}
By the law of large numbers: 
\begin{align}
\nonumber
&\lim _{m \rightarrow \infty}
\langle
\partial_{\mW_{N+1}} f(\x),
\partial_{\mW_{N+1}} f(\x^{\prime})
\rangle
\\&= \nonumber
 2 \left(\mathbb{E}_{ \w \sim \mathcal{N}(\bm{0},
  \sqrt{\frac{2}{m}} \cdot
 \bm{I} )} 
 \left( 
 \sigma(\w^{\top} \x )
 \cdot  
 \sigma(\w^{\top} \xp)
 \right)\right)^N
\\
    & = 2\cdot (\kappa_2(\x,\xp))^{N}\,.
\label{equ:contribution2}
\end{align}
The proof is completed by multiplying \cref{equ:contribution1}
by $N$ and adding by \cref{equ:contribution2}.

\subsection{Proof of \cref{thm:subweibull}}
\label{proof:subweibull}
Before we prove \cref{thm:subweibull}, we need to tackle a technical key issue: how to provide probability estimates for the multiplication of several sub-exponential random variables.
To this end, we introduce sub-Weibull random variables in the following definition, which allows for our case and still admits exponential decay tails.
\begin{definition}[Sub-Weibull distributions~\cite{zhang2020concentration}]
\label{def:subweibull}
Given positive constants $a, b$, a random variable $X$ is sub-Weibull if it satisfies
$ {P}(|X| \ge x) \le ae^{-bx^{\theta}}$, where $\theta>0$ is the order.
\end{definition}
{\bf Remark:} The classical sub-exponential and sub-Gaussian random variables are sub-Weibull by taking $\theta=1$ and $\theta=2$, respectively.

Based on the definition above, we have the following concentration inequality on sub-Weilbull random variables.
\begin{lemma}
[Sub-Weibull Concentration~\cite{zhang2020concentration}]
\label{lemma:SubWeibullConcentration}
Given some $\theta > 0$, if $\{X_k\}_{k=1}^K$ are independent mean zero random variable
such that the sub-Weibull norm $
\|{X_k}\|_{\psi_{\theta}} < \infty$ for all $1\le k\le K$, then for any weight vector $\bm w= (w_1, \ldots, w_K)
\in\mathbb{R}^n$,
for any $ \zeta \in (0,1)$
one has 
\begin{equation*}
P(|\sum{}_{k = 1}^K {{w_k}{X_k}} | \ge 2e\rho(\theta ){\left\|\bm b \right\|_2}\{ \sqrt \zeta  + {D}(\theta ){\zeta^{1/\theta }}\} ) \le 2{e^{ - \zeta}}\,,
\end{equation*}
where $\bm b = (w_1\|{X_1}\|_{\psi_{\theta}}, \ldots, w_K\|{X_K}\|_{\psi_{\theta}})^\top\
$, $D(\theta) := \frac{4^{1/\theta}}{\sqrt{2}\norm{\bm b}_2}\times\begin{cases}\norm{\bm b}_{\infty},&\mbox{if }\theta < 1,\\
{4e\|{\bm b}\|_{\frac{\theta }{{{\rm{1}} - \theta }}}}/{C(\theta)},&\mbox{if }\theta \ge 1,\end{cases}$
\begin{center}
$
\rho(\theta) ~:=~ \max\{ \sqrt{2}, 2^{1/\theta}\} \times \begin{cases}\sqrt{8}e^3(2\pi)^{1/4}e^{1/24}(e^{2/e}/\theta)^{1/\theta},&\mbox{if }\theta < 1,\\
4e + 2(\log 2)^{1/\theta},&\mbox{if }\theta \ge 1.
\end{cases}
$
\end{center}
\end{lemma}

\begin{proof}[Proof of \cref{thm:subweibull}]
Recall from \cref{equ:innerproduct}:
\begin{align}
&\nonumber
\langle
\partial_{\mW_{1}} f(\x),
\partial_{\mW_{1}} f(\x^{\prime})
\rangle
\\&=
\nonumber
\frac{1}{m} 
\sum_{k=1}^{m}
\underbrace{
2
W_{N+1}^{(k)}
W_{N+1}^{(k)}
\left(
\prod \limits_{n=2}^{N}\left(\frac{2}{m}
\sigma\left({\widetilde{\alpha}}_{n}^{(k)}(\x)\right)
\sigma\left({\widetilde{\alpha}}_{n}^{(k)}(\x^{\prime})\right)
\right)
\right)
\left(
\frac{2}{m}
\dot{\sigma}\left(\widetilde{\alpha}_{1}^{(k)}(\boldsymbol{x})\right) 
\dot{\sigma}\left(\widetilde{\alpha}_{1}^{(k)}
\left(\boldsymbol{x}^{\prime}\right)\right)\right)
\boldsymbol{x}^\top \boldsymbol{x}^{\prime}
} _{:=X_k}\,,
\end{align}
where $\bm\preact_n=\mW_n\x, \quad n = 1, \dots, N$ represent the pre-activation.
Firstly, we centralize $X_k$ and denote by $\widehat{X_k}$ as follows:
\begin{align*}
\widehat{X_k}=X_k-
2 \kappa_1(\x,\xp) 
(\kappa_2(\x,\xp) )^{N-1}\boldsymbol{x}^\top \boldsymbol{x}^{\prime}
\end{align*}

Since all the weight matrices $\mW_{n}, \quad n = 1, \dots, N+1$ are Gaussian,\
$\sigma\left({\widetilde{\alpha}}_{n}^{(k)}(\x)\right)$, $\sigma\left({\widetilde{\alpha}}_{n}^{(k)}(\x^{\prime})\right)$, and $\dot{\sigma}\left(\widetilde{\alpha}_{1}^{(k)}(\boldsymbol{x})\right)
\dot{\sigma}\left(\widetilde{\alpha}_{1}^{(k)}
\left(\boldsymbol{x}^{\prime}\right)\right)$ are sub-Gaussian random variables over the randomness of $\mW_{n}$, 
Thus, $\{\widehat{X_k}\}_{k=1}^m$ is zero mean sub-Weibull random variable with order $\theta={2}/{(2N+1)}$.
Plugging $w_1, \ldots, w_k = 1/m$ into \cref{lemma:SubWeibullConcentration}, we get
\begin{align*}
&\rho \left( \frac{2}{2N+1} \right)= \sqrt{2}^{2N-1}\sqrt{8} e^{3}(2 \pi)^{1 / 4} e^{1 / 24}\left(e^{2/ e} (2N-1)/ 2\right)^{(2N-1) / 2}
\\&
\|\bm b\|_2=\frac{1}{m}\|(\|{\widehat{X}_1}\|_{\psi_{\theta}}, \ldots, \|{\widehat{X}_m}\|_{\psi_{\theta}})\|_2
, \quad \quad 
D\left(\frac{2}{2N+1}\right) =\frac{4^{(2N+1)/2}\norm{\bm b}_{\infty} }{\sqrt{2}\norm{\bm b}_2} \,.
\end{align*}
Suppose that the width satisfies $m\geq 2^{4N-2}\log^{2N-1} (2N/\delta)$, then for any $\delta \in (0,1)$,
with probability at least $1-({\delta}/{N})$ over the randomness of initialization, we have
$$
\left|
\langle
\partial_{\mW_{1}} f(\x),
\partial_{\mW_{1}} f(\x^{\prime})
\rangle
-
2 \langle \x,\xp  \rangle  \kappa_1(\x,\xp) 
(\kappa_2(\x,\xp) )^{N-1} 
\right|
\leq 
4N\rho \left( \frac{2}{2N+1} \right)e
\sqrt{\frac{\log(2N/\delta)}{m}}
.
$$
Note that we only consider one weight matrix above, by applying the union bound, with probability at least $1-\delta$ over the randomness of initialization, we have:
\begin{align*}
\left|\left\langle\nabla f(\x), \nabla f(\x^{\prime})\right\rangle-
K(\x, \x^{\prime})
\right| \leq 
4N\rho \left( \frac{2}{2N+1} \right)e
\sqrt{\frac{\log(2N/\delta)}{m}}
\,.
\end{align*}
\end{proof}
\subsection{Proof of \cref{thm:stability_of_ntk}}
\label{proof:stability_of_ntk}
Before starting the proof, we introduce the following lemmas that are used to analyze
the random initialization of the weight matrices $\W_n \in \realnum^{m \times d}, \forall n \in [N]$ and $\W_{N+1} \in \realnum^{1 \times m}$.
\begin{lemma}\cite[Corollary 5.35]{vershynin2010introduction}
\label{thm:inequality_initialization}
For a weight matrix $\W \in \R^{m \times d}$ 
where each element is sampled independently from $\mathcal{N}(0, 1)$, for every $\zeta \geq 0$, with probability at least $1-2\mathrm{exp}(-\zeta^2/2)$ one has:
\begin{align*}
\sqrt{m} - \sqrt{d} - \zeta \leq \lambda_{min}(\W) \leq \lambda_{\max}(\W) \leq \sqrt{m} + \sqrt{d} + \zeta,
\end{align*}
where $\lambda_{\max}(\mW) $ 
and $\lambda_{\min}(\mW)$ represents the largest and smallest singular value of $\mW$, respectively.
\end{lemma}

Next, we will show the local boundness and the local Lipschitzness of the Jacobian.
We use $\|\cdot\|_\text{F}$ and $\|\cdot\|$ to represent the Frobenius norm and spectral norm of a matrix, respectively. The Euclidean norm of a vector is symbolized by $\|\cdot\|_2$.
\begin{lemma}
\label{thm:local_lip}
Consider the $N$-degree \pnns{} in \cref{equ:ccp_repara_infinite_relu}, assume the input $\vx \in \realnum^{d}$ is bounded $\|\vx\|_2\leq 1$, then there exists $\gamma_1>0$, $\gamma_2>0$ (both are independent of the width $m$) such that for every $r>0$, $ \delta \in (0,1)$, 
$m\geq \left(r+\sqrt{d}+2\log{\left((2N+2)/\delta)\right)}\right)^2$, with probability at least $1-\delta$
over the random initialization, the following holds for all $\bm{\rvtheta},\widetilde{\bm{\rvtheta}} \in D(\bm{\rvtheta},r):=\{\bm{\rvtheta} : \|\bm{\rvtheta}-\bm{\rvtheta}_0\|_2 \leq r\}$
\begin{align}
&\| \emJ(\bm{\rvtheta}) \|_\text{F} \leq \gamma_1,
\label{Eq_local_bound}
\\ 
&\| \emJ(\bm{\rvtheta}) - \emJ(\widetilde{\bm{\rvtheta}}) \|_\text{F} \leq \gamma_2 \| \bm{\rvtheta} - \widetilde{\bm{\rvtheta}} \|_2.
\end{align}
\end{lemma}
\vspace{-5mm}
\begin{proof}[Proof of \cref{thm:local_lip}]
Based on~\cref{thm:inequality_initialization} and union bound, when $m\geq \left(r+\sqrt{d}+2\log{\left((2N+2)/\delta)\right)}\right)^2$, with probability at least $1-\delta$ for any $ \delta \in (0,1)$,
the following inequalities hold for all $n = 1, \dots, N$ simultaneously:
\begin{align*}
&\left\| {{{{\mW_n}}}} \right\|
\leq
\sqrt{m} + \sqrt{d} + \delta  \leq 2\sqrt{m}, 
\\&
\left\| {\mW_{N+1}} \right\|
\leq
\sqrt{m} +1 + \delta
\leq 2\sqrt{m}, 
\\
&\left\| {\widetilde{\mW_n}} \right\|
=
\left\| {\mW_n}+\Delta{\mW_n} \right\| 
\leq 
\left\| {\mW_n}\right\| + \left\|\Delta{\mW_n} \right\|
\leq
\sqrt{m} + \sqrt{d} + \delta + \left\|\Delta{\mW_n} \right\|_\text{F}
\\&
\quad \quad \quad
\leq
\sqrt{m} + \sqrt{d} + \delta + \left\|\Delta \bm\rvtheta \right\|_2 \leq \sqrt{m} + \sqrt{d} + \delta + r
\leq 2\sqrt{m}.
\end{align*}
Below, we abbreviate the description of probability and the width requirement since the following events rely on the same random initialization of the weight matrices.
The following shorthand notations are made:
\begin{align}
\nonumber
{\widetilde{\boldsymbol{T}_n}} = 
\left( 
\sqrt{\frac{2}{m}}
 \sigma(\widetilde{\mW}_{n}\vx) *
   \ldots
   *
   \sqrt{\frac{2}{m}}
   \sigma(\widetilde{\mW}_{1}\vx)
\right),
\\ \nonumber
{\boldsymbol{T}_n} = 
\left( 
\sqrt{\frac{2}{m}}
 \sigma({\mW}_{n}\vx) *
   \ldots
   *
   \sqrt{\frac{2}{m}}
   \sigma({\mW}_{1}\vx)
\right)\,.
\end{align}
Firstly, we prove the local boundness (\cref{Eq_local_bound}).
Given that \cref{equ:ccp_repara_infinite_relu} is symmetric w.r.t $\lbrace \mW_i \rbrace_{i=1}^N $, we start with calculating the bound with respect of one of the parameter matrix $\mW_N$. The derivate is:
\begin{align*}
  {\partial _{{\mW_N}}}f(\vx)&=
  \left( \frac{2}{{\sqrt m }}{\mW_{N + 1}} * \frac{2}{{\sqrt m }}
  \dot{\sigma}({{\mW_N}x}) *
  \frac{2}{{\sqrt m }}
  {\sigma}({{\mW_{N-1}}x})
   * ... *
   \frac{2}{{\sqrt m }}
  {\sigma}({{\mW_{1}}x})  
  \right)
  {\vx^\top}
  \\ &=
  \left( \frac{2}{{\sqrt m }}{\mW_{N + 1}} * \frac{2}{{\sqrt m }}
  \dot{\sigma}({{\mW_N}x}) *
    {\boldsymbol{T}_{N-1}}
  \right)
  {\vx^\top}.
\end{align*}
Its Frobenius norm satisfies with probability:
\begin{align*}
  \left\| {{\partial _{{\mW_N}}}f(\vx)} \right\|_\text{F} 
  &= \left\| {\left( {\frac{2}{{\sqrt m }}{\mW_{N + 1}} * \frac{2}{{\sqrt m }}\dot{\sigma} \left( {{\mW_N}x} \right) * {\boldsymbol{T}_{N-1}}} \right){\vx^\top}} \right\|_\text{F} 
  \\&\leq
  \left\| { {\frac{2}{{\sqrt m }}{\mW_{N + 1}} * \frac{2}{{\sqrt m }}\dot{\sigma} \left( {{\mW_N}x} \right) * {\boldsymbol{T}_{N-1}}} } \right\|_2
  \\&\leq
    \left\| { {\frac{2}{{\sqrt m }}{\mW_{N + 1}}} } \right\|_2\left\| {\frac{2}{{\sqrt m }}\dot{\sigma} \left( {{\mW_N}\vx} \right)} \right\|_2\left\| {{\boldsymbol{T}_{N-1}}} \right\|_2
    \\&\leq  4\left\| {\frac{2}{{\sqrt m }}\dot{\sigma} \left( {{\mW_N}\vx} \right)}
   \right\|_2\left\| {{\boldsymbol{T}_{N-1}}} \right\|_2
   \\&\leq
  8 \left\| {{\boldsymbol{T}_{N-1}}} \right\|_2 
   \leq2^{2N+1},
\end{align*}
   where the second and the third inequality use the Cauchy–Schwarz inequality, the last inequality is based on the upper bound of $\left\| {{\boldsymbol{T}_{N-1}}} \right\|_2$, which holds with probability:
\begin{align}
  \left\| {\boldsymbol{T}_{N-1}} \right\|_2 
  &= 
  {\left\| {\frac{2}{{\sqrt m }}\sigma \left( {{\mW_{N-1}}x} \right)*...*\frac{2}{{\sqrt m }}\sigma \left( {{\mW_{1}}x} \right)} \right\|_2}
   \leq 
   {\left\| {\frac{2}{{\sqrt m }}\sigma \left( {{\widehat{\mW}}x} \right)} \right\|_2}^{N-1} 
   \\ &\leq 
   \left( \frac{{2}}{{\sqrt m }}
   {\left\| {{\widehat{\mW}}} \right\|}\right)^{N-1}  
  \leq 
  {\left| {\frac{{2}}{{\sqrt m }}{2}\sqrt m } \right|}^{N-1} = 4^{N-1},
  \label{equ:t_n_minus_1}
\end{align}
where $\widehat{\mW}
\in \realnum^{m \times d}$ is sampled from $\mathcal{N}(\bm 0, \bm 1)$.
Now we consider all the weight matrices except $\W_{N+1}$ that is not trained. We have the following bound with probability:
\begin{align*}
  \| \emJ(\bm{\rvtheta}) \|_\text{F} 
  =
  \sqrt{
 \sum \limits_{\vx \in \mathcal{X}} 
 \left(
 \sum \limits_{n=1}^{N} \left \| \frac{\partial  f(\vx;\bm \theta)}{\partial \mW_{n}} \right \|_\text{F}^2
 \right)
 }
 \leq  2^{2N+1} \sqrt{N|\mathcal{X}|}=\gamma_1,
\end{align*}
where $\gamma_1$ does not depend on the width $m$.
This completes the first part of the proof.

Next, we prove the local Lipschitzness (\cref{Eq_local_bound}). Similarly, since \cref{equ:ccp_repara_infinite_relu} is symmetric w.r.t $\lbrace \mW_i \rbrace_{i=1}^N $, firstly, we calculate the perturbation with respect of one of the parameter matrix $\mW_N$. The following inequality holds with probability:
\begin{align}
 &  {\left\| {{\partial _{{{\widetilde \mW}_N}}}f(\vx) - {\partial _{{\mW_N}}}f(\vx)} \right\|_\text{F}}
  \\& =
  \left\| {\left( {\frac{2}{{\sqrt m }}{\mW_{N + 1}} * \frac{2}{{\sqrt m }}\dot{\sigma}  \left( {{{\widetilde \mW}_N}\vx} \right) * {\widetilde{\boldsymbol{T}}_{N-1}}} \right)} \right.{\left. {{\vx^\top} - \left( {\frac{2}{{\sqrt m }}{\mW_{N + 1}} * \frac{2}{{\sqrt m }}\dot{\sigma}    \left( {{\mW_N}x} \right) * {\boldsymbol{T}_{N-1}}} \right){\vx^\top}} \right\|_\text{F}} 
  \\&\leq \label{equ:lip_ine1}
  {\left\| {\frac{2}{{\sqrt m }}{\mW_{N + 1}} * \frac{2}{{\sqrt m }}\dot{\sigma}    \left( { {{\widetilde \mW}_N}\vx} \right) * {\widetilde{\boldsymbol{T}}_{N-1}} - \frac{2}{{\sqrt m }}{\mW_{N + 1}} * \frac{2}{{\sqrt m }}\dot{\sigma}    \left( {{\mW_N}x} \right) * {\boldsymbol{T}_{N-1}}} \right\|_2}{\left\| \vx \right\|_2} 
  \\&\leq
  {\left\| {\frac{2}{{\sqrt m }}{\mW_{N + 1}}} \right\|_2}{\left\| {\frac{2}{{\sqrt m }}\dot{\sigma}    \left( { {{\widetilde \mW}_N}\vx} \right) * {\widetilde{\boldsymbol{T}}_{N-1}} - \frac{2}{{\sqrt m }}\dot{\sigma}    \left( {{\mW_N}\vx} \right) * {\boldsymbol{T}_{N-1}}} \right\|_2} 
  \\&\leq
  4{\left\| {\frac{2}{{\sqrt m }}\dot{\sigma}    \left( { {{\widetilde \mW}_N}\vx} \right) * {\widetilde{\boldsymbol{T}}_{N-1}} - \frac{2}{{\sqrt m }}\dot{\sigma}    \left( {{\mW_N}\vx} \right) * {\boldsymbol{T}_{N-1}}} \right\|_2} 
  \\&\leq \label{equ:lip_inetriangle}
  4{\left\| {\left( {\frac{2}{{\sqrt m }}\dot{\sigma} \left( {{\mW_N}\vx} \right) - \frac{2}{{\sqrt m }}\dot{\sigma}    \left( { {{\widetilde \mW}_N}\vx} \right)} \right) * {\boldsymbol{T}_{N-1}}} \right\|_2}
  +
  4{\left\| {\frac{2}{{\sqrt m }}\dot{\sigma}    \left( { {{\widetilde \mW}_N}\vx} \right) * \left( {{\boldsymbol{T}_{N-1}} - {\widetilde{\boldsymbol{T}}_{N-1}}} \right)} \right\|_2}
  \\& 
  \leq 4{\left\| {\left( {\frac{2}{{\sqrt m }}\dot{\sigma}    \left( {{{\widetilde \mW}_N}\vx} \right) - \frac{2}{{\sqrt m }}\dot{\sigma}    \left( {{\mW_N}\vx} \right)} \right)} \right\|_2}{\left\| {{\boldsymbol{T}_{N-1}}} \right\|_2}
  +
  4{\left\| {\frac{2}{{\sqrt m }}\dot{\sigma}    \left( { {{\widetilde \mW}_N}\vx} \right)} \right\|_2}{\left\| {{\boldsymbol{T}_{N-1}} - {\widetilde{\boldsymbol{T}}_{N-1}}} \right\|_2}
  \\&\leq  \label{equ:lip_relu}
  \frac{8}{{\sqrt m }} {\left\| {{{\widetilde \mW}_N}\vx - {\mW_N}\vx} \right\|_2}{\left\| {{\boldsymbol{T}_{N-1}}} \right\|_2}{+}
  8
  {\left\| {{\boldsymbol{T}_{N-1}} - {\widetilde{\boldsymbol{T}}_{N-1}}} \right\|_2}
  \\&\leq
  \frac{8}{{\sqrt m }} \left\| {{\widetilde \mW}_N}-{{\mW}_N} \right\|{\left\| {{\boldsymbol{T}_{N-1}}} \right\|_2}\;{+}8 {\left\| {{\boldsymbol{T}_{N-1}} - {\widetilde{\boldsymbol{T}}_{N-1}}} \right\|_2},
  \label{equ:lip_last}
\end{align}
where \cref{equ:lip_ine1} is due to 
$  {\left\| \va\vb^\top \right\|_\text{F}} \leq {\left\| \va \right\|_2}{\left\| \vb \right\|_2}$ for two arbitrary vectors $\va$ and $\vb$, \cref{equ:lip_inetriangle} comes from triangle inequality, \cref{equ:lip_relu} is based on Lipschitz continuous gradient
of the ReLU activation function.
 Using the result in \cref{equ:t_n_minus_1} and the following inequality:
\begin{align*}
 \left\| 
 {{\widetilde \mW}_N}-{{\mW}_N}
 \right\| 
 \leq
 \left\| {{\widetilde \mW}_N}-{{\mW}_N} \right\|_\text{F}
 \leq
 \left\| {\bm{\rvtheta} - \widetilde{\bm{\rvtheta}}} \right\|_2.
\end{align*}
The first term in \cref{equ:lip_last} can be bounded with probability by:
$$
  \frac{8}{{\sqrt m }}\left\| {{\widetilde \mW}_N}-{{\mW}_N} \right\|{\left\| {{\boldsymbol{T}_{N-1}}} \right\|_2} 
  \leq \frac{2^{2N+1}}{{\sqrt m }} 
  \left\| {\bm{\rvtheta} - \widetilde{\bm{\rvtheta}}} \right\|_2.
$$
For the second term in \cref{equ:lip_last}, we will bound 
${\left\| 
{\widetilde{\boldsymbol{T}}_{N-1}-{\boldsymbol{T}_{N-1}}} \right\|_2}
$ by induction. Base case satisfies with probability:
\begin{align*}
\left\|{\widetilde{\boldsymbol{T}}_{1}-
\boldsymbol{T}_{1}
}\right\|_2
&=
{\left\| {\frac{2}{{\sqrt m }}\sigma \left( {{{\widetilde \mW}_1}\vx} \right) - \frac{2}{{\sqrt m }}\sigma \left( {{\mW_1}\vx} \right)} \right\|_2}  \leq {\left\| {\frac{2}{{\sqrt m }}\left( {{{\widetilde \mW}_1}\vx - {\mW_1}\vx} \right)} \right\|_2} 
\\& \leq 
\frac{ 2}{{\sqrt m }}
\left\| 
{{\widetilde \mW}_1}-{{\mW}_1}
\right\|
 \leq 
\frac{ 2}{{\sqrt m }}
\left\| 
{\bm{\rvtheta} - \widetilde{\bm{\rvtheta}}}
\right\|_2 \,,
\end{align*}
Assume ${\left\| {{{\widetilde{\boldsymbol{T}_n}}} - {\boldsymbol{T}_n}} \right\|_2} \leq 
    \frac{C_2}{\sqrt m}
 \left\| {\bm{\rvtheta} - \widetilde{\bm{\rvtheta}}}
 \right\|_2$
, then, with probability:
\begin{align*}
  &{\left\| {{\widetilde{\boldsymbol{T}}_{n + 1}} - {\boldsymbol{T}_{n + 1}}} \right\|_2}
  \\ &\leq 
  {\left\| {\frac{2}{{\sqrt m }}\sigma \left( {{{\widetilde \mW}_{n + 1}}\vx} \right) * \left( {{{\widetilde{\boldsymbol{T}}}_n} - {\boldsymbol{T}_n}} \right)} \right\|_2} + {\left\| {{\boldsymbol{T}_n} * \frac{2}{{\sqrt m }}\left( {\sigma \left( {{{\widetilde \mW}_{n + 1}}\vx} \right) - \sigma \left( {{\mW_{n + 1}}\vx} \right)} \right)} \right\|_2} 
   \\ &\leq 
   \frac{2}{{\sqrt m }}{\left\| {\sigma \left( {{{\widetilde
  \mW}_{n + 1}}\vx} \right)} \right\|_2}\frac{C_2}{\sqrt m}\left\| {\widetilde{\bm\rvtheta}  - \bm\rvtheta } \right\|_2 + \frac{2}{{\sqrt m }}{\left\| {{\boldsymbol{T}_n}} \right\|_2}{\left\| {\left( {\sigma \left( {{{\widetilde \mW}_{n + 1}}\vx} \right) - \sigma \left( {{\mW_{n + 1}}\vx} \right)} \right)} \right\|_2} 
   \\ &\leq 
 \frac{ 2}{{\sqrt m }}\left\| {{{\widetilde \mW}_{n + 1}}} \right\|\frac{C_2}{\sqrt m}\left\| {\widetilde{\bm\rvtheta}  - \bm\rvtheta } \right\|_2 + \frac{2^{2n+1} }{{\sqrt m }}\left\| {\Delta \mW} \right\|
   \\ &\leq  
    \frac{ 2}{{\sqrt m }}\left\| {{{\widetilde \mW}_{n + 1}}} \right\|\frac{C_2}{\sqrt m}\left\| {\widetilde{\bm\rvtheta}  - \bm\rvtheta } \right\|_2 + \frac{2^{2n+1} }{{\sqrt m }}\left\| {\widetilde{\bm\rvtheta}  - \bm\rvtheta } \right\|_2
    \\ &\leq 
   \frac{ 2}{{\sqrt m }}2\sqrt m \frac{C_2}{\sqrt m}\left\| {\widetilde{\bm\rvtheta}  - \bm\rvtheta } \right\|_2 + \frac{2^{2n+1} }{{\sqrt m }}\left\| {\widetilde{\bm\rvtheta}  - \bm\rvtheta } \right\|_2 
    \\ &\leq 
    \left( {\frac{4C_2}{\sqrt m} + \frac{2^{2n+1} }{\sqrt m}} \right)\left\| {\widetilde{\bm\rvtheta}  - \bm\rvtheta } \right\|_2,
\end{align*}
where the first inequality uses the triangle inequality. Thus, we can bound with probability:
\begin{align*}
{\left\| 
{\widetilde{\boldsymbol{T}}_{N-1}-{\boldsymbol{T}_{N-1}}} \right\|_2}
\leq
 \frac{2^{3N-5}}{\sqrt m}   \left\| {\bm{\rvtheta} - \widetilde{\bm{\rvtheta}}} \right\|_2.
\end{align*}
Then \cref{equ:lip_last} becomes:
\begin{align*}
 &  {\left\| {{\partial _{{{\widetilde \mW}_N}}}f(\vx) - {\partial _{{\mW_N}}}f(\vx)} \right\|_\text{F}}
\leq
 \frac{8+2^{3N-5}}{\sqrt m} 
  \left\| {\bm{\rvtheta} - \widetilde{\bm{\rvtheta}}} \right\|_2.
\end{align*}
Now we consider all the weight matrices except $\W_{N+1}$ that is not trained. The following inequality holds with probability:
\begin{align*}
\| \emJ(\bm \rvtheta) - \emJ(\widetilde{\bm\rvtheta}) \|_\text{F} 
&  = 
\sqrt{
\sum \limits_{\vx \in \mathcal{X}} \Bigg ( 
\sum \limits_{n=1}^{N} \left \| \frac{\partial  f(\vx;\bm\rvtheta)}{\partial \mW_{n}} - \frac{\partial  f(\vx;\widetilde{\bm\rvtheta})}{\partial \widetilde{\mW}_{n}} \right \|_\text{F}^2 \Bigg )}
\\ & 
\leq \frac{8+2^{3N-5}}{\sqrt m} \sqrt{N|\mathcal{X}|}
\leq 
\frac{8+2^{3N-5}}{r+\sqrt{d}+2\log{(2N/\delta)}} \sqrt{N|\mathcal{X}|}
=\gamma_2,
\end{align*}
where $\gamma_2$ does not depend on the width $m$.
This completes the proof of \cref{thm:local_lip}.
\end{proof}
Finally, note that \cref{thm:stability_of_ntk} is an extension of \citet[Theorem G.1]{lee2019wide} from MLP to \pnns{}, the proof of \cref{thm:local_lip} is quite different for different networks while 
the idea of the remaining steps is based on the induction rule over time step $t$, which do not rely on the network.
Applying \cref{thm:local_lip} with 
$\gamma_1=\gamma_2=\frac{3Q R_0}{\lambda_{min}(\mK)}$ completes the proof, which is similar to the extension from MLPs to ResNets in \citet[Theorem 5]{tirer2020kernel}, RNN in \citet[Theorem 2]{alemohammad2021the}.
\section{Proofs of extrapolation}
\label{sec:appendix_proofextrapolation}

\subsection{Proof of \cref{thm:ndegree_pinets_infiniteNTK_extra}}
\label{proof:ndegree_pinets_infiniteNTK_extra}
\begin{proof}[Proof of \cref{thm:ndegree_pinets_infiniteNTK_extra}]
Recall that an infinite-width neural network trained through gradient descent is equivalent to kernel regression, 
the network output for any $\x \in \R^d$ is given by
\begin{align*}
 f (\x) &=\left(
 K(\x,\x_1), \cdots, 
 K(\x,\x_{|\gX|}) \right)
 \cdot \bm K^{-1} \bm{y}\,,
\end{align*}
where  $\bm{K}\in\mathbb{R}^{|\mathcal{X}| \times |\mathcal{X}|}$ is the NTK Gram matrix for training data, $K(\x,\x_i)$ is the kernel value between test data $\x$ and training data $\x_i$,  and $\bm{y} \in \mathbb{R^{|\mathcal{X}|}}$ are the training labels.
Since the NTK $K(\x,\xp)$ is 1-homogeneous w.r.t $\x$, the network output  $f (\x)$ is also $N$-homogeneous w.r.t $\x$. Therefore, given inputs $\x_0 = t \vv$ and $\x =  \x_0 + h\vv $, the outputs of the network are:
\begin{align*}
f(\x_0) =& f(t \vv) = t^N \cdot f(\vv)
\\
f(\x) =& f\left(\left((t+h)^N\right) \vv\right) = (t+h)^N\cdot f(\vv) \,,
\end{align*}
Thus:
\begin{align*}
f(\x) - f(\x_0)  
=
f(t\vv) - f((t+h)\vv)
= 
((t+h)^N-t^N) \cdot f(\vv) \,,
\end{align*}
Thus, the network extrapolates to at most $N$-degree function with respect to $h$.
\end{proof}
\subsection{Proof of \cref{thm:ndegree_mlps_infiniteNTK}}
\label{proof:ndegree_mlps_infiniteNTK}
\begin{proof}[Proof of \cref{thm:ndegree_mlps_infiniteNTK}]
The NTK of MLPs defined in \cref{eq:ntk_mlps}, denoted by  $K^{(N)}(\x,\xp)$, can be rewritten in a more compact form~\citep{nguyen2021tight}:
\begin{equation}
K^{(N)}(\x,\xp)=G^{(N)}(\x,\xp)
+
\sum_{n=1}^{N-1} G^{(n)}(\x,\xp) * \dot{G}^{(n+1)}(\x,\xp) * \ldots * \dot{G}^{(N)}(\x,\xp),
\label{eq:ndegree_mlps_infiniteNTK_rewrite}
\end{equation}
where $\text {for } n \in[3, N]$:
\begin{equation}
\begin{aligned}
&
K^{(1)}(\x,\xp)=G^{(1)}(\x,\xp)=\x^{\top} \xp \\
&G^{(2)}(\x,\xp)=2 \mathbb{E}_{
\w \sim \mathcal N(0, I)
}
\left[
{\sigma}(\w^\top \x) 
{\sigma}(\w^\top  \x')
\right] \\
&G^{(n)}(\x,\xp)=2 \mathbb{E}_{
\w \sim \mathcal N(0, 1)
}\left[\sigma\left(\sqrt{G^{(n-1)}(\x,\xp)} \w\right) \sigma\left(\sqrt{G^{(n-1)}(\x,\xp)} \w\right)\right]\,,
\end{aligned}
\end{equation}

$\text {for } n \in[2, N]$:
\begin{equation}
\begin{aligned}
&K^{(n)}(\x,\xp)=K^{(n-1)}(\x,\xp) * \dot{G}^{(n)}(\x,\xp)+G^{(n)}(\x,\xp)\\
&\dot{G}^{(n)}(\x,\xp)=2 \mathbb{E}_{
w \sim \mathcal N(0, 1)
}\left[\sigma^{\prime}\left(\sqrt{G^{(n-1)}(\x,\xp)} w\right) \sigma^{\prime}\left(\sqrt{G^{(n-1)}(\x,\xp)} w\right)\right]\,.
\end{aligned}
\end{equation}
Since the ReLU activation function $\sigma$ is 1-homogeneous and its derivative $\sigma^{\prime}$ is 0-homogeneous,  $G^{(n)}(\x,\xp)$ is 1-homogeneous and $\dot{G}^{(n)}(\x,\xp)$ 0-homogeneous w.r.t $\x$, $\text {for } n \in[1, N]$. Thus, the NTK $K^{(N)}(\x,\xp)$ is 1-homogeneous w.r.t $\x$.
Since the NTK $K^{(N)}(\x,\xp)$ is 1-homogeneous w.r.t $\x$, the network output  $f (\x)$ is also 1-homogeneous w.r.t $\x$. Therefore, given inputs $\x_0 = t \vv$ and $\x =  \x_0 + h\vv $, the outputs of the network are:
\begin{align*}
f(\x_0) =& f(t \vv) = t \cdot f(\vv)
\\
f(\x) =& f\left(\left(t+h\right) \vv\right) = (t+h)\cdot f(\vv)
\end{align*}
Thus:
\begin{align}
f(\x) - f(\x_0)  
& = h \cdot f(\vv)
 \label{eq:s2_infiniteNTKndegree}
\end{align}
Since the term $f(\vv)$ in \cref{eq:s2_infiniteNTKndegree} is finite in our assumption, the network extrapolates to a linear function with respect to $h$. This completes the proof.
\end{proof}

\subsection{Proof of \cref{thm:exact}}
\label{proof:exact}
\begin{lemma} \label{thm:feature}
A specific feature map $\phi(\x)$ induced by the NTK of a two-degree \pnns{} with ReLU activation function is
\begin{equation}
\begin{split}
\phi \left(\x\right) & = \left( c^{\prime} \x \cdot  \dot{\sigma}\langle\w^{(k)}, \x \rangle 
\cdot {\sigma}(\langle\w^{(k)}, \x \rangle) 
, 
c^{\prime\prime} {\sigma}(\langle\w^{(k)}, \x \rangle ) \cdot {\sigma}(\langle\w^{(k)}, \x\rangle )
 \right)
\\ & = 
\left( c^{\prime} \x 
\cdot
\langle \w^{(k)},x \rangle
\cdot \dot{\sigma}(\langle\w^{(k)}, \x \rangle)
, 
c^{\prime\prime} 
(\langle\w^{(k)},x\rangle)^2 \cdot
\dot{\sigma}(\langle\w^{(k)}, \x \rangle ) 
 \right),
\label{equ:featuremap_pinet}
\end{split}
\end{equation}
where $\w^{(k)}$ is sampled from $\N(\bm{0}, \bm{I})$
, $c^{\prime}$ and $c^{\prime\prime}$ are constants, the last equality is due to the property of ReLU function: $ {\sigma}(a) =a \dot{\sigma}(a)$ for any $a \in \R$. 
\end{lemma}
\begin{proof}[Proof of \cref{thm:feature}]
The NTK for the second degree \pnns{} with ReLU activation is given by
\begin{equation}
\label{equ:ntkpinet}
\begin{split}
\bm K(\x, \x^{\prime})   =  & \;  \frac{8}{m}
\cdot 
\mathop{\E}_{ \w  \sim \mathcal{N}(\bm{0}, \bm{I} )}  \left( \x^{\top} \xp  \cdot 
\dot{\sigma} \left(\w^{\top} \x \right) 
\cdot 
\dot{\sigma} \left(\w^{\top} \xp \right) \right) 
\cdot 
 \mathop{\E}_{ \w  \sim \mathcal{N}(\bm{0}, \bm{I} )} 
 \left( 
 \sigma(\w^{\top} \x )
 \cdot  
 \sigma(w^{\top} \xp)
 \right)
\\
 + & \; \frac{4}{m} \cdot 
 \mathop{\E}_{ \w  \sim \mathcal{N}(\bm{0}, \bm{I} )} 
 \left( 
 \sigma(\w^{\top} \x )
 \cdot  
 \sigma(\w^{\top} \xp)
 \right)
 \cdot 
  \mathop{\E}_{ \w  \sim \mathcal{N}(\bm{0}, \bm{I} )}
 \left( 
 \sigma(\w^{\top} \x )
 \cdot  
 \sigma(\w^{\top} \xp)
 \right)\,.
\end{split}
\end{equation}
Then, we utilize the kernel formula to construct the feature map that need to satisfy the following condition: 
\begin{equation}
\label{equ:temp}
\begin{split}
K(\x, \x^{\prime}) = \bigl< \phi (\x), \phi (\x^{\prime}) \bigr>\,.
\end{split}
\end{equation}
The following feature map 
would satisfy \cref{equ:temp}
because the inner product of $\phi(\x)$ and $\phi(\xp)$ for any $\x$, $\xp$ is equivalent to the expected value in \cref{equ:ntkpinet}, after integrating with respect to the density function of $\w$.
\begin{equation}
\begin{split}
\phi \left(\x\right) & = \left( c^{\prime} \x \cdot  \dot{\sigma}\langle\w^{(k)}, \x \rangle 
\cdot {\sigma}(\langle\w^{(k)}, \x \rangle) 
, 
c^{\prime\prime} {\sigma}(\langle\w^{(k)}, \x \rangle ) \cdot {\sigma}(\langle\w^{(k)}, \x\rangle )
 \right)
\\ & = 
\left( c^{\prime} \x 
\cdot
\langle \w^{(k)},x \rangle
\cdot \dot{\sigma}(\langle\w^{(k)}, \x \rangle) 
, 
c^{\prime\prime} 
(\langle\w^{(k)},x\rangle)^2 \cdot
\dot{\sigma}(\langle\w^{(k)}, \x \rangle ) 
 \right),
\end{split}
\end{equation}
where $\w^{(k)}$ is sampled from $\N(\bm{0}, \bm{I})$
, $c^{\prime}$ and $c^{\prime\prime}$ are constants, the last equality is due to the following property of ReLU function: $ {\sigma}(a) =a \dot{\sigma}(a)$ for any $a \in \R$. 
\end{proof}

Sequentially, we are ready to prove \cref{thm:exact}.

\begin{proof}[Proof of~\cref{thm:exact}]
According to \citet[Lemma 2]{xu2021how}, the kernel regression solution is equivalent to the following form:
\begin{align}
\label{equ:kernelregression_solution}
f (\x) = \bt^{\top}\phi(\x),
\end{align}
where the representation coefficient $\bt$ holds: 
\begin{equation}
\label{eq:minnorm}
\begin{split}
& \min_{\bt^{\prime}} \| \bt^{\prime}\|_2\\
\text{s.t.} \;\;\;  &\phi(\x_i)^{\top} \bt^{\prime} = y_i, \quad i = 1, \dots, {|\mathcal{X}|}\,.
\end{split}
\end{equation}
The feature map $\phi(\x)$ for a two-degree \pnns{} with ReLU activation is given in \cref{thm:feature}
\begin{align*} 
\phi \left(\x\right) = 
\left( c^{\prime} \x 
\cdot
\langle \w^{(k)},x \rangle
\cdot \dot{\sigma}(\langle\w^{(k)}, \x \rangle) 
, 
c^{\prime\prime} 
(\langle\w^{(k)},x\rangle)^2 \cdot
\dot{\sigma}(\langle\w^{(k)}, \x \rangle ) 
 \right) 
 \,,
\end{align*}
where $\w^{(k)}$ is sampled from $\N(\bm{0}, \bm{I})$, and $c^{\prime}$,$c^{\prime\prime}$ are constant. Below, for avoiding complicating the notation, we will discard the index and use $\w$ to represent a specific $\w^{(k)}$,
We assume the constants $c^{\prime}$ and $c^{\prime\prime}$ are $1$.
Note that $\bt$ consists of weights for each $\x \x^{\top} \bm w \cdot \I \left( \w^{{\top}} \x \geq 0 \right) \in \R^d$ 
and
 $ \x^{\top} \w \w^{\top} \x \cdot \I \left( \w^{{\top}} \x \geq 0 \right) \in \R$. 
 For any $\w \in \R^d$, the weight vectors corresponding to $\x \x^{\top} w \cdot \I \left( \w^{{\top}} \x \geq 0 \right)$ are  symbolized by $\hat{\bt}_{\w} = (\hat{\bt}^{(1)}_{\w}, ..., \hat{\bt}^{(k)}_{\w} ) \in \R^d$
 and the weight vectors for $ \x^{\top} \w \w^{\top} \x \cdot \I \left( \w^{{\top}} \x \geq 0 \right)$ are symbolized by $\hat{\bt}^{\prime}_{\w} \in \R$.
 Given the fact that if $\w^{\top} \x_i \geq 0$ for any $\w \in \R^d$, then $c\w^{\top} \x_i \geq 0$ for any $c > 0$, 
 we use the notation $\bt_{\w}$ and $\bt_{\w}^{\prime}$ to represent the combined effect of all weights $(\hat{\bt}^{(1)}_{c\w}, ..., \hat{\bt}^{(k)}_{c\w}) \in \R^d$ and $\hat{\bt}^{\prime}_{c\w} \in \R$ for all $c\w$ with $c> 0$. This allows us to change the distribution of $\w$ from $\mathcal{N}(\bm 0, \bm I_d)$
 to $\mathrm{Unif}(\mathbb{S}^{d})$. Specifically, for each $\w \sim \mathrm{Unif}(\mathbb{S}^{d})$, $ \bt_{\w}^{(j)} $ is denoted as the total effect of the weights in the same direction of $\w$. 
\begin{align*}
 \bt_{\w}^{(j)} = \int \hat{\bt}_{\bm{u}}^{(j)} \I \left( \frac{\w^{\top} \bm{u} }{\| \w\| \cdot \|\bm{u}\|} = 1   \right) \dd\pr(\bm{u}), \;\;\; j \in [d]
\end{align*}
where $\bm{u}  \sim \mathcal{N}(\bm{0}, \bm{I})$. 
Similarly, $\bt_{\w}^{\prime}$ is defined as follows:
\begin{align}
\label{eq:redef2}
\bt_{\w}^{\prime} &= \int \hat{\bt}_{\bm{u}} \I \left( \frac{\w^{\top} \bm{u} }{\| \w\| \cdot \|\bm{u}\|} = 1   \right) \cdot {\|\bm{u}\|} \dd\pr(\bm{u})
\end{align}
Then, the min-norm solution in \cref{eq:minnorm} is equivalent to:
\begin{align}
\label{eq:obj}
& \min_{\bt} \int   \left(\bt_{\w}^{(1)}\right)^2 + \left(\bt_{\w}^{(2)}\right)^2 + ... + \left(\bt_{\w}^{(k)}\right)^2 + \left( \bt_{\w}^{\prime} \right)^2 \dd\pr(\w) \\
\label{eq:tobesim}
\text{s.t.} 
\;\;\; &    
\int_{\w^{\top} \x_i \geq 0} 
\x_i^{\top} \bt_{\w} \w^{\top} \x_i  + 
\x_i^{\top} \bt_{\w}^{\prime} \w
\w^{\top} \x_i  
\;\;  
\dd\pr(\w) = 
\x_i^{\top} \bm{\mBeta}_{g} \x_i 
\;\;\; \forall i \in [{|\mathcal{X}|}],
\end{align}
where $\w \in \mathrm{Unif}(\mathbb{S}^{d})$. Thus, $\pr(\w)$ is a constant, which indicates that only half of the $\w$ on the unit sphere activate each specific $\x_i$. 
Therefore, we can further simplify the constraint in \cref{eq:tobesim} as 
\begin{align} \label{eq:rearranged}
\int_{\w^{\top} \x_i \geq 0}  
\x_i^{\top}
\left(
 \bt_{\w} \w ^{\top}  + 
 \bt_{\w}^{\prime} \w
\w^{\top}  -
2\bm{\mBeta}_{g} 
\right)
\x_i
\;\;  \dd\pr(\w) = 0 \;\;\; \forall i \in [{|\mathcal{X}|}],
\end{align}
where \cref{eq:rearranged} follows from the following steps
\begin{align*}
& \int_{\w^{\top} \x_i \geq 0} 
\x_i^{\top} \bt_{\w} \w^{\top} \x_i  + 
\x_i^{\top} \bt_{\w}^{\prime} \w
\w^{\top} \x_i   \dd\pr(\w) = \x_i^{\top} \bm{\mBeta}_{g} \x_i  \;\; \forall i \in [{|\mathcal{X}|}], \\
\Longleftrightarrow & 
\int_{\w^{\top} \x_i \geq 0} 
\x_i^{\top} \bt_{\w} \w^{\top} \x_i  + \x_i^{\top} \bt_{\w}^{\prime} \w
\w^{\top} \x_i   
\dd\pr(\w)
\\
& = \frac{1}{\int_{\w^{\top} \x_i \geq 0} \dd\pr(\w) } \cdot   \int_{\w^{\top} \x_i \geq 0}  \dd\pr(\w)   \cdot 
\x_i^{\top} \bm{\mBeta}_{g} \x_i
\;\;\; 
\forall i \in [{|\mathcal{X}|}],\\
\Longleftrightarrow & \int_{\w^{\top} \x_i \geq 0} 
\x_i^{\top} \bt_{\w} \w^{\top} \x_i  + \x_i^{\top} \bt_{\w}^{\prime} \w
\w^{\top} \x_i  
\dd\pr(\w) \\
& = 2\cdot \int_{\w^{\top} \x_i \geq 0}    
\x_i^{\top} \bm{\mBeta}_{g} \x_i \dd\pr(\w)  \;\;\; \forall i \in [{|\mathcal{X}|}],\\
\Longleftrightarrow  &  \int_{\w^{\top} \x_i \geq 0}
\x_i^{\top}
\left(
 \bt_{\w} \w ^{\top}  + 
 \bt_{\w}^{\prime} \w
\w^{\top}  -
2\bm{\mBeta}_{g} 
\right)
\x_i
\;\;  \dd\pr(\w) = 0  \;\;\; \forall i \in [{|\mathcal{X}|}].
\end{align*}
\begin{lemma}\label{thm:claim} The global optimum of \cref{eq:obj} subject to \cref{eq:rearranged}, i.e., 
\begin{align}\label{eq:objective}
& \min_{\bt} \int   \left(\bt_{\w}^{(1)}\right)^2 + \left(\bt_{\w}^{(2)}\right)^2 + ... + \left(\bt_{\w}^{(k)}\right)^2 + \left( \bt_{\w}^{\prime} \right)^2 \dd\pr(\w)\\
\label{eq:constraint}\text{s.t.}
\;\;\; &    
\int_{\w^{\top} \x_i \geq 0}  
\x_i^{\top}
\left(
 \bt_{\w} \w ^{\top}  + 
 \bt_{\w}^{\prime} \w
\w^{\top}  -
2\bm{\mBeta}_{g} 
\right)
\x_i
\;\;  \dd\pr(\w) = 0 \;\;\; \forall i \in [{|\mathcal{X}|}],
\end{align}
satisfies $  \bt_{\w} \w ^{\top}  + 
 \bt_{\w}^{\prime} \w
\w^{\top}  =
2\bm{\mBeta}_{g}   $ for all $\w$.
\end{lemma}
\begin{proof}[Proof of \cref{thm:claim}]
Through \cref{thm:claim}, we can achieve the goal of our proof towards \cref{thm:exact}, i.e., $f(\x) = f_{\rho}(\bm x)$. The reason is that if \cref{thm:claim} holds, for any $\x \in \R^d$:
\begin{align*}
f(\x) & = \int_{\w^{\top} \vx \geq 0}  \x^{\top} 
\left( \bt_{\w} \w ^{\top}  + 
 \bt_{\w}^{\prime} \w
\w^{\top} \right)
\vx
\;\;   \dd\pr(\w)\\
& = \int_{\w^{\top} \x \geq 0} 2 
\x^{\top} \bt_g \x \;\; \dd\pr(\w)\\
& = \int_{\w^{\top} \x \geq 0} \dd\pr(\w)   2  \x^{\top} \bt_g \x \\
& =  \frac{1}{2} 2   \x^{\top} \bt_g \x = f_{\rho}(\bm x)\,.
\end{align*}
Therefore, the remaining step is to prove \cref{thm:claim}.
Since \cref{eq:objective} is a convex optimization problem with affine constraint \cref{eq:constraint}, we can introduce the Lagrange multipliers and use the Karush–Kuhn–Tucker (KKT) condition. The Lagrange multiplier has the following form:
\begin{align}
\mathcal{L}(\bt, \lambda) = &\int   \left(\bt_{\w}^{(1)}\right)^2 + \left(\bt_{\w}^{(2)}\right)^2 + ... + \left(\bt_{\w}^{(k)}\right)^2 + \left( \bt_{\w}^{\prime} \right)^2 \dd\pr(\w) \\
&+ \sum_{i=1}^{|\mathcal{X}|} \lambda_i \cdot  \left(  \int_{\w^{\top} \x_i \geq 0}  
\x_i^{\top}
\left(
 \bt_{\w} \w ^{\top}  + 
 \bt_{\w}^{\prime} \w
\w^{\top}  -
2\bm{\mBeta}_{g} 
\right)
\x_i
\;\;  \dd\pr(\w) = 0 
\right)\,.
\end{align}
By setting the partial derivative to zero, we obtain:
\begin{align}\label{eq:lag}
&
\frac{\partial \mathcal{L}}{\partial \bt_{\w}^{(k)}} = 2 \bt_{\w}^{(k)} \pr(\w) + \sum_{i=1}^{|\mathcal{X}|} \lambda_i \cdot 
{
(
\x_i \x_i^{\top}\w
)}^d 
\cdot \I \left( \w^{\top} \x_i \geq 0 \right) = 0
\\
\label{eq:lag1} 
& 
\frac{\partial \mathcal{L}  }{\bt_{\w}^{\prime} } = 2 \bt_{\w}^{\prime} \pr(\w) + \sum_{i=1}^{|\mathcal{X}|} \lambda_i \cdot 
\x_i  \w \w^{\top} \x_i 
\cdot \I \left( \w^{\top} \x_i \geq 0 \right)   = 0
\\
\label{eq:lag2}
&
\frac{\partial \mathcal{L}}{\partial  \lambda_i} =   \int_{\w^{\top} \x_i \geq 0}  
\x_i^{\top}
\left(
 \bt_{\w} \w ^{\top}  + 
 \bt_{\w}^{\prime} \w
\w^{\top}  -
2\bm{\mBeta}_{g} 
\right)
\x_i
\;\;  \dd\pr(\w)  = 0\,.
\end{align}
It is obvious that the solution in~\cref{thm:claim} satisfies \cref{eq:lag2}. Thus, the remaining step is to show that there exist a set of $\lambda_i$ where $i \in [{|\mathcal{X}|}]$ that satisfies \cref{eq:lag} and \cref{eq:lag1}. 
We simplify \cref{eq:lag} and  \cref{eq:lag1}  as follows:
\begin{align} 
\label{eq:equation1}
 \bt_{\w}^{(k)}  = c \cdot \sum\limits_{i=1}^{|\mathcal{X}|} \lambda_i \cdot 
 {(
\x_i \x_i^{\top}\w
)}^d 
 \cdot \I \left( \w^{\top} \x_i \geq 0 \right),
\\
\label{eq:equation2}
 \bt_{\w}^{\prime}  = c \cdot \sum\limits_{i=1}^{|\mathcal{X}|} \lambda_i \cdot 
\x_i  \w \w^{\top} \x_i 
 \cdot \I \left( \w^{\top} \x_i \geq 0 \right)\,,
\end{align}
where $c $ is a constant. 
Combining \cref{eq:equation1} and \cref{eq:equation2}, we can simplify the constraint \cref{eq:equation2} as follows:
\begin{align}
\label{eq:equation3}
 \bt_{\w}^{\prime}= \bt_{\w}  \w^{\top}\,. 
\end{align}
The remaining step is to show that based on the condition on training data, there exists a set of $\lambda_i$ that satisfy \cref{eq:equation1} and \cref{eq:equation3}. 
For each $\w$, there must exist a set of $\lambda_i$ so that the following equations satisfy:
\begin{align}
& 
\bt_{\w}^{(k)}  = c \cdot \sum\limits_{i=1}^{|\mathcal{X}|} \lambda_i \cdot 
{(
\x_i \x_i^{\top}\w
)}^d 
\cdot \I \left( \w^{\top} \x_i \geq 0 \right)
\\
\label{eq:claim1_1}
 &  \bt_{\w}^{\prime}   =  \bt_{\w}^{\top}\w \\
 \label{eq:claim1_2}
 & \bt_{\w} \w ^{\top}  + 
 \bt_{\w}^{\prime} \w
\w^{\top}  =
2\bm{\mBeta}_{g}\,,
\end{align}
where $ \bt_{g}  $ and $\w$ are fixed. From ~\cref{eq:claim1_1} and~\cref{eq:claim1_2}, we can see that $\bt_{\w}$ is determined by $\bt_g$ and $\w$, and there exists a solution for this consistent linear system. Next, we are left with the following linear system that contains $d$ linear equations \[ \bt_{\w}^{(k)}  = c \cdot \sum\limits_{i=1}^{|\mathcal{X}|} \lambda_i \cdot 
{(
\x_i \x_i^{\top}\w
)}^d 
\cdot \I \left( \w^{\top} \x_i \geq 0 \right), \;\;\; \forall k \in [d]. \]
\\
Recall the assumption for the training data, there exist at least $d$ linearly independent $\x_i$ that activates a specific $\w$. This implies that for any $\w$ there exists at least $d$ free variables. Thus, the solutions for this linear system exist.
\end{proof}
Thus, the proof of \cref{thm:exact} is finished.
\end{proof}

\section{Proof of spectral bias}
\label{sec:proofofeigenvalue}
Let us recall the core notation. Denote by $\{Y_{k, j}\}_{j=1}^{F(d, k)}$ the $k$-degree spherical harmonics in $d+1$ variables. $G_{k}^{(\gamma)}$ represents the Gegenbauer polynomials with respect to the weight function $x \mapsto (1-x^2)^{\gamma - \frac{1}{2}}$ and degree $k$. Finally, denote by $F(d, k) := \frac{2k + d - 1}{k} {k+d-2 \choose d-1}$. 

Given the fact that $\bm{\kappa_1}$ and $\bm{\kappa_2}$ are also dot-product Mercer kernels, their corresponding decompositions in terms of Gegenbauer polynomials can be provided based on \cref{eqa:sph_eigen_poly}:
\begin{equation}
\begin{split}
    \langle \bm{x}, \bm{x}' \rangle \bm{\kappa_1}(\bm{x}, \bm{x'}) & = \sum_{k = 0}^{\infty} \mu_{1, k} F(d, k) G_{k}^{(\frac{d-1}{2})}(\langle \bm{x}, \bm{x}' \rangle) \\
    \bm{\kappa_2}(\bm{x}, \bm{x'}) & = \sum_{k = 0}^{\infty} \mu_{2, k} F(d, k) G_{k}^{(\frac{d-1}{2})}(\langle \bm{x}, \bm{x}' \rangle).
\end{split}
\end{equation}
Note that the decay in 
$\mu_{1,k} = \mu_{2,k}$ is $\Omega (k^{-d-1})$~\citep{bach2017breaking,cao2019towards, NEURIPS2019_c4ef9c39}. 
Sequentially, we are ready to prove \cref{theorem:pi_decay}.
\begin{proof}[Proof of \cref{theorem:pi_decay}]
For $N$-degree \pnns{}, in order to study the decay rate of the eigenvalues, we express the NTK obtained in \cref{equ:ntkpinet_ndegree} as the product of multiple kernels:
\begin{equation}
\begin{split}
    K
    (\bm{x}, \bm{x'})  =
     2\left(
    \sqrt{\frac{2}{m}}
    \sum_{k = 0}^{\infty} (N\mu_{1, k} + \mu_{2, k}) F(d, k) G_{k}^{(\frac{d-1}{2})}(\langle \bm{x}, \bm{x}' \rangle)\right) 
    \\
    \cdot \left(
    \sqrt{\frac{2}{m}}
    \left(\sum_{k = 0}^{\infty} \mu_{2, k} F(d, k) G_{k}^{(\frac{d-1}{2})}(\langle \bm{x}, \bm{x}' \rangle)\right)
    \right)
    ^{N-1}
\end{split}
\label{pi_ei_prod}
\end{equation}
Comparing the above equation with~\cref{eqa:sph_eigen_poly}, it turns out that we need to simplify~\cref{pi_ei_prod} and
equate the polynomial coefficients on both equations.
It is obvious that we get the form of the product of multiple polynomials in~\cref{pi_ei_prod}. Fortunately, the following Lemma allows us to express the product of two Gegenbauer polynomials as a linear
combination of other Gegenbauer polynomials.
\begin{lemma}
\cite[Eq (8)]{prod_gegen}
\label{lemma:product_of_Gegenbauer}
For $b \in \mathbb{R}$ and any $\Gm, \Gn \in \mathbb{N}$, there exists a set of positive coefficients 
$\lbrace \lambda^{(\Gm,\Gn)}_s  \rbrace_{s=0}^{\min(\Gm,\Gn)}$
such that
\begin{equation}
    \begin{split}
       G^{(b)}_\Gm(x) G^{(b)}_\Gn(x) = \sum^{\min(\Gm,\Gn)}_{s = 0}\lambda^{(\Gm,\Gn)}_s G^{(b)}_{\Gm+\Gn-2s}(x)\,, 
    \end{split}
    \label{gegen_split}
\end{equation}
where 
$$
\lambda^{(\Gm,\Gn)}_s=
\frac{\Gm+\Gn+v-2 s}{\Gm+\Gn+v-s} \cdot \frac{(v)_{s}(v)_{\Gm-s}(v)_{\Gn-s}}{s !(\Gm-s) !(\Gn-s) !} \cdot \frac{(2 v)_{\Gm+\Gn-s}}{(v)_{\Gm+\Gn-s}} \cdot \frac{(\Gm+\Gn-2 s) !}{(2 v)_{\Gm+\Gn-2 s}} \,,
$$
and 
 $$(v)_k := v(v+1)(v+2)....(v+k-1), \quad (v)_0 := 1 \,.$$
\label{lem:gegen-prod}
\end{lemma}
For convenience, we assume $v$ is an integer and $k$ even, then we set $\Gm=\Gn=k$, $s=0$, and apply \cref{lemma:product_of_Gegenbauer} recursively by $N$ times, we can obtain the lower bound regarding the coefficient of the term $C_{Nk}^{(\frac{d-1}{2})}$, which is the $(Nk)^{\text{th}}$ harmonic:
\begin{equation}
\mu_{Nk} F(d, Nk) 
\geq 
\left(
\sqrt{\frac{2}{m}}F(d, k)
\right)
^N
(N\mu_{1, k} + \mu_{2, k})\mu_{2, k}^{N-1}
\prod \limits_{\alpha=1}^{N}
\lambda^{(k,\alpha k)}_0
.
\label{eq:lower-bound}
\end{equation}
It suffices to obtain the form of $\lambda^{(k,\alpha k)}_0$.
The coefficient $\lambda^{(k,\alpha k)}_0$ defined in \cref{lem:gegen-prod}

\begin{equation}
\begin{split}
   \lambda^{(k,\alpha k)}_0  & = \frac{(\alpha k+k)+v} {(\alpha k+k)+v}. \frac{(v)_0  (v)_{k} (v)_{\alpha k}} {0! (k)! (\alpha k)!}. \frac{(2v)_{(\alpha k+k)}} {(v)_{(\alpha k+k)}}. \frac{(\alpha k+k)!} {(2v)_{(\alpha k+k)}} 
   \\& = 
   \frac{(v)_{k} (v)_{\alpha k}} {(k)! (\alpha k)!}. \frac{(\alpha k+k)!} {(v)_{(\alpha k+k)}}  
    \\& = 
    \frac{(v+k-1)!(v+\alpha k-1)!} {((v-1)!)^2(k)!(\alpha k)!} . \frac{(\alpha k+k)!(v-1)!} {(v+\alpha k + k-1)!} \\& = 
    \frac{(v+k-1)!(v+\alpha k-1)!} {(v-1)!(k)!(\alpha k)!}  
    .
    \frac{(\alpha k+k)!} {(v+\alpha k + k-1)!}
    \\&\sim
    \frac{(v+k-1)^{(v+k-0.5)}(v+\alpha k-1)^{(v+\alpha k-0.5)}}
    {(v-1)^{(v-0.5)}k^{(k+0.5)}(\alpha k)^{(\alpha k+0.5)}}
    .
    \frac{(\alpha k + k)^{(\alpha k + k+0.5)}}
    {(v+\alpha k + k -1)^{(v+\alpha k + k -0.5)}}
    \,,
\end{split}
\end{equation}
where we apply the Stirling’s approximation ( $n! \sim \sqrt{2\pi n}(\frac{n}{e})^n $) at the final step. Next, we consider the case when $k \gg v$. In order to match the term $C_{Nk}^{(\frac{d-1}{2})}$ in \cref{pi_ei_prod}, we set $v = (d-1)/2$ and obtain:

\begin{equation}
\begin{split}
    \lambda^{(k,\alpha k)}_0
    &\sim
    \frac{(k)^{(v+k-0.5)}(\alpha k)^{(v+\alpha k-0.5)}}
    {k^{(k+0.5)}(\alpha k)^{(\alpha k+0.5)}}
    .
    \frac{(\alpha k + k)^{(\alpha k + k+0.5)}}
    {(\alpha k + k)^{(v+\alpha k + k -0.5)}}
    \\&\sim
    (k)^{(v-1)}(\alpha k)^{(v-1)} 
    (\alpha k + k)^{(-v-1)} 
    \sim 
    {\left(\frac{\alpha k }{1+\alpha}\right)}^{{v-1}}
    =
    {\left(\frac{\alpha k }{1+\alpha}\right)}^{\frac{d-3}{2}}
    \,.
\end{split}
\label{equ:lambda_alphak}
\end{equation}
Plugging \cref{equ:lambda_alphak} into \cref{eq:lower-bound}, we obtain:
\begin{equation}
\begin{split}
    \mu_{Nk} & \geq 
    \frac{\left(
\sqrt{\frac{2}{m}}F(d, k)
\right)^N}{F(d, Nk)}
     (N\mu_{1, k} + \mu_{2, k})\mu_{2, k}^{N-1}
    \prod \limits_{\alpha=1}^{N}
    \lambda^{(k,\alpha k)}_0
    \\& ~\sim \frac{F(d, k)^N}{F(d, Nk)}(N\mu_{1, k} + \mu_{2, k})
    \mu_{2, k}^{N-1}
    \left(\frac{k}{N}\right)^\frac{d-3}{2}
    \\ & ~\sim \frac{k^{Nd}}{(Nk)^{d}}(N\mu_{1, k} + \mu_{2, k})
    \mu_{2, k}^{N-1}
    \left(\frac{k}{N}\right)^\frac{d-3}{2}
    \quad
    \text{(by Stirling)}
    \\ & ~\sim \frac{k^{Nd}}{(Nk)^{d}}
    \Omega (k^{-Nd-N})
    \left(\frac{k}{N}\right)^\frac{d-3}{2}
    \\ & ~\sim \Omega ({(kN^{3})}^{-\frac{d}{2}
    })
\label{kgtdset}
\end{split}
\end{equation}
Setting $k=k^\prime / N$ allows us to conclude the proof.
\end{proof}
\label{lem:decay-rate}
\begin{figure}[!hb]
\centering
\begin{subfigure}{0.9\linewidth}
    \centering \includegraphics[width=\textwidth]{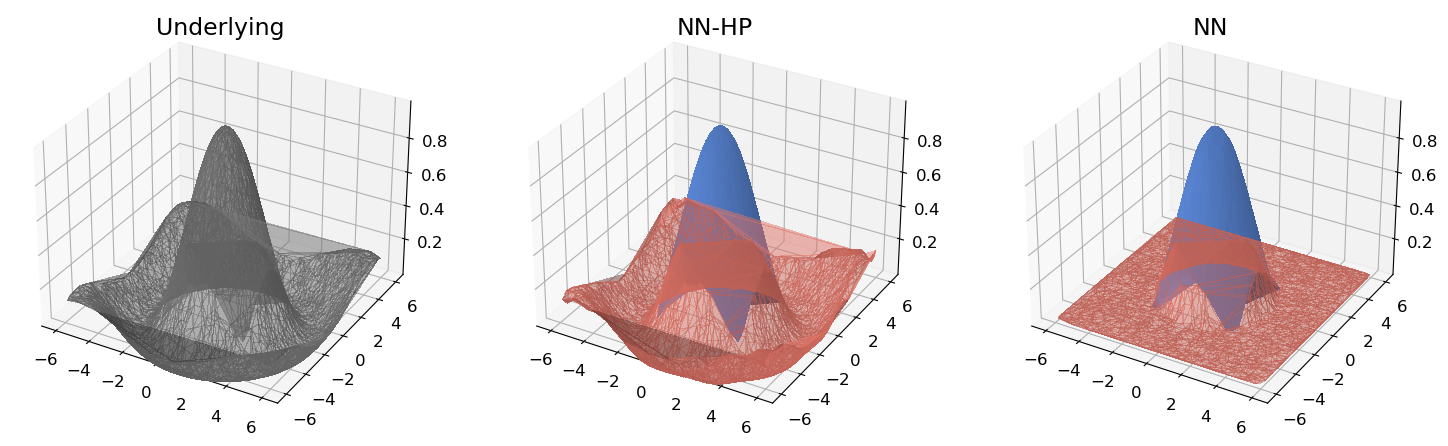}
\caption{Fitting results for underlying function 
$
f_{\rho}(\bm x)= 
\frac
{\sin \left({\sqrt{\|\bm x\|_2}}\right)}
{{\sqrt{\|\bm x\|_2}}}
$, where $\|\cdot\|_2$ indicates the Euclidean norm. The prediction within the training (extrapolation) region is presented by \textcolor{blue}{blue} (\textcolor{red}{red}) color.}
\end{subfigure}
\begin{subfigure}{0.9\linewidth}
    \centering
 \includegraphics[width=\textwidth]{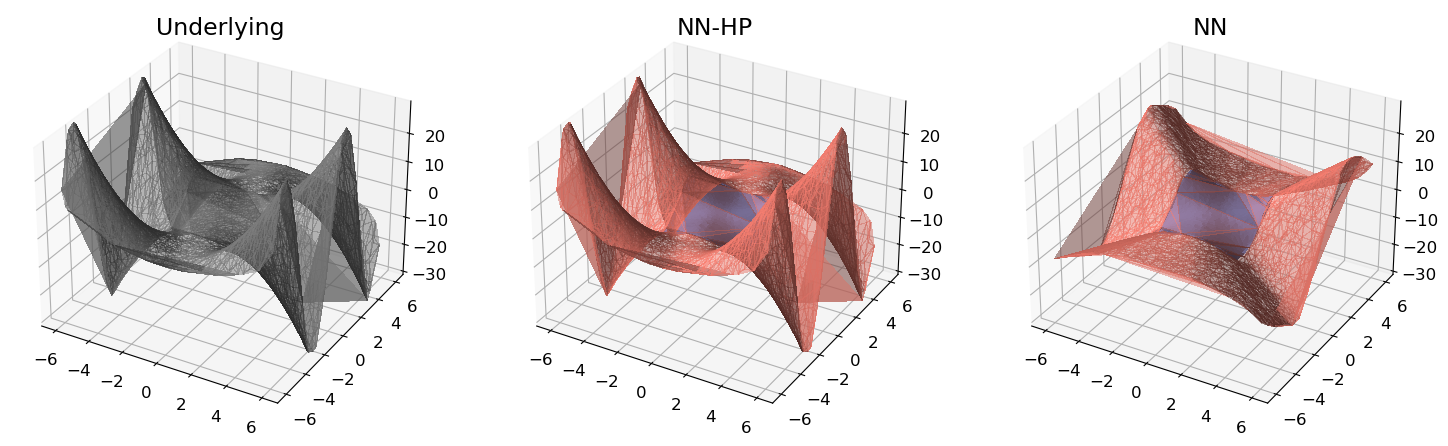}
\caption{Fitting results for underlying function 
$f_{\rho}(\bm x)=  {(x^{(1)})}^2\times\sin(x^{(2)})$. The prediction within the training (extrapolation) region is presented by \textcolor{blue}{blue} (\textcolor{red}{red}) color.
}
\end{subfigure}
\begin{subfigure}{0.9\linewidth}
    \centering \includegraphics[width=\textwidth]{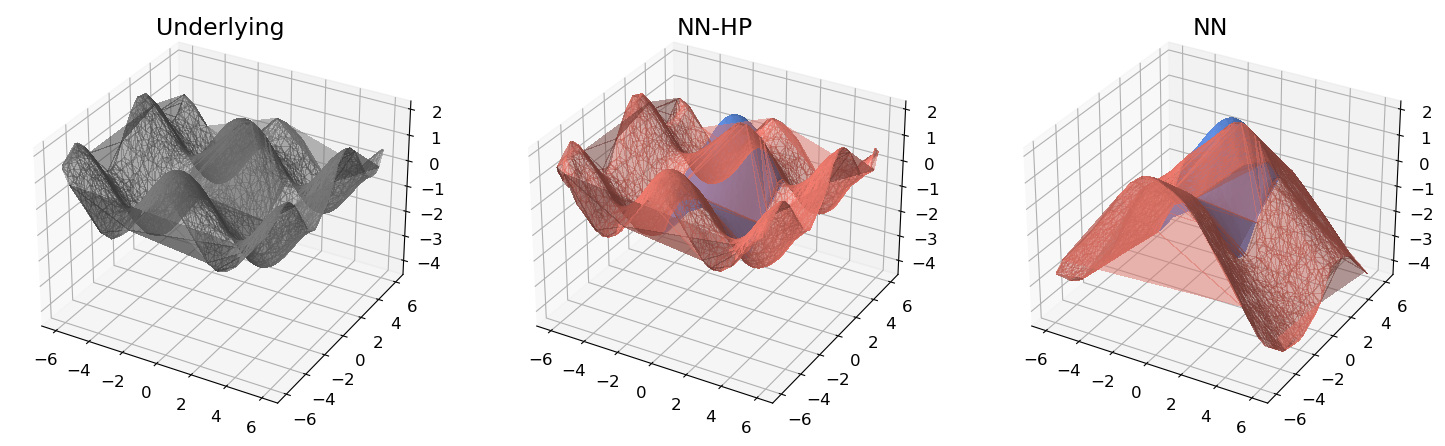}
\caption{Fitting results for the underlying function 
$f_{\rho}(\bm x)= \cos( x^{(1)})+\sin( x^{(2)})$. The prediction within the training (extrapolation) region is presented by \textcolor{blue}{blue} (\textcolor{red}{red}) color.
}
\end{subfigure}
\caption{This figure shows the results of fitting several analytically-known two-variable functions. We can see even though both \pnns{} and NNs can learn well in the training region, \pnns{} is much more flexible than standard NNs during extrapolation.}
\label{fig:learning2dfunction}
\end{figure}
\section{Details on the numerical experiments}
\label{sec:appendix_datail_experiment}
In the following content, we will describe the setup of several experiments including learning analytically-known function (\cref{sec:appendix_simplefunction}), variation of darkness (\cref{sec:appendix_darkness}), arithmetic extrapolation (\cref{sec:appendix_addition}), and learning harmonics (\cref{sec:appendix_harmonic}). The experiment of visual analogy task is included in \cref{sec:appendix_vaec}.
The experiment on the spectral bias in image classification is contained in \cref{sec:append_labelnoise}

\subsection{Experimental setup in learning analytically-known function}
\label{sec:appendix_simplefunction}

We describe the experimental setup corresponding to \cref{sec:learningsimplefunction}. $N$-layer fully-connected NNs are compared against \pnns{} with $N-1$ degree multiplicative interactions. The reason is that one-degree PNNs are equivalent to two-layer fully-connected NNs, according to the formula of PNNs provided in \cref{equ:ccp_repara_infinite_relu}. In the experiment, the training set consists of $20000$ data points in total.
The networks are trained for $50$ epochs with batch size $256$.
The squared loss is minimized through ADAM optimizer~\citep{kingma2014adam} with $\beta_{\text{1}} = 0.9$, $\beta_{\text{2}} = 0.999$, learning rate = $10^{-4}$. 

As a complement, we show additional results of fitting two-variable functions in \cref{fig:learning2dfunction} to further examine the power of Hadamard product.

In our work, the extrapolation relies on the support of the training data~\citep{van2021control}, as suggested by the previous work of \citet{xu2021how}. We note that for certain  applications and input data types, the convex hull might be required, however, we leave this as future work.
\subsection{Experimental setup in variation of darkness}
\label{sec:appendix_darkness}
This section describes the experimental setup of the variation of darkness experiment in~\cref{sec:extrapolationrealdataset}.
The following two datasets are used:
(a) MNIST dataset~\citep{lecun1998gradient}, which contains handwritten digits images from zero to nine. There are $60,000$ examples in the training set and $10,000$ examples in the testing set. Each image has the resolution $28 \times 28$. 
(b) Fashion-MNIST dataset \citep{xiao2017fashion}, which contains images of clothing with $10$ classes. There are $60,000$ examples in the training set and $10,000$ examples in the testing set. Each image has the resolution $28 \times 28$. The networks are trained for $20$ epochs with batch size $128$ with the criterion of cross entropy loss. The learning rate is chosen as $0.01$
. The width of the networks is $256$. Each network is trained for 3 runs.

\subsection{Experimental setup in arithmetic extrapolation}
\label{sec:appendix_addition}
This section describes the experimental setup of the arithmetic extrapolation experiment in~\cref{sec:extrapolationrealdataset}.
We construct a new dataset based on the MNIST dataset~\citep{lecun1998gradient} to demonstrate the addition of two (visual) numbers. We randomly pick $90$ combinations of two digits for training (out of the $100$ total combinations), and then we use the rest $10$ for extrapolation set. 
Specifically, in our three-fold cross-validation, we randomly pick up $90$ combinations of two digits and we sample $2000$ pairs for each combination to construct the training set. There are $90\times2000$ pairs in the training set and $10\times2000$ pairs in the testing set. Each network is trained for $100$ epochs with batch size $128$. The width of the networks is $256$. Each network is trained with squared-loss for 3 runs.

\subsection{Experimental setup in learning harmonics}
\label{sec:appendix_harmonic}
This section describes the experimental setup corresponding to \cref{sec:learningspherical}. 
We follow the setup in \citet{cao2019towards,choraria2022the}. The number of sample points is $1000$. The width of the network is $32768$. The network is trained for $30000$ iterations and optimized via stochastic gradient descent with learning rate $0.0016$.

\subsection{Visual analogy task}
\label{sec:appendix_vaec}
In this section, we scrutinize the extrapolation capability of \pnns{} on the visual analogy task on VAEC dataset~\cite{webb2020learning}.
For each pair of four images $A,B,C,D$, the proportional analogy problem is in the
form $A : B :: C : D$ based on the brightness, size,
and 2-D location. The model is required to select the correct $D$ among several candidates when given $A,B,C$. We conduct the scale extrapolation experiment introduced in the paper as it’s similar to our experiment on the variation of brightness, which treats the scale factor $\alpha$ as the extent of extrapolation. 
$\alpha = 1$ indicates the training set. $\alpha \in \{ 2,...,6 \}$
indicates the extrapolation set,
where the values of the dataset are multiplied by a scale factor $\alpha$ ranging from $2$ to $6$. 
We use the original best model in the paper as baseline (NNs) and insert Hadamard product as \pnns{} to compare. Apart from the network architecture, the training details are the same as in \citet{webb2020learning}. We run each method $8$ times and report the mean of accuracy in \cref{tab:vaec}.
Results show that both models achieve similar performance in the training regime while NNs-HP extrapolates better than standard NNs in most regimes.

\begin{table}[htb]
\centering
\caption{
Experimental results in the task of visual analogy on VAEC dataset. 'Ext' abbreviates 'extrapolation'.
We can see that
NNs-HP has better extrapolation performance in most extrapolation regimes.
} 
\begin{tabular}{l@{\hspace{0.05cm}} c@{\hspace{0.2cm}}c@{\hspace{0.2cm}}c
@{\hspace{0.2cm}}c
@{\hspace{0.2cm}}c
@{\hspace{0.2cm}}c} 
    \hline
     & Training ($\alpha = 1$) & Ext ($\alpha = 2$)
     & Ext ($\alpha = 3$)
     & Ext ($\alpha = 4$)
     & Ext ($\alpha = 5$)
     & Ext ($\alpha = 6$)
     \\
  \hline
    NNs& $99.5\%$  
    & $\textbf{76.2\%}$
    & $55.5\%$
    & $46.3\%$
    & $42.6\%$
    & $40.4\%$
    \\
   \pnns{} & $99.7\%$
&$73.7\%$
    & $\textbf{57.5}\%$
    & $\textbf{49.0}\%$
    & $\textbf{45.3}\%$
    & $\textbf{42.9}\%$
        \\
   \hline
\end{tabular}
\label{tab:vaec}
\end{table}

\subsection{Spectral bias in image classification}
\label{sec:append_labelnoise}
This experiment studies how the frequency of the noise affects the validation performance in image classification, which further validates our theoretical result on spectral bias in \cref{sec:methodspec}.
Specifically, 
we follow the standard set up in \citet{rahaman2019spectral}.
we consider a binary classification task with labels $3$ and $8$ on the MNIST dataset. We add noises with different frequencies to the label. We test \pnns{} with three, six, and nine-degree multiplicative interaction and compare it with the corresponding standard fully-connected neural networks. Both networks are optimized through Adam. We select mean squared loss as the criterion and choose learning rate $0.0001$. We train each network for $1000$ iterations. The width of the network is $256$. 
The results in \cref{exp:dip} present the 'dip' of validation mean squared error (MSE) during the process of training.
In the comparison of each order, for instance, in \cref{fig:dip1} we can see that in the case of higher frequencies, e.g., $0.3$ and $0.5$, the validation dips of \pnns{} are apparently smaller than that of \nns{} in the early stage of during training. The reason is that \pnns{} can speed up the learning of high-frequency information based on \cref{sec:methodspec}.
\begin{figure}[!thb]
\centering
\begin{subfigure}{0.9\linewidth}
    \centering \includegraphics[width=\textwidth]{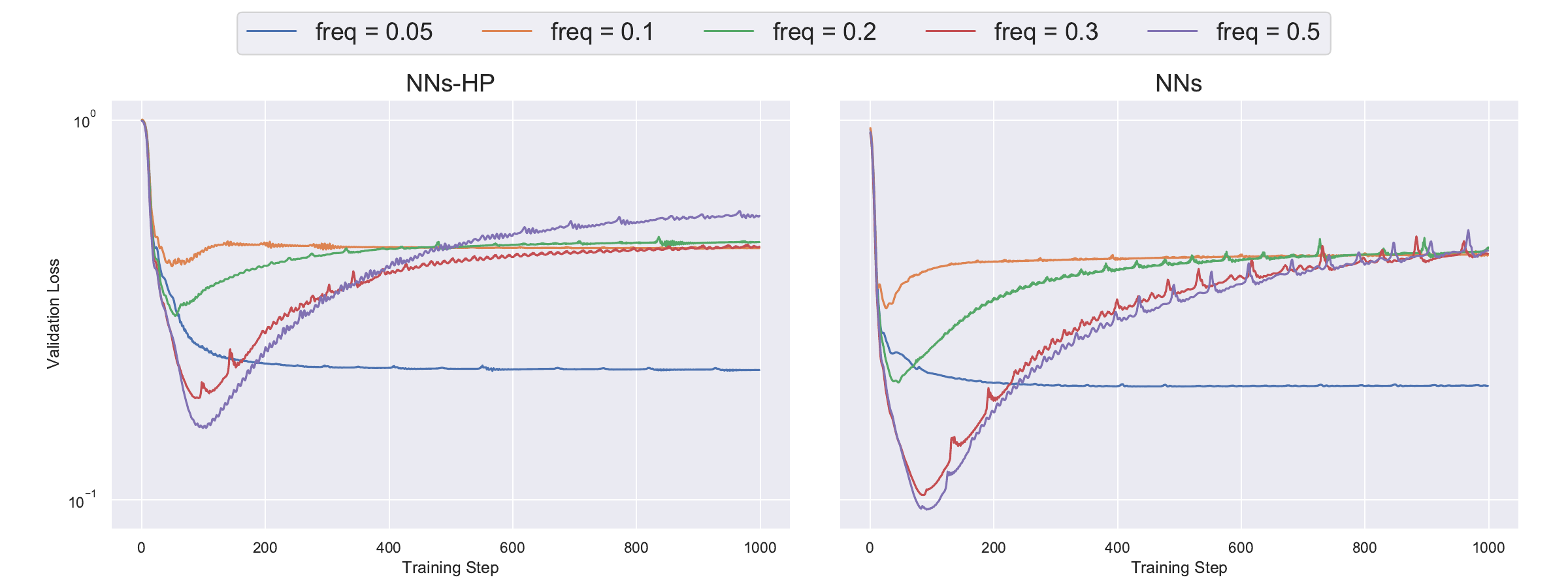}
\caption{
Results of \pnns{} with three-degree multiplicative interaction and the corresponding \nns{}.}
\label{fig:dip1}
\end{subfigure}
\begin{subfigure}{0.9\linewidth}
    \centering \includegraphics[width=\textwidth]{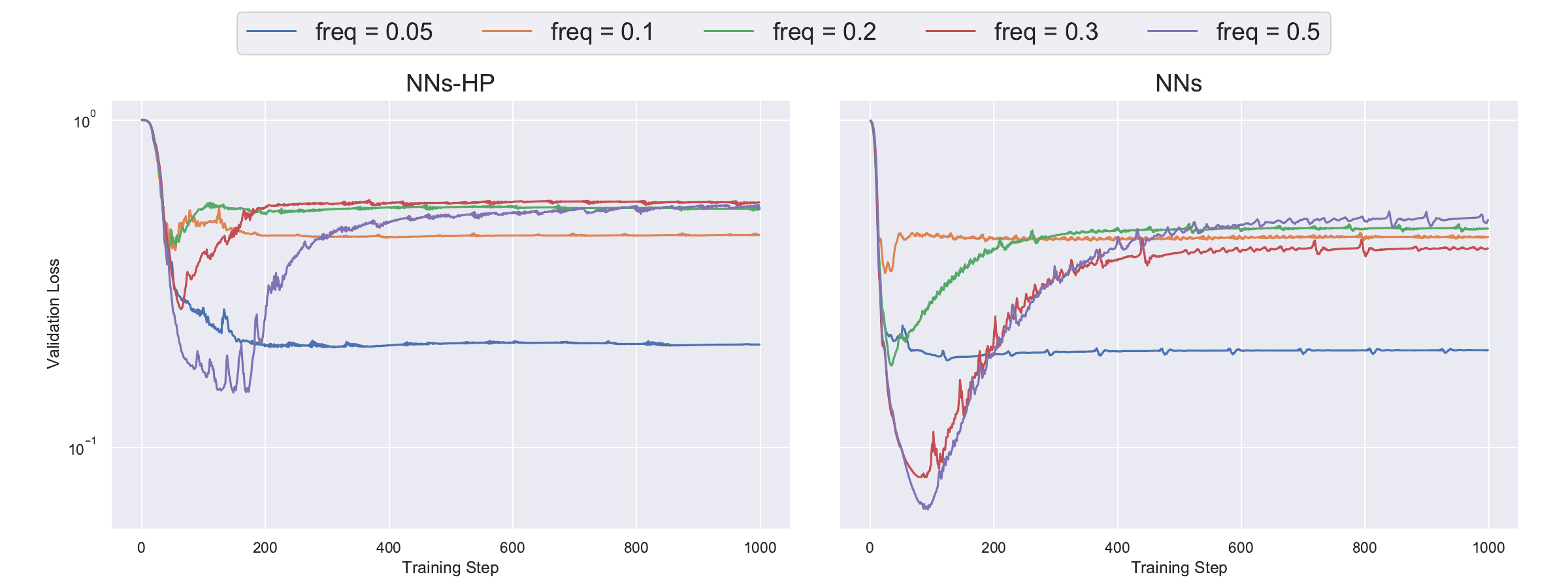}
\caption{
Results of \pnns{} with six-degree multiplicative interaction and the corresponding \nns{}.
}
\end{subfigure}
\begin{subfigure}{0.9\linewidth}
    \centering \includegraphics[width=\textwidth]{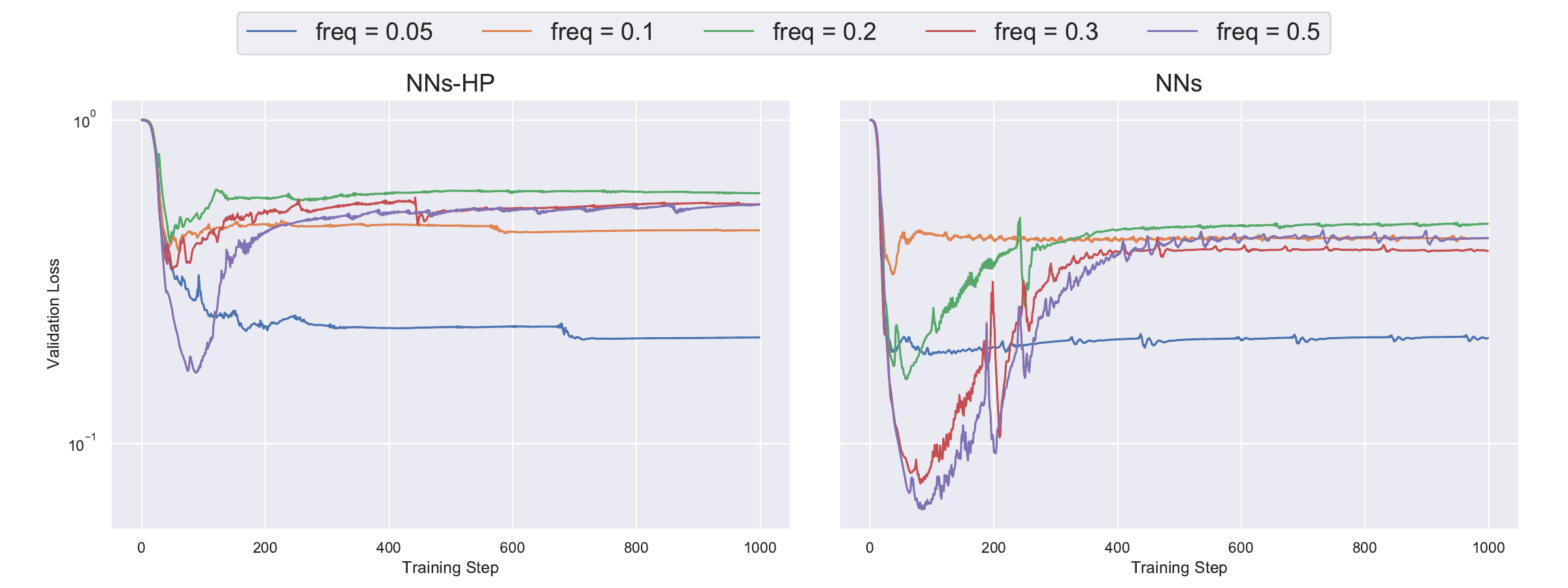}
\caption{
Results of \pnns{} with nine-degree multiplicative interaction and the corresponding \nns{}.
}
\end{subfigure}
\caption{
The above figures show the dip of validation loss of \pnns{} and \nns{} during training. Since \pnns{} is able to speed up learning high-frequency information, we can see that for high-frequency noise, such dip for \pnns{} is smaller than that of \nns{} at the early stage of training.}
\label{exp:dip}
\end{figure}
\section{Additional result on multiplicative filter networks}
\label{sec:mfn}
Multiplicative filter network (\mfn{}) is another instance of \pnns{} which inserts the Hadamard between the sinusoidal or Gabor wavelet functions among each layer \citep{fathony2021multiplicative}. \mfn{} has demonstrated stronger performance over standard neural networks in several representation tasks. 
\subsection{Theoretical analysis}
In this section, we will derive the neural tangent kernel of \mfn{} and then analyze the extrapolation behavior of MFN. We consider the following \mfn{}:
\begin{align*}
& \bm{\fm}_1 = \sqrt{\frac{2}{m}}	\sin(\W_1 \x),\;\;
\mlpout =  \sqrt{\frac{2}{m}}(\W_{N+1} \bm{\fm}_{N})
,\;\;
\bm{\fm}_n =\sqrt{\frac{2}{m}}\	\sin\left(\W_n \x \right)*\bm{\fm}_{n},\; n = 2, \ldots, N  \,,
\end{align*}
where each element in $\mW_{N+1} \in \R^{ 1\times m }$ and $\mW_n \in \R^{m\times d}$, for $n = 1, \ldots, N $ is independently sampled from $\mathcal{N}( 0, 1)$.
Note that we multiply by the scaling factor $\sqrt{\frac{2}{m}}$ after each \degorlay{}
to ensure that the norm of the network output is preserved at initialization with infinite-width setting.
In \cref{thm:mfn_infiniteNTK_formula} we develop the NTK $K(\x, \x^{\prime})$ of \mfn{}.
\begin{lemma} \label{thm:mfn_infiniteNTK_formula}
The neural tangent kernel of the MFN has the following form:
\begin{equation}
\begin{split}
K(\x, \x^{\prime})   =  & \; 
2N \cdot \langle \x,\xp  \rangle  \kappa_3(\x,\xp) (\kappa_4(\x,\xp))^{N-1} + 2 (\kappa_4(\x,\xp))^{N}\,,
\end{split}
\end{equation}
where $\kappa_3$ and $\kappa_4$ are defined by taking the random Gaussian vector $ \w \in \realnum^{d}$,
\begin{equation*}
\begin{split}
 & \kappa_3= \mathbb{E}_{\w \sim  \mathcal{N}(\bm{0}, \sqrt{\frac{2}{m}} \cdot
   \bm{I} )}
 \left( 
 \cos(\w^{\top} \x )
 \cdot  
 \cos(\w^{\top} \xp)
 \right), 
  \kappa_4=
  \mathbb{E}_{ \w \sim  \mathcal{N}(\bm{0},
  \sqrt{\frac{2}{m}} \cdot
 \bm{I} )} 
 \left( 
 \sin(\w^{\top} \x )
 \cdot  
 \sin(\w^{\top} \xp)
 \right)\,.
\end{split}
\end{equation*}
\end{lemma}
\begin{proof}[Proof of \cref{thm:mfn_infiniteNTK_formula}]
 We will compute the gradient with respect to each weight and then sum up the inner products to obtain the NTK.
Below, we denote by $\bm\preact_n=\mW_n\x, n \in [N]$.
Firstly, we compute the contribution to the NTK w.r.t $\mW_1$, its corresponding derivative is as follows:
\begin{align}
\partial_{\mW_{1}} f(\x)
& \nonumber=
    \sqrt{\frac{2}{m}}
    \left[
        \mW_{N+1}^\top
        \left(
        \prod \limits_{n=2}^{N}
        \text{Diag}
            \left(
            \sin\left(\sqrt{\frac{2}{m}}\tilde{\bm\alpha}_{n}(\boldsymbol{x})\right)
            \right)
        \right)
        \cos\left(\sqrt{\frac{2}{m}}\tilde{\bm\alpha}_{1}(\boldsymbol{x})\right)
    \right]^\top
    \left(\partial_{\mW_{1}} \tilde{\bm\alpha}_{1}
    (\boldsymbol{x})
    \right)^\top
\\& \nonumber=
    \sqrt{\frac{2}{m}}
    \left[
        \mW_{N+1}^\top
        \left(
        \prod \limits_{n=2}^{N}
        \text{Diag}
            \left(
            \sin\left(\sqrt{\frac{2}{m}}\tilde{\bm\alpha}_{n}(\boldsymbol{x})\right)
            \right)
        \right)
        \cos\left(\sqrt{\frac{2}{m}}\tilde{\bm\alpha}_{1}(\boldsymbol{x})\right)
    \right]^\top
    \vx ^\top
\end{align}
where $\text{Diag}(\cdot)$ converts a vector to a diagonal matrix.
The inner product follows that:
\begin{align}
& \nonumber
\langle
\partial_{\mW_{1}} f(\x),
\partial_{\mW_{1}} f(\x^{\prime})
\rangle
\\
&=
\frac{2}{m} 
\sum_{j=1}^{m}
W_{N+1}^{(j)}
W_{N+1}^{(j)}
\left(
\prod \limits_{n=2}^{N}\left(\frac{2}{m}
\sin\left({\tilde{\alpha}}_{n}^{(j)}(\x)\right)
\sin\left({\tilde{\alpha}}_{n}^{(j)}(\x^{\prime})\right)
\right)
\right)
\left(
\frac{2}{m}
\cos\left(\tilde{\alpha}_{1}^{(j)}(\boldsymbol{x})\right) 
\cos\left(\tilde{\alpha}_{1}^{(j)}
\left(\boldsymbol{x}^{\prime}\right)\right)\right)
\boldsymbol{x}^\top \boldsymbol{x}^{\prime}.
\end{align}
By the law of large numbers, we obtain:
\begin{align}
\lim _{m \rightarrow \infty}
\langle
\partial_{\mW_{1}} f(\x),
\partial_{\mW_{1}} f(\x^{\prime})
\rangle
=2 \langle \x,\xp  \rangle  \kappa_3(\x,\xp) 
(\kappa_4(\x,\xp) )^{N-1} \,.
\label{equ:contribution1_mfn}
\end{align}
Since the formula of the network is symmetric w.r.t $\lbrace \mW_i \rbrace_{i=1}^N $, the contributions of $\lbrace \mW_i \rbrace_{i=1}^N $ to the NTK are the same, we can trivially multiply \cref{equ:contribution1_mfn} by $N$.
\\
\\
Next, we will compute the contribution to the NTK w.r.t $\mW_{N+1}$, its corresponding derivative is as follows:
\begin{align*}
\partial_{\mW_{N+1}} f(\x)
 = \sqrt{\frac{2}{m}}
\left( 
\sqrt{\frac{2}{m}}
 \sin(\tilde{\bm\alpha}_{N}) *
   \ldots
   *
   \sqrt{\frac{2}{m}}
   \sin(\tilde{\bm\alpha}_{1})
\right)\,.
\end{align*}
The inner product follows that:
\begin{align*}
& 
\langle
\partial_{\mW_{N+1}} f(\x),
\partial_{\mW_{N+1}} f(\x^{\prime})
\rangle
=
\frac{2}{m} 
\sum_{j=1}^{m}
\left(
\prod \limits_{n=1}^{N}
\left(
\frac{2}{m}
\sin\left(\tilde{\alpha}_{n}^{(j)}(\boldsymbol{x})\right)
\sin\left(\tilde{\alpha}_{n}^{(j)}
(\x^{\prime})\right)
\right)
\right)\,.
\end{align*}
By the law of large numbers:
\begin{align}
\nonumber
&\lim _{m \rightarrow \infty}
\langle
\partial_{\mW_{N+1}} f(\x),
\partial_{\mW_{N+1}} f(\x^{\prime})
\rangle
\\
& = 2\cdot (\kappa_4(\x,\xp))^{N}\,.
\label{equ:contribution2_mfn}
\end{align}
The proof is completed by multiplying \cref{equ:contribution1_mfn}
by $N$ and adding by \cref{equ:contribution2_mfn}.
\end{proof}
The derived kernel enables us to study how \mfn{} trained by gradient descent extrapolates. 
\begin{theorem}
\label{thm:mfn_infiniteNTK_extra}
Suppose we train~\mfn{} with $N$-degree multiplicative interaction
with infinite-width on $\lbrace (\x_i, y_i)\rbrace_{i=1}^{|\mathcal{X}|}$, and the network is optimized with squared loss in the NTK regime.
For any direction $\vv \in \R^d$ that satisfies $\| \bm v \|_2 =\max\lbrace\| \bm x_{i} \|^2\rbrace$,
let $\x_0 = t \vv$ and $\x = \x_0 + h \vv$ with $t>1$ and $h>0$ be the extrapolation data points,
the output $f(\x_0 + h \vv)$ can extrapolate to $\text{poly}=(
\sin (\alpha t), \cos(\alpha t),t)$,  where $\alpha$ is constant and the order of $\sin (\alpha t)$ is up to $N$, the order of $\cos(\alpha t)$ and $t$ is one.
\end{theorem}
\begin{proof}[Proof of \cref{thm:mfn_infiniteNTK_extra}]
A specific feature map $\phi(\x)$ induced by the NTK of \mfn{} is
\begin{equation}
\begin{split}
\phi \left(\x\right) & = \left( c^{\prime} \x \cdot  \cos \langle\w, \x \rangle 
\cdot {\sin}(\langle\w, \x \rangle)^{N-1}
, 
c^{\prime\prime} {\sin}(\langle\w, \x \rangle )^{N}
 \right)\,,
\end{split}
\end{equation}
where $\w$ is sampled from $\N(\bm{0}, \bm{I})$, $c^{\prime}$ and $c^{\prime\prime}$ are constants.
Note that kernel regression solution is equivalent to the following form:
\begin{align}
\label{}
f (\x) = \bt^{\top}\phi(\x),
\end{align}
where the representation coefficient $\bt$ holds: 
\begin{equation}
\begin{split}
& \min_{\bt^{\prime}} \| \bt^{\prime}\|_2\\
\text{s.t.} \;\;\;  &\phi(\x_i)^{\top} \bt^{\prime} = y_i, \quad i = 1, \dots, {|\mathcal{X}|}\,.
\end{split}
\end{equation}
Therefore, given inputs $\x_0 = t \vv$ and $\x =  \x_0 + h\vv $, we have:
\begin{align*}
f(\x) - f(\x_0)  
&= 
    \bt^{\top}
    \left(
    \phi((t+h)\vv)-\phi(t\vv)
    \right)
\\&= 
    \bt_1 ^\top
        \left(
            c^{\prime} (t+h)\vv \cdot  \cos \langle\w, (t+h)\vv \rangle 
            \cdot {\sin}(\langle\w, (t+h)\vv \rangle)^{N-1}
        \right)
    \\&- 
    \bt_1^\top
        \left(
            c^{\prime} t\vv \cdot  \cos \langle\w, t\vv \rangle 
            \cdot {\sin}(\langle\w, t\vv \rangle)^{N-1}  
        \right)
    \\&+ 
    \beta_2
        \left(
            c^{\prime\prime} {\sin}(\langle\w, (t+h)\vv \rangle )^{N}
        \right)
    -
    \beta_2
        \left(
            c^{\prime\prime} {\sin}(\langle\w, t\vv \rangle )^{N}
        \right)\,,
\end{align*}
Therefore, the network can extrapolate to $\text{poly}=(
\sin (\alpha t), \cos(\alpha t),t)$,  where $\alpha$ is constant and the order of $\sin (\alpha t)$ is up to $N$. This completes the proof.
\end{proof}
\subsection{Numerical result}
In this section, we provide  additional experimental results on extrapolation with MFN.
Firstly, we follow the setup in \cref{sec:appendix_vaec} and present the result on VAEC dataset as follows, where we can see that We can see that
MFN has better extrapolation performance than standard NN in most extrapolation regimes.
\begin{table}[!htb]
\centering
\caption{
Experimental results in the task of visual analogy on VAEC dataset. 'Ext' abbreviates 'extrapolation'.
}
\begin{tabular}{l@{\hspace{0.05cm}} c@{\hspace{0.2cm}}c@{\hspace{0.2cm}}c
@{\hspace{0.2cm}}c
@{\hspace{0.2cm}}c
@{\hspace{0.2cm}}c} 
    \hline
     & Training ($\alpha = 1$) & Ext ($\alpha = 2$)
     & Ext ($\alpha = 3$)
     & Ext ($\alpha = 4$)
     & Ext ($\alpha = 5$)
     & Ext ($\alpha = 6$)
     \\
  \hline
    NN& $99.5\%$  
    & $\textbf{76.2\%}$
    & $55.5\%$
    & $46.3\%$
    & $42.6\%$
    & $40.4\%$
    \\
  MFN & $99.1\%$
&$74.5\%$
    & $\textbf{56.2}\%$
    & $\textbf{47.1}\%$
    & $\textbf{43.0}\%$
    & $\textbf{40.6}\%$
        \\
   \hline
\end{tabular}
\end{table}

Next, we apply MFN in the task of arithmetic extrapolation, as introduced in \cref{sec:extrapolationrealdataset}, and follow the same experimental setup. 
The following result further showcases the improvement of MFN over standard NN.
\begin{table}[!htb]
\centering
\caption{
Results with MFN and standard NN in the task of arithmetic extrapolation.}
\scalebox{0.96}{
\begin{tabular}{l@{\hspace{0.15cm}} l@{\hspace{0.15cm}} c@{\hspace{0.25cm}} c@{\hspace{0.25cm}} c@{\hspace{0.25cm}}c}
    \hline
Method&&Rounding&Floor/ceiling&$\pm1$
    \\
    \hline
    \multirow{2}{*}{\nn (Dense)} & Interpolation 
    &
    $0.980$ 
        &
    $ 0.999$ 
            &
    $ 0.999$ 
\\&   Extrapolation 
    &
    $ 0.436 $ 
        &
    $ 0.805$ 
            &
    $ 0.887$ 
    \\
    \hline
     \multirow{2}{*}{MFN (Dense) 
     } 
     &   Interpolation 
    &
    $ 0.996$
        &
    $ 0.997$
            &
    $ 0.999$
     &
\\&   Extrapolation 
    &
    $ \bm{0.720}$
        &
    $ \bm{0.874}$
            &
    $ \bm{0.916}$
    \\
    \hline
    & & \multicolumn{4}{c}{} \\
\hline

     \multirow{2}{*}{\nn (Conv)} 
     &   Interpolation 
    &
    $ 0.945$ 
    &
    $ 0.983$ 
        &
    $ 0.994$ 

     &
\\&   Extrapolation 
    &

    $ 0.617 $ 
        &
    $ 0.918$ 
            &
    $ 0.953$ 
    \\
    \hline
     \multirow{2}{*}{MFN (Conv)}
     &   Interpolation 
    &
    $ 
   0.947
    $
        &
    $0.996$
            &
    $ 0.999$
     &
\\&   Extrapolation 
    &
    $ \bm{0.824 }$
        &
    $ \bm{0.925}$
            &
    $ \bm{0.954}$
    \\
    \hline
\end{tabular}
}
\end{table}
\section{Additional result on non-local networks with Hadamard product}
\label{sec:nonlocal}
 Non-local networks have demonstrated stellar performance in capturing long-range dependencies of the input signals~\citep{wang2018non}. Particularly, Poly-NL is one of the non-local networks that utilize three-degree polynomial to reduce
the complexity of traditional non-local networks from quadratic to linear. Note that the formula in \citep{babiloni2021poly} is based on standard polynomial expansion, which we have analyzed in the main body. To make our analysis more general, we consider the following single Poly-NL block that uses Softmax as activation function:
\begin{align*}
&\bm{\fm}_1=\x 
\,, \quad  
    \bm{\fm}_2=
    \sigma{\left(\W_1 \bm{\fm}_1\right)} 
    \,,
\\& 
\bm{\fm}_3=\text{Softmax}\{
    \left(\vw_Q \bm{\fm}_2^{\!\top}\right)
    *
    \left(\vw_K \bm{\fm}_2^{\!\top}\right)
    \}
    \left( \bm{\fm}_2 w_V \right) 
    \,,
\\& 
    \mlpout = \sqrt{\frac{2}{m}}(\vw_{2}^{\!\top} \bm{\fm}_{3})
    ,
\end{align*}
where $\x \in \R^{d}$, $\mW_1 \in \R^{m \times d}$, $\vw_Q \in \R^{m}$,  $\vw_K \in \R^{m}$, $w_V \in \R$,    
$\vw_2 \in \R^{m}$, each element in the weight is independently sampled from $\mathcal{N}( 0, 1)$, the Softmax is row-wise.
Firstly, we give the neural tangent kernel of the Poly-NL, wherein we only train the weight $w_Q$ and $w_K$.
\begin{lemma} \label{thm:nonlocal_infiniteNTK_formula}
The neural tangent kernel of the Poly-NL has the following form:
\begin{equation*}
K(\x, \x^{\prime})   =  
    4 \cdot \mathbb{E}_{w_{3}, w_{4} \sim \mathcal{N}(0, 1)}
    \vy_2^{\!\top}
    \left(
        \text{Diag}{(\bm{\tau})}
        -\bm{\tau}
        {\bm{\tau}}^{\!\top}
    \right)
    \left(\vy_2 * \vy_2 \right)
    \vy_2^{\prime}{^{\!\top}}
    \left(
    \text{Diag}{({\bm{\tau^{\prime}}})}
    -{\bm{\tau^{\prime}}}
{\bm{\tau^{\prime}}}^{\top}
    \right)
    \left(\vy^{\prime}_2 * \vy_2^{\prime} \right)\,,
\end{equation*}
where we denote by $\bm{\tau}=\text{Softmax}(w_{3}w_{4}\left( \vy_2 * \vy_2 \right))$,
$\bm{\tau}^\prime=\text{Softmax}(w_{3}w_{4}\left( \vy_2^\prime * \vy_2^\prime \right))$
, $\vy_2=\sigma{(\mW_1\vx)}$, $\vy_2^\prime=\sigma{(\mW_1\vx^\prime)}$, 
where $w_3$ and $w_4$ is independently sampled from $\mathcal{N}( 0, 1)$.
\end{lemma}
\begin{proof}[Proof of \cref{thm:nonlocal_infiniteNTK_formula}]
 Firstly, we compute the Jacobian with respect to $\w_Q$:
 \begin{align*}
\partial_{\vw_{Q}}f(\x)
&= 
    \sqrt{\frac{2}{m}}
    \frac{\partial
    \left({\sum_{i=1}^{m}w_2^{(i)}
    y_3^{(i)}}
    \right)}{\partial{\vw_{Q}}}
\\&=
    \sqrt{\frac{2}{m}}
    \sum_{i=1}^{m}w_2^{(i)}
    w_V\vy_2^{\!\top}
    \frac{{
    \partial
    \text{Softmax}
    \left(
    w_Q^{(i)} w_K^{(i)}
    \left(\vy_2 * \vy_2 \right)
    \right)
    }}
    {\partial{\vw_{Q}}}
\\&=
    \sqrt{\frac{2}{m}}
   \sum_{i=1}^{m}w_2^{(i)}
   w_K^{(i)}
    w_V\vy_2^{\!\top}
    \left(
        \text{Diag}{(\bm{\varphi}_i)}
        -\bm{\varphi}_i
        {\bm{\varphi}_i}^{\!\top}
    \right)
    \left(\vy_2 * \vy_2 \right)
    {\ve_i}^{\!\top},
\end{align*}
 where we denote by 
    $\bm{\varphi}_i=
    \text{Softmax}
    \left(
    w_Q^{(i)} w_K^{(i)}
    \left(\vy_2 * \vy_2 \right)
    \right)\in \R^m$.
Next, in order to obtain the NTK, we calculate the inner product of the Jacobian:
\begin{align*}
&
    \langle
    \partial_{\vw_{Q}} f(\x),
    \partial_{\vw_{Q}} f(\x^{\prime})
    \rangle
\\&=
   \frac{2}{m}
   \sum_{i=1}^{m}
   (w_2^{(i)}  w_K^{(i)} w_V)^2
       \vy_2^{\!\top}
        \left(
            \text{Diag}{(\bm{\varphi}_i)}
            -\bm{\varphi}_i
            {\bm{\varphi}_i}^{\!\top}
        \right)
        \left(\vy_2 * \vy_2 \right)
   \vy_2^{\prime}{^{\!\top}}
    \left(
        \text{Diag}{
        \left(
        {\bm{\varphi}}^{\prime i
        }
        \right)
        }
        -{\bm{\varphi}}^{\prime i}
{\bm{\varphi}}^{\prime i^{\!\top}}
    \right)
    \left(\vy^{\prime}_2 * \vy_2^{\prime} \right)
\end{align*}
By the law of large numbers, as $m \rightarrow \infty$,
we obtain:
\begin{align}
\nonumber
&\lim _{m \rightarrow \infty}
\langle
\partial_{\vw_{Q}} f(\x),
\partial_{\vw_{Q}} f(\x^{\prime})
\rangle
\\& =
    2 \cdot \mathbb{E}_{w_{3}, w_{4} \sim \mathcal{N}(0, 1)}
    \vy_2^{\!\top}
    \left(
        \text{Diag}{(\bm{\tau})}
        -\bm{\tau}
        {\bm{\tau}}^{\!\top}
    \right)
    \left(\vy_2 * \vy_2 \right)
    \vy_2^{\prime}{^{\!\top}}
    \left(
    \text{Diag}{({\bm{\tau^{\prime}}})}
    -{\bm{\tau^{\prime}}}
    {\bm{\tau^{\prime}}}^{\top}
    \right)
    \left(\vy^{\prime}_2 * \vy_2^{\prime} \right)\,,
    \label{equ:ntk_nonlocal_con1}
\end{align}
where we denote by $\bm{\tau}=\text{Softmax}(w_{3}w_{4}\left( \vy_2 * \vy_2 \right))$.
 Since the weight $\w_Q$ and $\w_K$ are symmetric in the formula of Poly-NL, the proof is completed by multiplying \cref{equ:ntk_nonlocal_con1} by two.
 \end{proof}
Now we are ready to analyze the extrapolation behaviour of Poly-NL.
\begin{theorem}
\label{thm:nonlocaextra}
Suppose we train Poly-NL
with infinite-width on $\lbrace (\x_i, y_i)\rbrace_{i=1}^{|\mathcal{X}|}$, and the network is optimized with squared loss in the NTK regime.
For any direction $\vv \in \R^d$ that satisfies $\| \bm v \|_2 =\max\lbrace\| \bm x_{i} \|^2\rbrace$ 
,
let $\x^\prime = t \vv$ and $\x = \x^\prime + h \vv$ with $t>1$ and $h>0$ be the extrapolation
data points, then for $\delta \in (0,1)$ and some constant $C$, when $m\geq 2\ln{(2/\delta)}+d+\sqrt{8d\ln{(2/\delta)}}$, with probability at least $1-\delta$, we have:
\begin{equation*}
|f(\x) - f(\x^\prime)|\leq
C t^3 h^3 m^{\frac{3}{2}}\|\vv\|^3
.
\end{equation*}
\end{theorem}
\begin{proof}[Proof of \cref{thm:nonlocaextra}]
We first bound the spectral norm of the weight matrix $\mW_1$.
\begin{lemma}
\label{equ:lemmarandomw1}
    Based on the randomness of the weight $\mW_1$,
for $\delta \in (0,1)$, when $m\geq 2\ln{(2/\delta)}+d+\sqrt{8d\ln{(2/\delta)}}$, with probability at least $1-\delta$, we have $\|\mW_1\|\leq 2\sqrt{m}.$
\end{lemma}
We can choose a certain feature map $\phi(\x)$ induced by the NTK in \cref{thm:nonlocal_infiniteNTK_formula} is
\begin{equation}
\begin{split}
\phi \left(\x\right) & =
\left(\vy_2^{\!\top}
    \left(
        \text{Diag}{(\bm{\tau})}
        -\bm{\tau}
        {\bm{\tau}}^{\!\top}
    \right)
    \left(\vy_2 * \vy_2 \right)
\right)
\,,
\end{split}
\end{equation}
where $\widetilde{\bm \tau}=\text{Softmax}(w_{5}w_{6}\left( \vy_2 * \vy_2 \right))$, $w_5, w_6$ are iid sampled from $\N(0, 1)$, $c^{\prime}$ and $c^{\prime\prime}$ are constants.
Similarly, using the solution of kernel regression, we
can calculate the output of the network as follows. Given $\x^\prime = t \vv$ and $\x =  \x^\prime + h\vv $, we add the prime symbol to the variable in the network associated to the input $\x^\prime$, we have:
\begin{equation}
\begin{split}
|f(\x) - f(\x^\prime)|
&= 
    |\beta
    \left(
    \phi((t+h)\vv)-\phi(t\vv)
    \right)|
\\&\leq
    |\beta|
            \|\vy_2\|
            \|
            \left(
            \text{Diag}{(\bm{\tau})}
            -\bm{\tau}
            {\bm{\tau}}^{\!\top}
            \right) 
            \|
        \|\vy_2*\vy_2-\vy_2^\prime*\vy_2^\prime\|
    \\&+
    |\beta|
    \|\vy_2^\prime*\vy_2^\prime\|
        \|\vy_2\|
         \left\|
            \left(
            \text{Diag}{(\bm{\tau})}
            -\bm{\tau}
            {\bm{\tau}}^{\!\top}
            \right) 
            -
            \left(
            \text{Diag}{(\bm{\tau}^\prime)}
            -\bm{\tau}^\prime
        {\bm{\tau}\prime}^{\!\top}
            \right) 
        \right\|
    \\&+|\beta|
    \|\vy_2^\prime*\vy_2^\prime\|
         \left\|
            \left(
            \text{Diag}{(\bm{\tau}^\prime)}
            -\bm{\tau}^\prime
        {\bm{\tau}^\prime}^{\!\top}
            \right) 
        \right\|  
        \| \vy_2-\vy_2^\prime \|.
\end{split}
\label{equ:nonlocal_allbound}
\end{equation}
We start by bounding the first term in \cref{equ:nonlocal_allbound}.
For $\delta \in (0,1)$, when $m\geq 2\ln{(2/\delta)}+d+\sqrt{8d\ln{(2/\delta)}}$, with probability at least $1-\delta$, we have:
\begin{equation}
\begin{split}
&
    |\beta|
    \|\vy_2\|
    \|
    \left(
    \text{Diag}{(\bm{\tau})}
    -\bm{\tau}
    {\bm{\tau}}^{\!\top}
    \right) 
    \|
    \|\vy_2*\vy_2-\vy_2^\prime*\vy_2^\prime\|
\\&\leq
    |\beta|
    \| \vy_2\|
    \left(
    \|\text{Diag}{(\bm{\tau})} \|
    +
    \|\bm{\tau}
    {\bm{\tau}}^{\!\top}\|
    \right)
    \left(
    \| \vy_2\|+
    \| \vy_2^\prime\| 
    \right)
    \left(
    \| \vy_2 - \vy_2^\prime\|
    \right)
\\&\leq
    2|\beta|
    \|\mW_1 (t+h)\vv\|
    \left(
    \|(t+h)\mW_1 \vv\|
    +
    \|t\mW_1 \vv\|
    \right)
    \left(
    \| h  \mW_1  \vv \|
    \right)
\\&\leq
    4|\beta|(t+h)\sqrt{m}\|\vv\|
    \left(
    2\sqrt{m}(t+h)\|\vv\|+
    2\sqrt{m}t\|\vv\|
    \right)
    \left(
    2\sqrt{m}h\|\vv\|
    \right)
\\&=
16 |\beta| m^\frac{3}{2} \|\vv\|^3
\left(2t^2h+3th^2+h^3\right).\,
\end{split}
\label{equ:nonlocal_bound1}
\end{equation}
where the first inequality comes from triangle inequality, the second inequality is due to the fact that the output of softmax ranges from zero to one, and the $1$-Lipschitz of ReLU.
Next, we bound the second term in \cref{equ:nonlocal_allbound}. 
Similarly, by \cref{equ:lemmarandomw1}, over the same randomness of the weight $\mW_1$, with probability at least $1-\delta$, we have:
\begin{equation}
\begin{split}
&
    |\beta|
    \|\vy_2^\prime*\vy_2^\prime\|
        \|\vy_2\|
         \left\|
            \left(
            \text{Diag}{(\bm{\tau})}
            -\bm{\tau}
            {\bm{\tau}}^{\!\top}
            \right) 
            -
            \left(
            \text{Diag}{(\bm{\tau}^\prime)}
            -\bm{\tau}^\prime
        {\bm{\tau}^\prime}^{\!\top}
            \right) 
        \right\| 
 \\& \leq 
     |\beta|
    \|\vy_2^\prime\|^2
    \|\vy_2\|
    \left\|
    \left(
    \text{Diag}{\left(\bm{\tau}-\bm{\tau}^\prime\right)}
    \right\|+
    \left\|
    \bm{\tau}
     {\bm{\tau}}^{\!\top}-
     \bm{\tau}^\prime
    {\bm{\tau}}^\prime{^{\!\top}}
    \right\|
    \right)
 \\& \leq 
    4
    |\beta| 
    mt^2 \|\vv\|^2
    \times
    2\sqrt{m} (t+h) \|\vv\|
    \times
    4
 = 32 |\beta| m^{\frac{3}{2}}\|\vv\|^{3}(t^3+t^2h).
\end{split}
\label{equ:nonlocal_bound2}
\end{equation}
Next, we bound the third term in \cref{equ:nonlocal_allbound}. Similarly, by \cref{equ:lemmarandomw1}, over the same randomness of the weight $\mW_1$, with probability at least $1-\delta$, we have:
\begin{equation}
\begin{split}
&|\beta|
    \|\vy_2^\prime*\vy_2^\prime\|
         \left\|
            \left(
            \text{Diag}{(\bm{\tau}^\prime)}
            -\bm{\tau}^\prime
        {\bm{\tau}^\prime}^{\!\top}
            \right) 
        \right\|  
        \| \vy_2-\vy_2^\prime \|
\\&\leq
     |\beta|
    \|\vy_2^\prime\|^2
    \left(
    \left\|
    \text{Diag}{(\bm{\tau}^\prime)}
    \right\|
    +
    \left\|
    \bm{\tau}^\prime
        {\bm{\tau}^\prime}^{\!\top}
    \right\|
    \right)
    \| \vy_2-\vy_2^\prime \|  
\\&\leq
    4mh^2 
    |\beta|
    |\vv\|^2
    \times 2 \times 2\sqrt{m}h\|\vv\|
    =16|\beta| m^{\frac{3}{2}}
    \|\vv\|^{3} h^3.
\end{split}
\label{equ:nonlocal_bound3}
\end{equation}
Therefore, the proof is completed by summing up \cref{equ:nonlocal_bound1,equ:nonlocal_bound2,equ:nonlocal_bound3}.
\end{proof}
\section{Societal impact}
\label{sec:societalimpact}
This work studies a cutting-edge network architecture, i.e., neural network with Hadamard product (NN-Hp), from a theoretical perspective. The analysis of the corresponding NTK 
lays a theoretical foundation for the interested practitioner to further study other priorities of NN-Hp such as convergence and generalization.
Furthermore, our current analysis mainly focuses on the theoretical side of extrapolation. We believe our insight and empirical evidence in extrapolation will allow the investigation of other more complicated OOD problems among the ML community, such as domain adaption and invariant learning. 
Therefore, we do not expect any negative societal bias from this work.  

\section{Limitations}
\label{sec:pntk_limitations}

In this work, we illustrate how our theory can be applicable in a variety of experimental settings, especially on extrapolation. 
Nevertheless, we do not focus explicitly on obtaining state-of-the-art numerical results in real-world applications, which could be one limitation of this work. 

Our proof framework is based on NTK for understanding theoretical properties of neural networks.
However, NTK still works in ``linear'' regime \citep{lee2019wide,woodworth2020kernel}, which appears difficult to fully demonstrate the success of practical neural networks.
Nevertheless, this is a common limitation of NTK-based analysis in the community. 

\end{document}